\pdfoutput=1
\documentclass{article} %
\usepackage{times}
\usepackage{geometry}

\usepackage[utf8]{inputenc} %
\usepackage[T1]{fontenc}    %
\usepackage[colorlinks,citecolor=blue,linkcolor=blue]{hyperref}       %
\usepackage{url}            %
\usepackage{booktabs}       %
\usepackage{amsfonts}       %
\usepackage{nicefrac}       %
\usepackage{microtype}      %
\usepackage{tikz,pgfplots}	                            %
\usepackage{mkolar_definitions}
\usepackage{url}
\usepackage{authblk}
\usepackage{comment}

\usepackage{graphicx}
\usepackage{subfig}
\usepackage{multirow}
\usepackage{floatrow}
\newfloatcommand{capbtabbox}{table}[][\FBwidth]

\PassOptionsToPackage{numbers}{natbib}
\usepackage{natbib}

\usepackage{algorithm,algorithmicx}
\usepackage[noend]{algpseudocode}
\usepackage{setspace}
\usepackage{xspace}
\usepackage{graphicx}

\usepackage{enumerate}
\usepackage{amsthm,amsmath,amssymb}
\usepackage{amsfonts,dsfont}
\usepackage{nicefrac}
\usepackage{microtype}
\usepackage{mathtools}

\usepackage{xcolor}
\newcount\Comments  %
\definecolor{darkgreen}{rgb}{0,0.5,0}
\definecolor{darkred}{rgb}{0.7,0,0}
\definecolor{teal}{rgb}{0.3,0.8,0.8}
\newcommand{\kibitz}[2]{\ifnum\Comments=1\textcolor{#1}{\textsf{\footnotesize #2}}\fi}

\newenvironment{packed_item}{
  \begin{itemize}
    \setlength{\itemsep}{1pt}
    \setlength{\parskip}{-1pt}
    \setlength{\parsep}{0pt}
}{\end{itemize}}

\newcounter{qcounter}
 {\end{list}}

\usepackage{ifthen}

\newcommand{\dfc}[1]{}

\newcommand{\mycomment}[1]{\textcolor{blue}{\slash\slash~#1}}

\newcommand{\prob}{\Delta}
\newcommand{\mdp}{\Mcal}
\newcommand{\emf}{q}
\newcommand{\decoder}{g^\star}
\newcommand{\horizon}{H}
\newcommand{\rewardf}{R}

\newcommand{\startdist}{\mu}

\newcommand{\observations}{\mathcal{X}}
\newcommand{\obs}{x}
\newcommand{\states}{\mathcal{S}}
\newcommand{\state}{s}

\newcommand{\actions}{\mathcal{A}}
\newcommand{\action}{a}
\newcommand{\policy}{\pi}
\newcommand{\policies}{\Pi}
\newcommand{\nspolicies}{\policies_{\textrm{NS}}}
\newcommand{\homingp}{\pi^\star}

\newcommand{\transition}{T}

\newcommand{\unf}{{\tt Unf}}
\newcommand{\step}{\tau}
\newcommand{\supp}{\textrm{supp}}
\newcommand{\errcsc}{\Delta_{csc}}

\newcommand{\errreg}{\Delta_{reg}}

\newcommand{\etamin}{\eta_{min}}
\newcommand{\nkid}{\textrm{N}_\textrm{KD}}
\newcommand{\nbkid}{\textrm{N}_\textrm{BD}}
\newcommand{\nfkid}{\textrm{N}_\textrm{FD}}
\newcommand{\timestep}{\tau}

\newcommand{\homingalg}{{\tt HOMER}}
\newcommand{\psdpalg}{{\tt PSDP}}
\newcommand{\gpsalg}{{\tt GPS}}
\newcommand{\oraclealg}{{\tt ExpOracle}}

\newcommand{\cpeoracle}{{\tt REG}}
\newcommand{\cboracle}{{\tt CB}}
\newcommand{\csoracle}{{\tt CSC}}

\newcommand{\fabs}{\phi}		%
\newcommand{\forabs}{\hat{\phi}^{(\textup{F})}_{h-1}}
\newcommand{\forabscurr}{\hat{\phi}^{(\textup{F})}_h}
\newcommand{\babs}{\phi'}		%
\newcommand{\backabs}{\hat{\phi}^{(\textup{B})}_h}
\newcommand{\backwardabs}{\phi^\star_{\textup{back}}}
\newcommand{\cover}{\Psi}

\newcommand{\prior}{\rho_h}
\newcommand{\marginal}{\mu_{h-1}}

\newcommand{\npolo}{n_{\textup{psdp}}}
\newcommand{\nevalo}{n_{\textup{eval}}}

\newcommand{\nreg}{n_{\textup{reg}}}
\newcommand{\npol}{n_{\textup{psdp}}}
\newcommand{\neval}{n_{\textup{eval}}}

\newcommand{\creg}{\textup{Time}_{\textup{reg}}}
\newcommand{\cpol}{\textup{Time}_{\textup{pol}}}

\newcommand{\ppo}{{\tt PPO}}
\newcommand{\ppornd}{{\tt PPO+RND}}
\newcommand{\aac}{{\tt A2C}}
\newcommand{\aacrnd}{{\tt A2C+RND}}

\newcommand{\du}{{\tt PCID}}

\newcommand\blfootnote[1]{%
  \begingroup
  \renewcommand\thefootnote{}\footnote{#1}%
  \addtocounter{footnote}{-1}%
  \endgroup
}

\renewcommand{\algorithmiccomment}[1]{\hfill\textcolor{blue}{$\slash\slash$ #1}}
\newtheorem{assumption}{Assumption} 
\usepackage{tikz}
\usetikzlibrary{decorations.pathreplacing,calc}
\newcommand{\tikzmark}[1]{\tikz[overlay,remember picture] \node (#1) {};}

\newcommand*{\AddNote}[4]{%
    \begin{tikzpicture}[overlay, remember picture]
        \draw [decoration={brace,amplitude=0.5em},decorate,ultra thick,blue]
            ($(#3)!(#1.north)!($(#3)-(0,1)$)$) --  
            ($(#3)!(#2.south)!($(#3)-(0,1)$)$)
                node [align=center, text width=2.25cm, pos=0.5, anchor=west] {#4};
    \end{tikzpicture}
}%

\begin{document}

\title{Kinematic State Abstraction and Provably Efficient Rich-Observation Reinforcement Learning}

\date{}

\author[]{
Dipendra Misra}
\author[]{
Mikael Henaff}
\author[]{
Akshay Krishnamurthy}
\author[]{
John Langford}

\affil[]{Microsoft Research, New York, NY}
\maketitle
\blfootnote{\{dimisra, mihenaff, akshaykr, jcl\}@microsoft.com}
\vspace{-1cm}

\maketitle

\begin{abstract}
  We present an algorithm, $\homingalg$, for exploration and
  reinforcement learning in rich observation environments that are
  summarizable by an unknown latent state space. The algorithm
  interleaves representation learning to identify a new notion of
  \emph{kinematic state abstraction} with strategic exploration to
  reach new states using the learned abstraction. The algorithm
  provably explores the environment with sample complexity scaling
  polynomially in the number of latent states and the time horizon,
  and, crucially, with no dependence on the size of the observation
  space, which could be infinitely large.
This exploration
  guarantee further enables sample-efficient global policy
  optimization for any reward function.  On the computational side, we
  show that the algorithm can be implemented efficiently whenever
  certain supervised learning problems are tractable.  Empirically, we
  evaluate $\homingalg$ on a challenging exploration problem, where we
  show that the algorithm is exponentially more sample efficient than
  standard reinforcement learning baselines.
\end{abstract}

\section{Introduction}
\label{sec:intro}

Modern reinforcement learning applications call for agents that
operate directly from rich sensory information such as megapixel
camera images. This rich information enables representation of
detailed, high-quality policies and obviates the need for
hand-engineered features.  However, exploration in such settings is
notoriously difficult and, in fact, statistically intractable in
general~\citep{AuerRL,lattimore2012pac,krishnamurthy2016richobs}.
Despite this, many environments are highly structured and do admit
sample efficient algorithms~\citep{jiang2017contextual};  indeed, we
may be able to summarize the environment with a simple state space and
extract these states from raw observations. With such structure, we
can leverage techniques from the well-studied tabular setting to
explore the environment~\citep{hazan2018provably}, efficiently recover
the underlying dynamics~\citep{strehl2008analysis}, and optimize any
reward
function~\citep{kearns2002near,Rmax,DelayedQ,dann2017unifying,azar2017minimax,jin2018q}. But
can we learn to decode a simpler state from raw observations alone?

The main difficulty is that learning a state decoder, or a compact
representation, is intrinsically coupled with exploration. On one
hand, we cannot learn a high-quality decoder without gathering
comprehensive information from the environment, which may require a
sophisticated exploration strategy. On the other hand, we cannot
tractably explore the environment without an accurate decoder. These
interlocking problems constitute a central challenge in reinforcement
learning, and a provably effective solution remains elusive despite decades of
research ~\cite{Mccallum1996, Ravindran2004, Jong2005, Li06towardsa, bellemare2016unifying, nachum2018near} (See~\pref{sec:related} for a discussion of related work).

In this paper, we provide a solution for a significant sub-class of
problems known as Block Markov Decision Processes
(MDPs)~\cite{du2019provably}, in which the agent operates directly on
rich observations that are generated from a small number of unobserved
latent states. Our algorithm, $\homingalg$, learns a new reward-free
state abstraction called \emph{kinematic inseparability}, which it
uses to drive exploration of the environment. Informally, kinematic
inseparability aggregates observations that have the same forward and
backward dynamics. Shared backward dynamics crucially implies that a
single policy simultaneously maximizes the probability of observing a
set of kinematically inseparable observations, which is useful for
exploration. Shared forward dynamics is naturally useful for recovering
the latent state space and model. Perhaps most importantly, we show
that a kinematic inseparability abstraction can be recovered from a
bottleneck in a regressor trained on a contrastive estimation problem
derived from raw observations.

$\homingalg$ performs strategic exploration by training policies to
visit each kinematically inseparable abstract state, resulting in a
\emph{policy cover}.  These policies are constructed via a reduction
to contextual bandits~\citep{bagnell2004policy}, using a
dynamic-programming approach and a synthetic reward function that
incentivizes reaching an abstract state. Crucially, $\homingalg$
interleaves learning the state abstraction and policy cover in an
inductive manner: we use the policies learned from a coarse
abstraction to reach new states, which enables us to refine the state
abstraction and learn new policies
(See~\pref{fig:interleaving-process} for a schematic). Each process is
essential to the other. Once the policy cover is constructed, it can
be used to efficiently gather the information necessary to find a
near-optimal policy for any reward function.

\begin{figure*}
\centering
\includegraphics[scale=0.35]{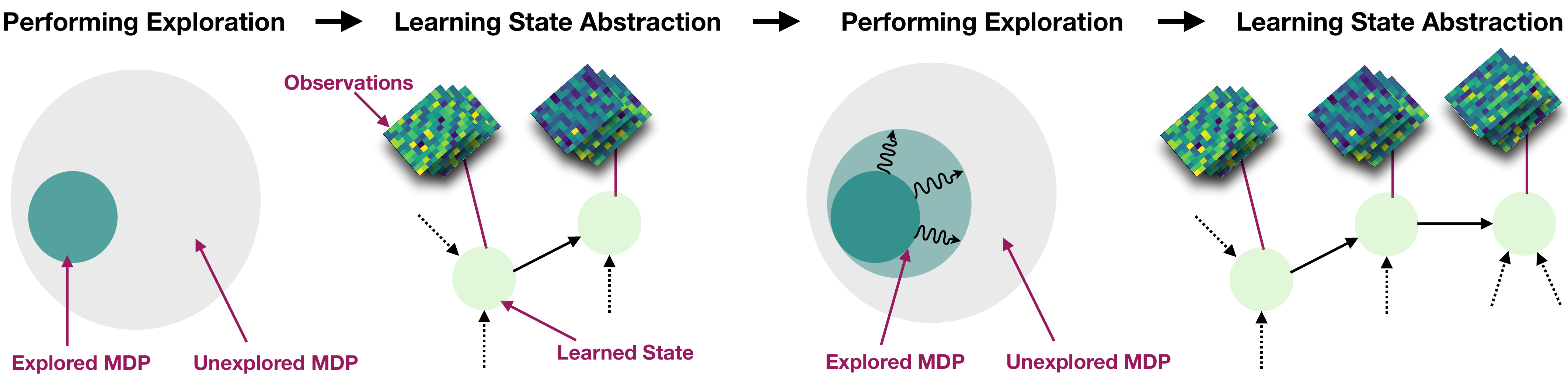}
\caption{$\homingalg$ learns a set of exploration policies and a state
  abstraction function by iterating between exploring using the
  current state abstraction and refining the state abstraction based
  on the new experience.}
\label{fig:interleaving-process}
\end{figure*}

We analyze the statistical and computational properties of
$\homingalg$ in episodic Block MDPs. We prove that $\homingalg$ learns
to visit every latent state and also learns a near-optimal policy for
any given reward function with a number of trajectories that is
polynomial in the number of latent states, actions, horizon, and the
complexity of two function classes used by the algorithm. There is no
explicit dependence on the observation space size.  The main
assumptions are that the latent states are reachable and that the
function classes are sufficiently expressive. There are no
identifiability or determinism assumptions beyond decodability of the
Block MDP, resulting in significantly greater scope than prior
work~\citep{du2019provably,dann2018oracle}.  On the computational
side, $\homingalg$ operates in a reductions model and can be
implemented efficiently whenever cetain supervised learning problems
are tractable.

Empirically, we evaluate $\homingalg$ on a challenging RL
problem with high-dimensional observations, precarious dynamics, and
sparse, misleading rewards. The problem is googal-sparse:
the probability of encountering an optimal reward through random
search is $10^{-100}$. $\homingalg$ recovers the underlying state
abstraction for this problem and consistently finds a near-optimal
policy, outperforming popular RL baselines that use naive exploration strategies~\citep{mnih16, schulman2017proximal} or more sophisticated exploration bonuses~\citep{burda2018exploration}, as well as the recent PAC-RL algorithm of~\citep{du2019provably}. %

\section{Preliminaries}

We consider reinforcement learning (RL) in episodic Block Markov
Decision Processes (Block MDP), first introduced
by~\citet{du2019provably}. A Block MDP $\mdp$ is described by a large
(possibly infinite) observation space $\observations$, a finite
unobservable state space $\states$, a finite set of actions
$\actions$, and a time horizon $H \in \NN$. The process has a starting
state distribution $\startdist \in \prob(\states)$\footnote{Du et
  al.~\cite{du2019provably} assume the starting state is
  deterministic, which we generalize here.}, transition function
$\transition: \states \times \actions \to \prob(\states)$, emission
function $\emf: \states \to \prob(\observations)$, and a reward
function $\rewardf: \observations \times \actions \times \observations
\to \prob([0,1])$. The agent interacts with the environment by
repeatedly generating $H$-step trajectories
$(s_1,x_1,a_1,r_1,\ldots,s_H,x_H,a_H,r_H)$ where $s_1 \sim
\startdist$, $s_{h+1} \sim \transition(\cdot | s_h,a_h)$, $x_h \sim
\emf(s_h)$ and $r_h \sim \rewardf(x_h,a_h, x_{h+1})$ for all $h \in
    [\horizon]$, and all actions are chosen by the agent. We set
    $\rewardf(x_\horizon, a_\horizon, x_{\horizon+1}) =
    \rewardf(x_\horizon, a_\horizon)$ for all $x_\horizon, a_\horizon$
    as there is no $\obs_{\horizon+1}$.
We assume that for any trajectory $\sum_{h=1}^\horizon r_h \le 1$. The
agent \emph{does not} observe the states $\state_1, \ldots, \state_H$.
As notation, we often denote sequences using the ``$:$'' operator,
e.g., $s_{1:H} = (\state_1,\ldots,\state_H)$.

Without loss of generality, we partition $\states$ into subsets
$\states_1,\ldots,\states_H$, where $\states_h$ are the only states
reachable at time step $h$.
We similarly partition $\observations$ based on time step into
$\observations_1, \ldots, \observations_\horizon$. Formally,
$\transition(\cdot \mid \state, \action) \in \Delta(\states_{h+1})$
and $\emf(\state) \in \Delta(\observations_{h})$ when $\state \in
\states_h$.  This partitioning may be internal to the agent as we can
simply concatenate the time step to the states and observations.  Let
$\step: \observations \to [H]$ be the time step function, associating
an observation to the time point where it is reachable.

A \emph{policy} $\policy: \observations \to \prob(\actions)$ chooses actions
on the basis of observations and defines a distribution over
trajectories. We use $\EE_\pi[\cdot], \PP_\pi[\cdot]$ to denote
expectation and probability with respect to this distribution. 
The goal of the agent is to find a policy that achieves high expected
reward. We define the value function and policy value as:
\begin{align*}
\forall h \in [H], \state \in \states_h: ~ V(s; \pi) \defeq \EE_\pi\sbr{\sum_{h'=h}^H r_{h'} \mid s_h = s}, \qquad V(\pi) \defeq \EE_{\pi}\sbr{\sum_{h=1}^H r_h} = \EE_{s_1 \sim \mu}\sbr{V(s_1;\pi)}.
\end{align*}
As the observation space is extremely large, we consider a function
approximation setting, where the agent has access to a policy class
$\policies: \observations \to \prob(\actions)$. We further define the
class of non-stationary policies $\nspolicies \defeq
\policies^\horizon$ to enable the agent to use a different policy for
each time step: a policy $\policy = \policy_{1:\horizon} = (\policy_1, \ldots,
\policy_\horizon) \in \nspolicies$ takes action $\action_h$ according
to $\policy_h$.\footnote{We will often consider a $h$-step
  non-stationary policy $(\policy_1, \ldots, \policy_h) \in
  \policies^h$ when we only plan to execute this policy for $h$
  steps.} The optimal policy in this class is $\policy^\star \defeq
\argmax_{\policy \in \nspolicies}V(\policy)$, and our goal is to find
a policy with value close to the optimal value, $V(\pi^\star)$.

\paragraph{Environment assumptions.}
The key difference between Block MDPs and general Partially-Observed
MDPs is a disjointness assumption, which removes partial observability
effects and enables tractable learning.

\begin{assumption}
\label{assumption:disjoint-emission} 
The emission distributions for any two states $\state, \state' \in \states$ are
disjoint, that is $\textrm{supp}(\emf(s)) \cap \textrm{supp}(\emf(s')) = \emptyset$
whenever $\state \ne \state'$.
\end{assumption}
This disjointness assumption is used by \citet{du2019provably}, who argue that it
is a natural fit for visual grid-world scenarios  such as in~\pref{fig:mdp-example-right}, which are common in empirical RL research.  The name ``Block MDP'' arises since each hidden state $\state$ emits observations from a disjoint
block $\observations_s \subseteq \observations$. 
The assumption allows us to define an \emph{inverse mapping}
$\decoder: \observations \to \states$ such that for each $\state \in
\states$ and $\obs \in \textrm{supp}(\emf(s))$, we have $\decoder(x) =
s$. The agent \emph{does not} have access to
$\decoder$. 

Apart from disjointness, the main environment assumption is that
states are reachable with reasonable probability. To formalize this,
we define a \emph{maximum visitation probability} and
\emph{reachability parameter}:
\begin{align*}
\eta(\state) \defeq \max_{\pi \in \nspolicies} \PP_\pi\sbr{s}, \qquad \etamin = \min_{\state \in \states} \eta(s).
\end{align*}
Here $\PP_\pi[s]$ is the probability of visiting $s$ along the
trajectory taken by $\pi$.  As in~\citet{du2019provably}, our sample
complexity scales polynomially with $\etamin^{-1}$, so this quantity
should be reasonably large. In contrast with prior
work~\citep{du2019provably, dann2018oracle}, we do not require any
further identifiability or determinism assumptions on the environment.

We call the policies that visit a particular state with maximum
probability \emph{homing policies}.

\begin{definition}[Homing Policy]\label{def:homing} 
The homing policy for an observation $\obs \in \observations$ is
$\policy_{\obs} = \arg\max_{\policy \in \nspolicies}
\PP_\policy\sbr{\obs}$. Similarly, the homing policy for a state
$\state \in \states$ is $\policy_{\state} \defeq \argmax_{\policy \in
  \nspolicies} \PP_\pi\sbr{s}$.
\end{definition}

In~\pref{app:appendix-homing}, we prove some interesting
properties for these policies.  One key property is their
\emph{non-compositional} nature. We cannot extend homing policies for
states in $\states_{h-1}$ to find homing policies for states in
$\states_h$. For example, for the Block MDP
in~\pref{fig:mdp-example-left}, the homing policy for $\state_5$ takes
action $\action_1$ in $\state_1$ but the homing policies for
$\state_2, \state_3$, and $\state_4$ do not take action $\action_1$ in
$\state_1$. Non-compositionality implies that we must take a global
policy optimization approach for learning homing policies, which we
will do in the sequel.

\paragraph{Reward-free learning.}
In addition to reward-sensitive learning, where the goal is to
identify a policy with near-optimal value $V(\pi)$, we also consider a
reward-free objective.  In this setting, the goal is to find a small
set of policies that can be used to visit the entire state space. If we had access to the set of homing policies for every state then this set would suffice. However, in practice we can only hope to learn an approximation. We capture this idea with an $\alpha$-\emph{policy cover}.

\begin{definition}[Policy Cover]\label{def:policy-cover}
A finite set of non-stationary policies $\Psi$ is called an $\alpha$-policy cover if
for every state $\state \in \states$ we have $\max_{\policy \in
  \Psi}\PP_\pi\sbr{s} \ge \alpha \eta(\state)$.
\end{definition}
Intuitively, we hope to find a policy cover $\Psi$ of size
$O(|\states|)$. By executing each policy in turn, we can collect a
dataset of observations and rewards from all states at which point it
is relatively straightforward to maximize any
reward~\citep{kakade2002approximately,munos2003error,bagnell2004policy,munos2008finite,antos2008learning,farahmand2010error,chen2019information,agarwal2019optimality}. Thus,
constructing a policy cover can be viewed as an intermediate objective
that facilitates reward sensitive learning.

\begin{figure}%
    \centering
    \subfloat[]{{\includegraphics[height=.24\linewidth]{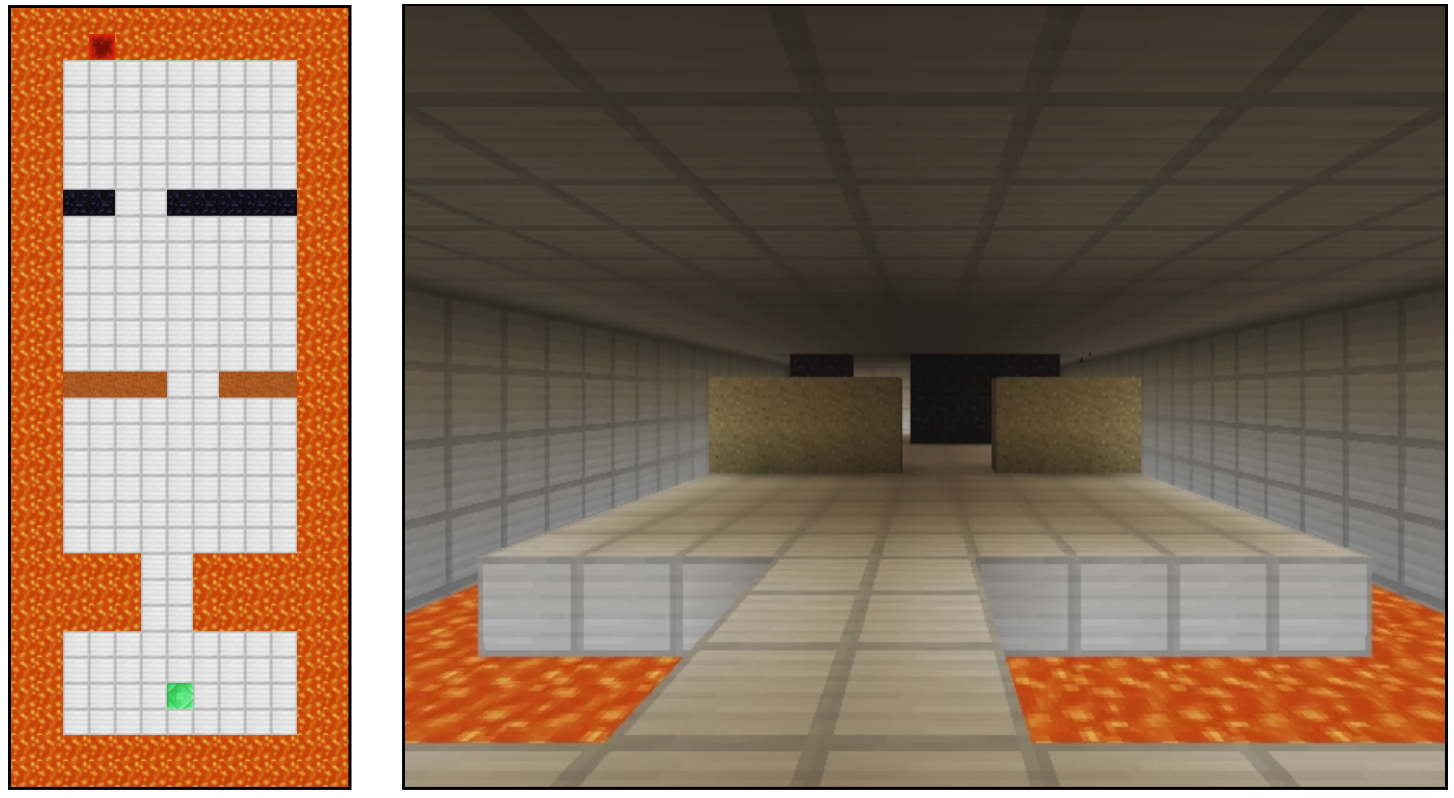} }\label{fig:mdp-example-right}}%
    \qquad
    \subfloat[]{{\includegraphics[height=.24\linewidth]{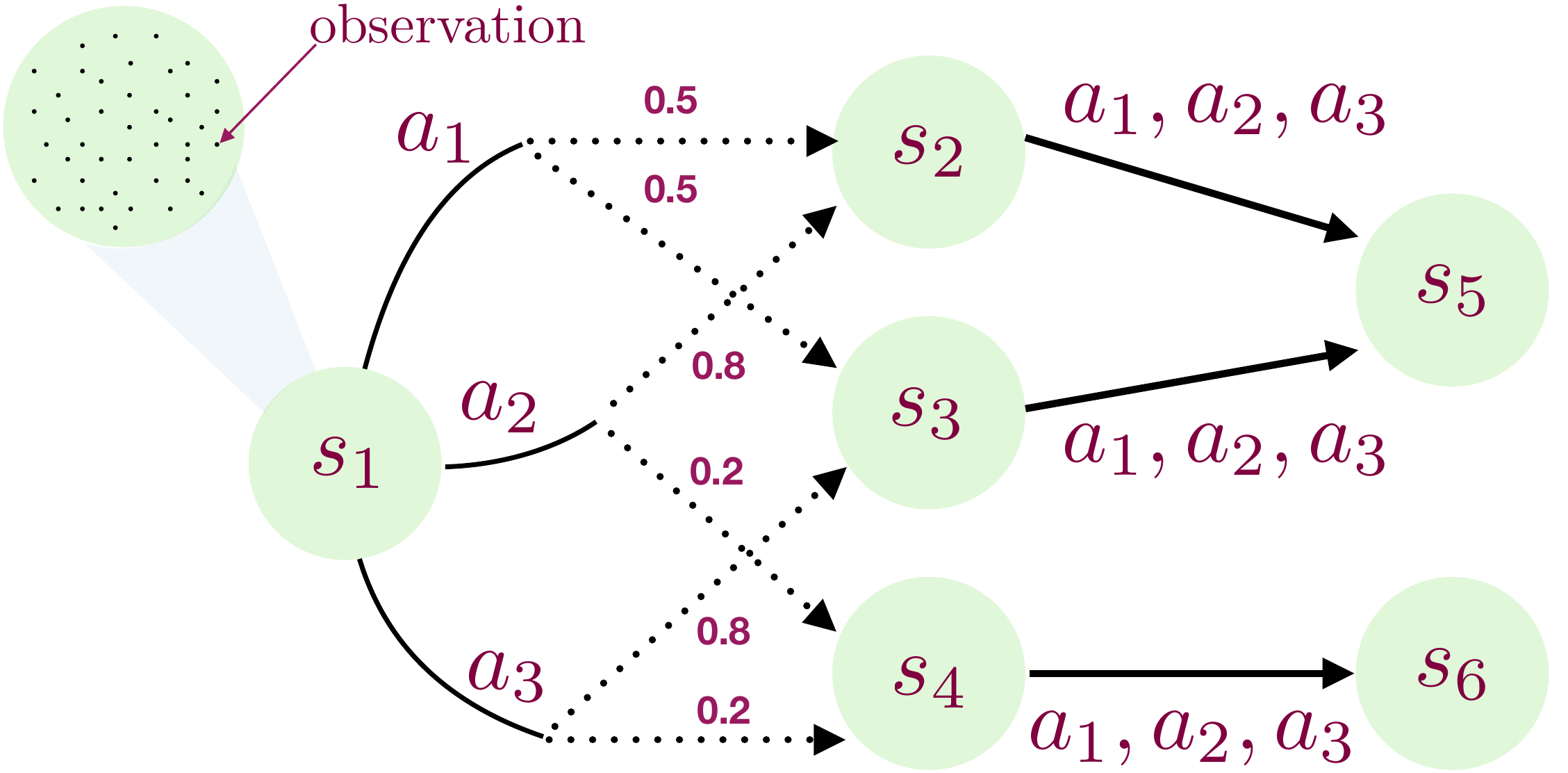}}\label{fig:mdp-example-left}}%
\vspace{-0.25cm}
    \caption{Two example Block MDPs. \textbf{Left:} An example from
      the Minecraft domain~\cite{johnson2016malmo}. The agent's
      observation is given by the history of all observed raw
      images. The grid on the left shows the latent state space
      structure. \textbf{Right:} The agent starts deterministically in
      $s_1$ and can take three different actions $\actions=\{a_1, a_2,
      a_3\}$. Dashed lines denote stochastic transitions while solid
      lines are deterministic. The numbers on each dashed arrow depict
      the transition probabilities.  We do not show observations for
      every state for brevity.}%
\label{fig:mdp-example}
\end{figure}

\paragraph{Function classes.}
In the Block MDP setting we are considering, the agent may never see
the same observation twice, so it must use function approximation to
generalize across observations. Our algorithm uses two function
classes. The first is simply the policy class $\Pi: \observations
\mapsto \prob(\actions)$, which was used above to define the optimal
value and the maximum visitation probabilities. For a simpler
analysis, we assume $\policies$ is finite and measure statistical
complexity via $\ln |\policies|$. However, our results only involve
standard uniform convergence arguments, so extensions to infinite
classes with other statistical complexity notions is
straightforward. We emphasize that $\policies$ is typically not fully
expressive.

We also use a family $\Fcal_N$ of regression functions with a specific
form.  To define $\Fcal_N$, first define $\Phi_N: \observations \to
[N]$ which maps observations into $N$ discrete abstract states.
Second, define $\Wcal_N: [N]\times\actions\times[N] \to [0,1]$ as
another ``tabular'' regressor class which consists of \emph{all}
functions of the specified type. Then, we set $\Fcal_N \defeq \{
(x,a,x') \mapsto w(\phi^{(\textup{F})}(x),a,\phi^{(\textup{B})}(x')):
w \in \Wcal_N, \phi^{(\textup{F})},\phi^{(\textup{B})} \in \Phi_N\}$
and $\Fcal \defeq \cup_{N \in \NN} \Fcal_N$. For the analysis, we
assume that $|\Phi_N|$ is finite and our bounds scale with $\log
|\Phi_N|$, which allows us to search over an exponentially large space
of abstraction functions. As $\Wcal_N$ is all functions over a
discrete domain, it has pointwise entropy growth rate of
$N^2|\actions|$ (see~\pref{app:supporting} for a formal definition),
and these two considerations determine the complexity of the
regression class $\Fcal_N$. As above, we remark that our results use
standard uniform convergence arguments, so it is straightforward to
extend to other notions.

\paragraph{Computational oracles.}
As we are working with large function classes, we consider an oracle
model of computation where we assume that these classes support
natural optimization primitives. This ``reductions'' approach
abstracts away computational issues and addresses the desired
situation where these classes are efficiently searchable. Note that
the oracle model provides no statistical benefit as the oracles can
always be implemented via enumeration; the model simply serves to guide
the design of practical algorithms.

Specifically, for the policy class $\policies$, we assume access to an
\emph{offline contextual bandit} optimization routine:
\begin{align*}
  \cboracle(D, \policies) \defeq \argmax_{\policy \in \policies} \sum_{(x,a,p,r) \in D} \EE_{a' \sim \pi(.|x)} \sbr{\frac{r \one\{a'=a\}}{p}}.
\end{align*}
This is a one-step importance weighted reward maximization problem,
which takes as input a dataset of $(x,a,p,r)$ quads, where $x \in
\observations$, $ a \in \actions$, $p \in [0,1]$ and $r \in \RR$ is
the reward for the action $a$, which was chosen with probability $p$.
This optimization arises in contextual bandit
settings~\citep{strehl2010learning,bottou2013counterfactual,swaminathan2015counterfactual},
and is routinely implemented via a further reduction to cost sensitive
classification~\cite{agarwal2014taming}.\footnote{We call this a
  contextual bandit oracle rather than a cost-sensitive classification
  oracle because the dataset is specified in contextual bandit format
  even though the oracle formally solves a cost-sensitive
  classification problem. The advantage is that in practice, we can
  leverage statistical improvements developed for contextual bandits,
  such as doubly-robust estimators~\cite{dudik2014doubly}. }

For the regression class $\Fcal_N$, we assume that we can solve
\emph{square loss minimization} problems:
\begin{align*}
  \cpeoracle(D, \Fcal_N) \defeq \argmin_{f \in \Fcal_N} \sum_{(x,a,x',y) \in D} (f(x,a,x') - y)^2.
\end{align*}
Here, the dataset consists of $(x,a,x',y)$ quads where $x,x' \in
\observations$, $a \in \actions$ and $y \in \{0,1\}$ is a binary
label. Square loss minimization is a standard optimization problem
arising in supervised learning, but note that our function class
$\Fcal_N$ is somewhat non-standard. In particular, even though square
loss regression is computationally tractable for convex classes, our
class $\Fcal$ is nonconvex as it involves quantization. On the other
hand, latent categorical models are widely used in
practice~\cite{jang2016categorical,hu2017learning}, which suggests
that these optimization problems are empirically tractable.

We emphasize that these oracle assumptions are purely computational
and simply guide the algorithm design.  In our experiments, we
instantiate both $\Pi$ and $\Fcal_N$ with neural networks, so both
oracles solve nonconvex problems. This nonconvexity does not hinder
the empirical effectiveness of the algorithm.

For running time calculations, we assume that a single call to
$\cboracle$ and $\cpeoracle$ with $n$ examples can be solved in
$\cpol(n)$ and $\creg(n)$ time, respectively.

\section{Kinematic Inseparability State Abstraction}
\label{sec:kinematic-inseparable}

The foundational concept for our approach is a new form of state
abstraction, called \emph{kinematic inseparability}. This abstraction
has two key properties. First, it can be learned via a reduction to
supervised learning, which we will discuss in detail
in~\pref{sec:learning-alg}. Second, it enables reward-free exploration
of the environment, studied in~\pref{sec:oracle-kinematic-insep}. In
this section, we present the key definitions, some interesting
properties, and some intuition.

For exploration, a coarser state abstraction, called \emph{backward
  kinematic inseparability}, is sufficient.

\begin{definition}[Backward Kinematic Inseparability]
\label{def:bki}
Two observations $\obs'_1, \obs'_2 \in \observations$ are \emph{backward
kinematically inseparable} (KI) if for all distributions $u \in
\Delta(\observations \times \actions)$ supported on
$\observations \times \actions$ and $\forall \obs \in \observations,
\action \in \actions$ we have
\begin{equation*}
\PP_{u}(\obs, \action \mid \obs'_1) = \PP_{u}(\obs, \action \mid \obs'_2), \qquad \mbox{where } \PP_{u}(\obs, \action \mid \obs') \defeq \frac{T(\obs' \mid \obs, \action)u(\obs, \action)}{\sum_{\tilde{\obs}, \tilde{\action}} T(\obs' \mid \tilde{\obs}, \tilde{\action}) u(\tilde{\obs}, \tilde{\action})}.
\end{equation*}
$\PP_{u}(\obs, \action \mid \obs')$ is the \emph{backward dynamics}
measuring the probability that the previous observation and action was
$(\obs, \action)$ given that the current observation is $\obs'$ and
the prior over $(\obs,\action)$ is $u$.
\end{definition}

The significance of backward KI is evident from the following lemma.
\begin{lemma} 
If $\obs_1, \obs_2$ are backward kinematic inseparable then for all
$\policy_1, \policy_2 \in \policies$ we have
$\frac{\PP_{\policy_1}(\obs_1)}{\PP_{\policy_2}(\obs_1)} =
\frac{\PP_{\policy_1}(\obs_2)}{\PP_{\policy_2}(\obs_2)}$.
\end{lemma}

The proof of this lemma and all mathematical statements in this paper
are deferred to the appendices. At a high level, the lemma shows that
backward KI observations induce the same ordering over policies with
respect to visitation probability. This property is useful for
exploration, since a policy that maximizes the probability of visiting
an abstract state, also maximizes the probability of visiting each
individual observation in that abstract state
\emph{simultaneously}. Thus, if we train a policy to visit backward KI
abstract state $i$ for each $i$ --- which we can do in an inductive
manner with synthetic reward functions, as we will see in the next
subsection --- we guarantee that \emph{all} observations are visited
with high probability, so we have a policy cover. In this way, this
lemma helps establish that a backward KI abstraction enables
sample-efficient reward-free exploration.

While backward KI is sufficient for exploration, it ignores the
forward dynamics, which are useful for learning a model. This
motivates the definition of forward kinematic inseparability.

\begin{definition}[Forward Kinematic Inseparability] 
\label{def:fki}
Two observations $\obs_1, \obs_2 \in \observations$ are \emph{forward
kinematically inseparable} (KI) if for every $\obs' \in \observations$
and $\action \in \actions$ we have
\begin{equation*}
T(\obs' \mid \obs_1, \action) = T(\obs' \mid \obs_2, \action).
\end{equation*}
\end{definition}

Finally, observations are \emph{kinematically inseparable} if they
satisfy both of these definitions.
\begin{definition}[Kinematic Inseparability]
\label{def:ki}
Two observations $\obs'_1, \obs'_2$ are \emph{kinematically
  inseparable} if for every distribution $u \in \Delta(\observations
\times \actions)$ with support over $\observations \times \actions$
and for every $\obs,\obs'' \in \observations$ and $\action,\action' \in \actions$ we
have
\begin{equation*}
\PP_{u}(\obs, \action \mid \obs'_1) = \PP_{u}(\obs, \action \mid \obs'_2) \,\,\,\, \mbox{ and } \,\,\,\, T(\obs'' \mid \obs'_1, \action') = T(\obs'' \mid \obs'_2, \action').
\end{equation*}
\end{definition}

It is straightforward to verify that all three of these notions are
equivalence relations, and hence they partition the observation
space. The \emph{backward kinematic inseparability dimension}, denoted
$\nbkid$, is the size of the coarsest partition generated by the
backward KI equivalence relation, with $\nfkid$ and $\nkid$ defined
similarly for the forward KI and KI relations. We also use mappings
$\phi^\star_B,\phi^\star_F,\phi^\star: \observations \rightarrow \NN$
to denote these abstractions, for example $\phi^\star_B(\obs_1) =
\phi^\star_B(\obs_2)$ if and only if $\obs_1$ and $\obs_2$ are
backward KI.

In ~\pref{app:appendix-kinematic-inseparability}, we collect and prove
several useful properties of these state abstractions. Importantly, we
show that for Block MDPs, observations emitted from the same state are
kinematically inseparable and, hence, $\max\{\nfkid, \nbkid\} \le
\nkid \le |\states|$. Ideally, we would like $\nkid = |\states|$ so
that the abstract states are in correspondence with the real states of
the environment, and we could recover the true model by learning the
dynamics between abstract states. However, we may have $\nkid <
|\states|$, but only in cases where the true state space is
unidentifiable from observations.~\pref{fig:mdp-canonical} depicts
such an example. From the left panel, if we split $s_2$ into two
states $s_{2,a}, s_{2,b}$ with the same forward dynamics and
proportional backward dynamics, we obtain the structure in the right
panel. Note that these two Block MDPs are indistinguishable from
observations, so we say that the simpler one is the \emph{canonical
  form}.

\begin{definition}[Canonical Form] 
A Block MDP is in \emph{canonical form} if $\forall \obs_1,\obs_2 \in
\observations$: $\decoder(\obs_1) = \decoder(\obs_2)$ if and only if
$\obs_1$ and $\obs_2$ are kinematically inseparable.
\end{definition}

\begin{figure}[t]
    \centering
    \subfloat[]{{\includegraphics[height=.15\linewidth]{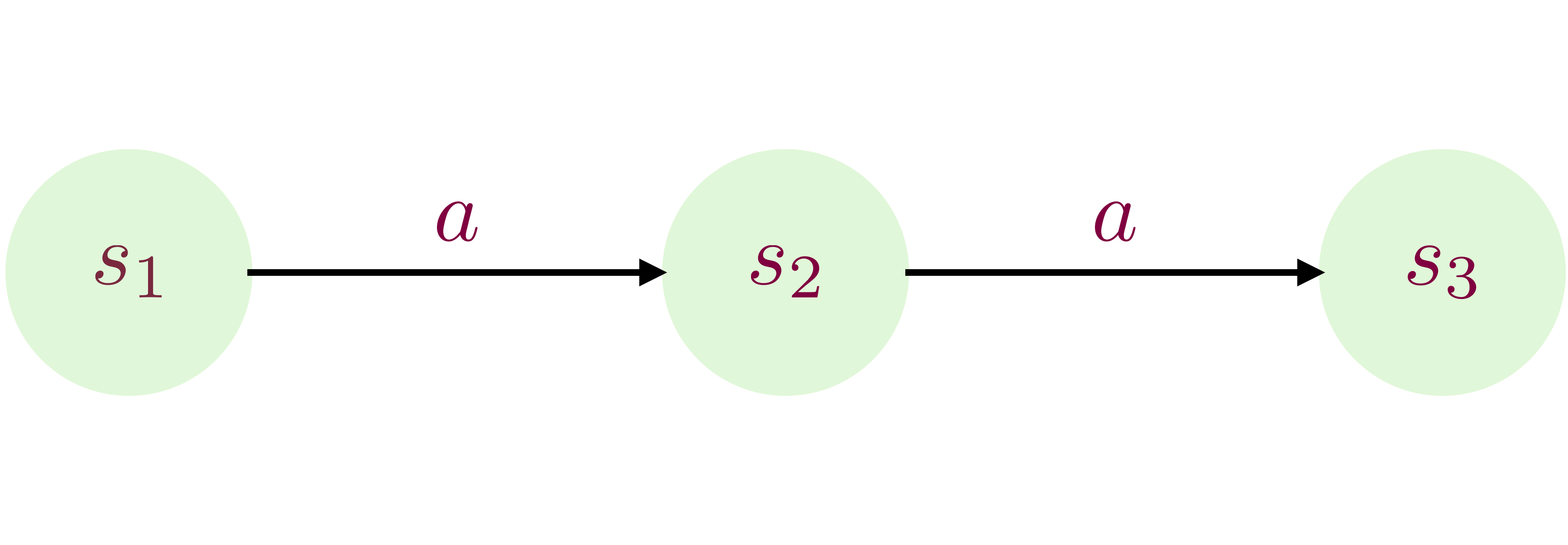} \label{fig:mdp-canonical-form}}}%
    \qquad
    \subfloat[]{{\includegraphics[height=.15\linewidth]{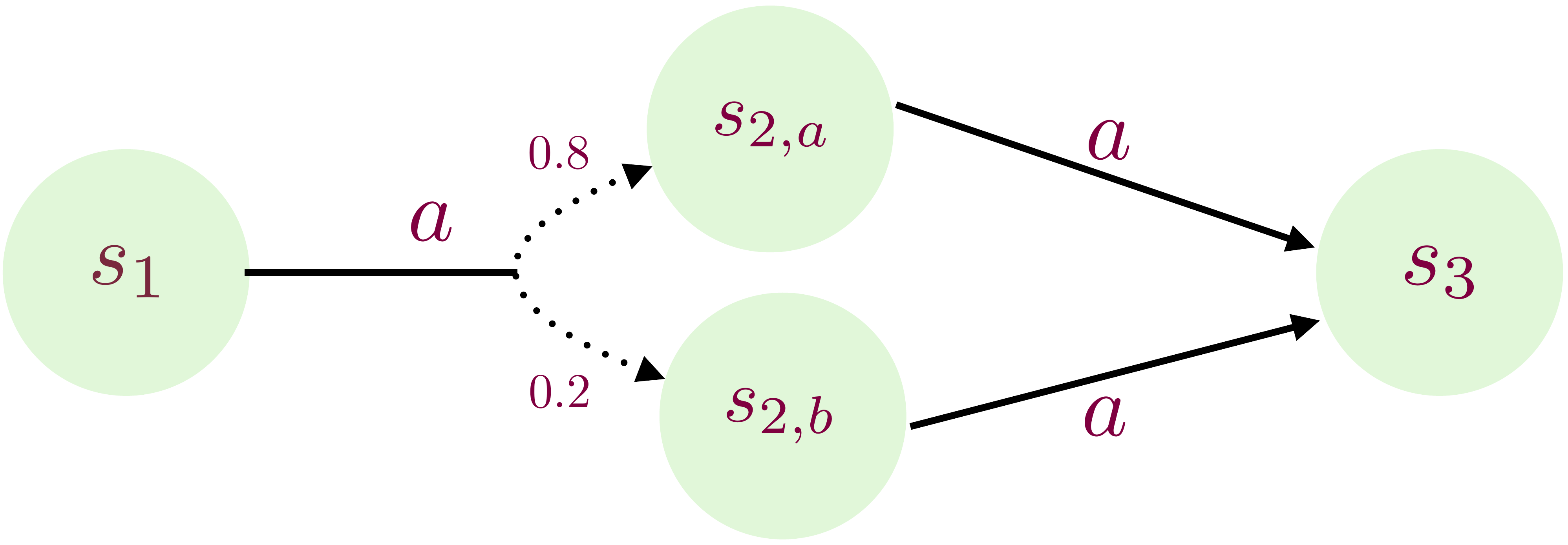} }\label{fig:mdp-non-canonical-form}}%
    \vspace{-0.25cm}
    \caption{\textbf{Left}: A Block MDP with 3 states and 1 action
      (observations are not depicted).  \textbf{Right}: We take the
      MDP on the left and treat $\state_2$ as two states: $\state_{2,
        a}$ and $\state_{2, b}$ where $\state_{2, a}$ contains
      observations that are emitted 80\% of the time from $\state_2$
      and $\state_{2, b}$ contains the rest. There is no way to
      distinguish between these two MDPs, and we call the left MDP the
      canonical form.}
\label{fig:mdp-canonical}
\end{figure}

Note that canonical form is simply a way to identify the notion of
state for a Block MDP. It does not restrict this class of environments
whatsoever.

\subsection{Exploration with an Oracle Kinematic Inseparability  Abstraction}
\label{sec:oracle-kinematic-insep}

We now show that backward KI enables reward-free strategic exploration
and reward-sensitive learning, by developing an algorithm,
$\oraclealg$, that assumes oracles access to a backward KI abstraction
function $\phi^\star: \observations \to [N]$. The pseudocode is
displayed in~\pref{alg:oracle-exp}.

$\oraclealg$ takes as input a policy class $\policies$, a backward KI
abstraction $\phi^\star$, and three hyperparameters $\epsilon, \delta,
\eta \in (0, 1)$. The hyperparameter $\eta$ is an estimate of the
reachability parameter (we want $\eta \leq \etamin$), while
$\epsilon,\delta$ are the standard PAC parameters.
The algorithm operates in two phases: a reward-free phase in which it
learns a policy cover
(\pref{alg:oracle-exp},~\pref{line:oracle-exp-loop-start}-\pref{line:oracle-exp-loop-end})
and a reward-sensitive phase where it learns a near-optimal policy for
the given reward function
(\pref{alg:oracle-exp},~\pref{line:oracle-reward-sensitive}). In the
reward-free phase, the algorithm proceeds inductively, using a policy
cover for time steps $1,\ldots,h-1$ to learn a policy cover for time
step $h$. In the $h^{\textrm{th}}$ iteration, we define $N$
\emph{internal reward functions} $\{R_{i,h}\}_{i=1}^N$ corresponding
to each output of $\phi^\star$.
The reward function $R_{i,h}$ gives a reward of $1$ if the agent
observes $\obs'$ at time step $h$ satisfying $\phi^\star(\obs') = i$
and $0$ otherwise. The internal reward functions incentivize the agent
to reach different backward KI abstract states.

\begin{algorithm}[t]
\caption{$\oraclealg(\policies, \phi^\star, \epsilon, \delta, \eta)$. Reinforcement learning in a Block MDP with oracle access to a Backward KI abstraction $\phi^\star: \observations \rightarrow [N]$}
\label{alg:oracle-exp}
\begin{algorithmic}[1]\onehalfspacing
\State Set $\npol = \tilde{\Ocal}\left( \frac{N^4\horizon^2 |\actions|}{\eta^2} \ln\left(\frac{|\policies|}{\delta}\right) \right)$,\,\,  $ \neval = \tilde{\Ocal}\left( \frac{N^2\horizon^2 |\actions|}{\epsilon^2} \ln\left(\frac{|\policies|}{\delta}\right) \right)$ and $\Psi_{1:H} = \emptyset$
\For{$h=2,\ldots, \horizon$} \label{line:oracle-exp-loop-start}
\For{$i = 1$ to $N$} \label{line:oracle-exp-policy-start}
\State Define $\rewardf_{i, h}(\obs, \action, \obs') \defeq \one\{\step(\obs') = h \land \phi^\star(\obs') = i \}$~\hfill \mycomment{Define internal reward functions} \label{line:oracle-reward}
\State $\policy_{i, h} \gets \psdpalg(\Psi_{1:h-1}, \rewardf_{i, h}, h-1, \policies, \npolo)$~\hfill \mycomment{Learn an exploration policy using $\psdpalg$}  \label{line:oracle-psdp}
\State $\Psi_h \gets \Psi_h \cup \{\policy_{i, h}\}$  
\EndFor \label{line:oracle-exp-policy-end}
\EndFor  \label{line:oracle-exp-loop-end}
\State $\hat{\policy}  \gets \psdpalg(\Psi_{1:\horizon}, \rewardf, \horizon, \policies, \nevalo)$~\hfill \mycomment{Reward Sensitive Learning}\label{line:oracle-reward-sensitive}
\State \textbf{return} $\hat{\policy}$, $\Psi_{1:\horizon}$\label{line:oracle-return}
\end{algorithmic}
\end{algorithm}

\begin{algorithm}[t]
\caption{$\psdpalg$($\Psi_{1:h}, \rewardf', h, \policies, n$). Optimizing reward function $\rewardf'$ given policy covers $\Psi_{1:h}$}
\label{alg:rl_via_homing_policies}
\begin{algorithmic}[1]\onehalfspacing
\For{$t=h, h-1,\cdots, 1$}\label{line:psdp-horizon-for-loop-start}
\State $D=\emptyset$
\For{$n$ times}\label{line:psdp-dataset-start}
\State $(\obs, \action, p, r) \sim \unf(\Psi_t) \circ \unf(\actions) \circ \hat{\policy}_{t+1} \circ \cdots \circ \hat{\policy}_{h}$\label{line:psdp-sampling-proc} ~\hfill\mycomment{$\psdpalg$ sampling procedure (see text)}
\State $D \leftarrow \{(\obs, \action, p, r)\} \cup D$\label{line:psdp-single-sample}
\EndFor\label{line:psdp-dataset-end}
\State $\hat{\policy}_t \gets \cboracle(D, \policies)$~\hfill \mycomment{Solve a contextual bandit problem given dataset $D$}\label{line:psdp-csc-call}
\EndFor\label{line:psdp-horizon-for-loop-end}
\State \textbf{return} $(\hat{\policy}_1, \hat{\policy}_2, \cdots, \hat{\policy}_h)$
\end{algorithmic}
\end{algorithm}

We find a policy that optimizes this internal reward function using
the subroutine $\psdpalg$, displayed
in~\pref{alg:rl_via_homing_policies}. This subroutine is based on
Policy Search by Dynamic Programming~\citep{bagnell2004policy}, which,
using an exploratory data-collection policy, optimizes a reward
function by solving a sequence of contextual bandit
problems~\citep{langford2008epoch} in a dynamic programming fashion.
In our case, we use the policy covers for time steps $1,\ldots,h-1$ to
construct the exploratory policy
(\pref{alg:rl_via_homing_policies},~\pref{line:psdp-sampling-proc}).

Formally, at time step $t$ of $\psdpalg$, we solve the following
optimization problem
\begin{equation*}
\max_{\policy \in \policies}  \EE_{\obs_t \sim D_t, a_t \sim \policy, a_{t+1:h} \sim \hat{\policy}_{t+1:h}} \left[\sum_{h'=t}^h R'(x_{h'}, a_{h'}, x_{h'+1})\right],
\end{equation*}
using the previously computed solutions $(\hat{\policy}_{t+1}, \cdots,
\hat{\policy}_h)$ for future time steps. The context distribution
$D_t$ is obtained by uniformly sampling a policy in $\Psi_t$ and
rolling-in with it until time step $t$. To solve this problem, we
first collect a dataset $D$ of tuples $(x,a,p,r)$ of size $n$ by (1)
sampling $x$ by rolling-in with a uniformly selected policy in
$\Psi_t$ until time step $t$, (2) taking action $a$ uniformly at
random, (3) setting $p \defeq 1/|\actions|$, and (4) executing
$\hat{\pi}_{t+1:h}$, and (5) setting $r \defeq \sum_{h'=t}^h
r_{t'}$. Then we invoke the contextual bandit oracle $\cboracle$ with
dataset $D$ to obtain $\hat{\pi}_t$. Repeating this process we obtain
the non-stationary policy $\hat{\pi}_{1:h}$ returned by $\psdpalg$.

The learned policy cover $\Psi_h$ for time step $h$ is simply the
policies identified by optimizing each of the $N$ internal reward
functions $\{R_{i,h}\}_{i=1}^N$. Once we find the policy covers
$\Psi_1,\ldots,\Psi_H$, we perform reward-sensitive learning via a
single invocation of $\psdpalg$ using the \emph{external} reward
function $R$
(\pref{alg:oracle-exp},~\pref{line:oracle-reward-sensitive}). Of
course, in a purely reward free setting, we can omit this last step
and simply return the policy covers $\Psi_{1:H}$.

\section{Learning Kinematic Inseparability for Strategic Exploration} 
\label{sec:learning-alg}

We now present our main algorithm, $\homingalg$, displayed
in~\pref{alg:learn_homing_policy}, which learns a kinematic
inseparability abstraction while performing reward-free strategic
exploration. Given hypothesis classes $\policies$ and $\Fcal$, a
positive integer $N$, and three hyperparameters $\eta, \epsilon,
\delta \in (0, 1)$, $\homingalg$ learns a policy cover of size $N$ and
a state abstraction function for each time step. We assume that $N \ge
\nkid$ and $\eta \le \etamin$ for our theoretical analysis.

\begin{algorithm}[t]
\caption{$\homingalg(\policies, \Fcal, N, \eta, \epsilon, \delta)$. Reinforcement and abstraction learning in a Block MDP.} %
\label{alg:learn_homing_policy}
\begin{algorithmic}[1]\onehalfspacing
\State Set $\nreg = \otil\left(\frac{N^6 |\actions|^3}{\eta^3}\left(N^2|\actions| + \ln\left(\frac{|\Phi_N|\horizon}{\delta}\right) \right)\right)$, $\npol =\otil\left( \frac{N^4\horizon^2 |\actions|}{\eta^2} \ln\left(\frac{|\policies|}{\delta}\right) \right)$, and $\Psi_{1:H} = \emptyset$ \label{line:homer-initialization}
\For{$h=2,\ldots, H$} \label{line:homer-time-step-start}
\State $D = \emptyset$
\For{$\nreg$ times}  \label{line:homer-dataset-step-start}\tikzmark{top-abstraction}
\State $(\obs_1, \action_1, \obs'_1), (\obs_2, \action_2, \obs'_2) \sim \unf(\Psi_{h-1}) \circ \unf(\actions)$  \label{line:homer-sampling-procedure}\hspace*{3.8cm}\tikzmark{right-abstraction}
\State $y \sim \textrm{Ber}(\nicefrac{1}{2})$\label{line:homer-sample-add-start}
\If{$y=1$} 
\State $D \gets D \cup \{([\obs_1,\action_1,\obs'_1], 1)\}$, \hspace*{0.35cm}\mycomment{Add a real transition}
\Else 
\State $D \gets D \cup \{([\obs_1, \action_1, \obs'_2], 0)\}$.\label{line:homer-sample-add-end} \hspace*{0.35cm}\mycomment{Add an imposter transition}
\EndIf \label{line:homer-single-end}
\EndFor \label{line:homer-dataset-step-end}
\State $(w_h, \forabs, \backabs) \gets \cpeoracle(\Fcal_N, D)$ \hspace*{0.9cm}\mycomment{Supervised learning on $D$}\label{line:homer-learn-encoding-function}\tikzmark{bottom-abstraction}
\For{$i = 1$ to $N$} \label{line:homer-learn-policy-start} \tikzmark{top-explore}
\State Define $\rewardf_{i, h}(\obs, \action, \obs') \defeq \one\{\step(\obs') = h \land \backabs(\obs') = i ]$  \hspace*{3.3cm}\tikzmark{right-explore} \label{line:homer-reward}
\State $\policy_{i, h} \gets \psdpalg(\Psi_{1:h-1}, \rewardf_{i, h}, h - 1, \policies, \npol)$ \label{line:homer-policy-c-ver-psdp}
\State $\Psi_h \gets \Psi_h \cup \{\policy_{i, h}\}$  
\EndFor \label{line:homer-learn-policy-end} \tikzmark{bottom-explore}
\EndFor  \label{line:homer-time-step-end}
\State Set $\neval = \otil\left(\frac{N^2\horizon^2|\actions|}{\epsilon^2}
\ln\left(\frac{|\policies|}{\delta}\right)\right)$\label{line:homer-neval}
\State $\hat{\policy}  \gets \psdpalg(\Psi_{1:\horizon}, \rewardf, \horizon, \policies, \neval)$ \hspace*{1.5cm}\mycomment{Reward Sensitive Learning}\label{line:homer-reward-sensitive}
\State \textbf{return} $\hat{\policy}$, $\Psi_{1:\horizon}$, $\hat{\phi}^{(\textup{F})}_{1:\horizon-1}$, $\hat{\phi}^{(\textup{B})}_{2:\horizon}$ \label{line:homer-return}
\end{algorithmic}
\AddNote{top-abstraction}{bottom-abstraction}{right-abstraction}{Learn  State Abstraction}%
\AddNote{top-explore}{bottom-explore}{right-explore}{Learn Policy Cover}%
\vspace{-0.4cm}
\end{algorithm}

The overall structure of $\homingalg$ is similar to $\oraclealg$ with
a reward-free phase preceding the reward-sensitive one. As with
$\oraclealg$, $\homingalg$ proceeds inductively, learning a policy
cover for observations reachable at time step $h$ given the learned
policy covers $\Psi_{1:h-1}$ for previous steps
(\pref{alg:learn_homing_policy},~\pref{line:homer-time-step-start}-\pref{line:homer-time-step-end}).

The key difference with $\oraclealg$ is that, for each iteration $h$,
we first learn an abstraction function $\backabs$ over
$\observations_h$. This is done using a form of contrastive estimation
and our function class $\Fcal_N$. Specifically in the
$h^{\textrm{th}}$ iteration, $\homingalg$ collects a dataset $D$ of
size $\nreg$ containing real and imposter transitions. We define a
sampling procedure: $(\obs, \action, \obs') \sim \unf(\Psi_{h-1})
\circ \unf(\actions)$ where $\obs$ is observed after rolling-in with a
uniformly sampled policy in $\Psi_{h-1}$ until time step $h-1$, action
$\action$ is taken uniformly at random, and $\obs'$ is sampled from
$T(\cdot \mid \obs,\action)$
(\pref{alg:learn_homing_policy},~\pref{line:homer-sampling-procedure}). We
sample two independent transitions $(\obs_1, \action_1, \obs'_1),
(\obs_2, \action_2, \obs'_2)$ using this procedure, and we also sample
a Bernoulli random variable $y\sim \textrm{Ber}(\nicefrac{1}{2})$. If
$y=1$ then we add the observed transition $([\obs_1, \action_1,
  \obs'_1], y)$ to $D$ and otherwise we add the \emph{imposter}
transition $([\obs_1, \action_1, \obs'_2], y)$
(\pref{alg:learn_homing_policy},~\pref{line:homer-sample-add-start}-\pref{line:homer-sample-add-end}).
Note that the imposter transition may or may not correspond feasible
environment outcomes.

We call the subroutine $\cpeoracle$ to solve the supervised learning
problem induced by $D$ with model family $\Fcal_N$
(\pref{alg:learn_homing_policy},~\pref{line:homer-learn-encoding-function}),
and we obtain a predictor $\hat{f}_h = (w_h,\forabs,\backabs)$.
As we will show, $\backabs$ is closely related to a backward KI
abstraction, so we can use it to drive exploration at time step $h$,
just as in $\oraclealg$.  Empirically, we will see that $\forabs$ is
closely related to a forward KI abstraction, which is useful for
auxiliary tasks such as learning the transition dynamics or
visualization.

With $\backabs$ the rest of the algorithm proceeds similarly to
$\oraclealg$. We invoke $\psdpalg$ with the $N$ internal reward
functions induced by $\backabs$ to find the policy cover
(\pref{alg:learn_homing_policy},~\pref{line:homer-policy-c-ver-psdp}).
Once we have found policy covers for all time steps, we perform
reward-sensitive optimization just as before
(\pref{alg:learn_homing_policy},~\pref{line:homer-reward-sensitive}).
As with $\oraclealg$ we can ignore the reward-sensitive phase and
operate in a purely reward-free setting by simply returning the policy
covers $\Psi_{1:H}$.

We combine the two abstractions as
$\overline{\phi}_h \defeq (\forabscurr, \backabs)$ to form the learned KI
abstraction, 
where for any $\obs_1, \obs_2 \in \observations$, $\overline{\phi}_h(\obs_1)
= \overline{\phi}_h(\obs_2)$ if and only if $\forabscurr(\obs_1) =
\forabscurr(\obs_2)$ and $\backabs(\obs_1) = \backabs(\obs_2)$. We
define $\hat{\phi}^{(\textup{B})}_1(\obs) \equiv 1$ and
$\hat{\phi}^{(\textup{F})}_\horizon \equiv 1$ as there is no backward
and forward dynamics information available at these steps,
respectively. Empirically, we use $\overline{\phi}$ for learning the
transition dynamics and visualization (see~\pref{sec:exp}).

\section{Theoretical Analysis} 
\label{sec:complexity-analysis}

Our main theoretical contribution is to show that both $\oraclealg$
and $\homingalg$ compute a policy cover and a near-optimal policy with
high probability in a sample-efficient and computationally-tractable
manner. The result requires an additional expressivity assumption on
our model classes $\policies$ and $\Fcal$, which we now state.

\begin{assumption}[Realizability]\label{assum:realizability}
Let $\Rcal \defeq \{R\} \cup \{ (\obs,\action,\obs') \mapsto \one\cbr{\phi(\obs')=i \land \step(\obs')=h} \mid
\phi \in \Phi_N, i \in [N], h\in [\horizon] \}$ be the set
of external and internal reward functions. We assume that $\policies$
satisfies \emph{policy completeness} for every $R' \in \Rcal$: for every $h \in [\horizon]$ and $\policy'
 \in \nspolicies$, there exists $\policy \in
\policies$ such that
\begin{align*}
\forall \obs \in \observations_h: ~~~ \policy(\obs) = \argmax_{\action \in \actions} \EE\sbr{\sum_{h'=h}^H r_{h'} \mid \obs_h = \obs, \action_h=\action, \action_{h'} \sim\policy'}.
\end{align*}
We also assume that $\Fcal$ is \emph{realizable}: for any $h
\in [H]$, $N \ge \nkid$, and any prior distribution $\rho \in \prob(\states_h)$ with
$\textrm{supp}(\rho) = \states_h$, there exists $f_\rho \in \Fcal_N$ such
that
\begin{align*}
\forall \obs_{h-1}, \action_{h-1}, \obs_h: \qquad f_\rho(\obs_{h-1}, \action_{h-1}, \obs_h) = \frac{T(\decoder(\obs_h) | \decoder(\obs_{h-1}), \action_{h-1})}{T(\decoder(\obs_h) | \decoder(x_{h-1}), \action_{h-1}) + \rho(\decoder(\obs_{h}))}.
\end{align*}
\end{assumption}

Completeness assumptions are common in the analysis of dynamic
programming style algorithms for the function approximation
setting~\cite{antos2008learning} (see Chen and
Jiang~\cite{chen2019information} for a detailed discussion). Our exact
completeness assumption appears in the work of \citet{dann2018oracle},
who use it to derive an efficient algorithm for a restricted version
of our setting with deterministic latent state transitions.
 
The realizability assumption on $\Fcal$ is adapted to our learning
approach: as we use $\Fcal$ to distinguish between real and imposter
transitions, $\Fcal$ should contain the Bayes optimal classifier for
these problems. In $\homingalg$, the sampling procedure
$\unf(\Psi_{h-1}) \circ \unf(\actions)$ that is used to collect data
for the learning problem in the $h^{th}$ iteration induces a marginal
distribution $\rho \in \Delta(\states_h)$ and the Bayes optimal
predictor for this problem is $f_\rho$ (See~\pref{lem:p-form}
in~\pref{app:appendix-sample-complexity}). It is not hard to see that
if $x_1,x_2$ are kinematically inseparable then $f_\rho(x_1,a,x') =
f_\rho(x_2,a,x')$ and the same claim holds for the third argument of
$f_\rho$. Therefore, by the structure of $\Fcal_N$, a sufficient
condition to ensure realizability is that $\Phi_N$ contains a
kinematic inseparability abstraction, which is reasonable as this is
precisely what we are trying to learn.

\paragraph{Theoretical Guarantees.} 
We first present the guarantee for $\oraclealg$.

\begin{theorem}[$\oraclealg$ Result]
\label{thm:oracle-exp-statement}
For any Block MDP, given a backward KI
abstraction $\backwardabs: \observations \rightarrow [N]$ such that
$\backwardabs \in \Phi_N$, and parameters $(\epsilon,\eta,\delta) \in
(0,1)^3$ with $\eta \leq \etamin$, $\oraclealg$ outputs policy covers
$\cover_{1:H}$ and a reward sensitive policy $\hat{\policy}$ such that
the following holds, with probability at least $1-\delta$:
\begin{enumerate}
\item For each $h \in [H]$, $\cover_h$ is a $\nicefrac{1}{2}$-policy cover for $\states_h$;
\item $V(\hat{\policy}) \geq \max_{\policy \in \nspolicies} V(\policy) - \epsilon$.
\end{enumerate}
The sample complexity is $\otil\rbr{ NH^2\npol + H\neval}
= \otil\rbr{N^2 H^3|\actions|\log(|\policies|/\delta)
  \rbr{NH/\etamin^2 + 1/\epsilon^2}}$, and the running time is
$\otil\rbr{NH^3\npol + H^2 \neval + NH^2\cdot \cpol(\npol)
  + H\cdot \cpol(\neval)}$, where $\npol$ and $\neval$ are defined
in~\pref{alg:oracle-exp}.
\end{theorem}

\pref{thm:oracle-exp-statement} shows that we can learn a policy cover
and use it to learn a near-optimal policy, assuming access to a
backward KI abstraction. The sample complexity is
$\textrm{poly}(N,H,|\actions|,\etamin^{-1},\epsilon^{-1},\log
|\policies|/\delta)$, which at a coarse level is our desired
scaling. We have not attempted to optimize the exponents in the sample
complexity or running time.

We also remark that we may be able to achieve this oracle guarantee
with other algorithmic approaches. Two natural candidates are (1) a
model-based approach where we learn the dynamics models over the
backward KI abstract states and plan to visit abstract states, and (2)
a value-based approach like Fitted Q-Iteration with the same synthetic
reward structure as we use
here~\citep{munos2008finite,antos2008learning,farahmand2010error,chen2019information}. We
have not analyzed these approaches in our precise setting, and they
may actually be more sample efficient than our policy-based
approach. Despite this,~\pref{thm:oracle-exp-statement} suffices to
establish our main conceptual takeway: that a backward KI abstraction
can be used for sample efficient exploration and policy optimization.

We next present the result for $\homingalg$. Here we show that
$\homingalg$ achieves a similar guarantee to $\oraclealg$,
\emph{without} prior access to a backward KI abstraction. In other
words, $\homingalg$ provably learns a backward KI abstraction and uses
it for exploration and policy optimization.

\begin{theorem}[Main Result]\label{thm:main-theorem} 
For any Block MDP and hyperparameters $\epsilon, \delta, \eta \in (0,
1), N \in \NN$, satisfying $\eta \le \etamin$ and $N \ge \nkid$,
$\homingalg$ outputs exploration policies $\Psi_{1:H}$
and a reward sensitive policy $\hat{\policy}$ satisfying:
\begin{enumerate}
\item $\Psi_h$ is an $\nicefrac{1}{2}$-policy cover of $\states_h$ for every $h \in [\horizon]$
\item $V(\hat{\policy}) \ge \max_{\policy \in \nspolicies} V(\policy) -  \epsilon$
\end{enumerate}
with probability least $1-\delta$.  The sample complexity of
$\homingalg$ is $\Ocal\left(\npol N \horizon^3 + \nreg \horizon +
\neval \horizon\right)$ where $\npol,\nreg,\neval$ are as specified
in~\pref{alg:learn_homing_policy}. This evaluates to
\begin{align*}
\otil\rbr{ \frac{N^8|\actions|^4 H}{\etamin^3} +
  \frac{N^6|\actions|H}{\etamin^3}\ln(|\Phi_N|/\delta) +
  \rbr{\frac{N^5H^4|\actions|}{\etamin^2} +
    \frac{N^2H^3|\actions|}{\epsilon^2}}\ln(|\policies|/\delta)}.
\end{align*}
The running time is
\begin{equation*}
\order\rbr{ \npol N \horizon^3 + \nreg \horizon^2 + \neval \horizon^2 + \cpol(\npol) N\horizon^2 + \creg(\nreg)\horizon +  \cpol(\neval) \horizon }.
\end{equation*}
\end{theorem}

In
comparison with the guarantee for $\oraclealg$, the main qualitative
difference is that the guarantee for $\homingalg$ also has a
dependence on $\log |\Phi_N|$, which is natural as $\homingalg$
attempts to learn the backward KI function. Crucially, the logarithmic
dependence on $|\Phi_N|$ implies that we can take $\Phi_N$ to be
exponentially large and achieve a guarantee that is qualitatively
comparable to that of $\oraclealg$. This demonstrates that we can
learn a backward KI abstraction function from an exponentially large class
and then use it for exploration and policy optimization.

In terms of computation, the running time is polynomial in our oracle
model, where we assume we can solve contextual bandit problems over
$\Pi$ and regression problems over $\Fcal_N$. In~\pref{sec:exp}, we
will see that these problems can be solved effectively in practice. 

The closest related result is for the $\du$ algorithm
of~\citet{du2019provably}. $\du$ provably finds a policy cover in a
restricted class of Block MDPs in a sample- and
computationally-efficient manner. The precise details of the guarantee
differs from ours in several ways (e.g., additive versus
multiplicative error in policy cover definition, different
computational and expressivity assumptions), so the sample complexity
bounds are incomparable. However,~\pref{thm:main-theorem} represents a
significant conceptual advance as it eliminates the identifiability
assumptions required by $\du$ and therefore greatly increases the
scope for tractable reinforcement learning.

\paragraph{Why does $\homingalg$ learn kinematic inseparability?} 
A detailed proof of both theorems is deferred to~\pref{app:psdp-appendix}-\pref{app:appendix-sample-complexity}, but
for intuition, we provide a sketch of how $\homingalg$ learns a
kinematic inseparability abstraction. For this discussion only, we
focus on asymptotic behavior and ignore sampling issues.

Inductively, assume that $\Psi_{h-1}$ is a policy cover of
$\states_{h-1}$, let $D(x,a,x')$ be the marginal distribution over
real and imposter transitions sampled by $\homingalg$ in the
$h^{\textrm{th}}$ iteration
(\pref{line:homer-dataset-step-start}--\pref{line:homer-dataset-step-end}),
and let $\rho$ be the marginal distribution over
$\observations_{h}$. First observe that as $\Psi_{h-1}$ is a policy
cover, we have $\supp(D) =
\observations_{h-1}\times\actions\times\observations_h$, which is
crucial for our analysis. Let $\hat{f} =
(\hat{w},\forabs,\backabs)$ be the output of the
regression oracle $\cpeoracle$ in this iteration. The first
observation is that the Bayes optimal regressor for this problem is
$f_\rho$ defined in~\pref{assum:realizability}, and, with
realizability, in this asymptotic discussion we have $\hat{f} \equiv
f_\rho$.

Next, we show that for any two observations $x_1',x_2' \in
\observations_h$, if $\backabs(x_1') = \backabs(x_2')$ then
$x_1'$ and $x_2'$ are backward kinematically inseparable. If this
precondition holds, then
\begin{align*}
\forall x \in \observations_{h-1}, a \in \actions: ~~ f_\rho(x, a, x_1') = \hat{p}(\forabs(x), a, \backabs(x_1')) = \hat{p}(\forabs(x), a, \backabs(x_2')) = f_\rho(x, a, x_2').
\end{align*}
Then, by inspection of the form of $f_\rho$, we have
\begin{align*}
f_\rho(x,a,x_1') = f_\rho(x,a,x_2') \Leftrightarrow \frac{T(x_1' \mid x,a)}{\rho(x_1')} = \frac{T(x_2' \mid x,a)}{\rho(x_2')}.
\end{align*}
As this identity holds for all $x \in \observations_{h-1}, a \in
\actions$ and trivially when $x \not\in \observations_{h-1}$, it is
easy to see that $x_1',x_2'$ satisfy~\pref{def:bki}, so they are
backward KI. Formally, taking expectation with prior $u
\in \Delta(\observations, \actions)$, we have
\begin{align*}
\PP_u(\obs, \action \mid \obs'_1) = \frac{T(\obs'_1 \mid \obs, \action)u(\obs, \action)}{\sum_{\tilde{\obs}, \tilde{\action}} T(\obs'_1 \mid \tilde{\obs}, \tilde{\action}) u(\tilde{\obs}, \tilde{\action})} &= \frac{\frac{\rho(\obs'_1)}{\rho(\obs'_2)}T(\obs'_2 \mid \obs, \action)u(\obs, \action)}{\sum_{\tilde{\obs}, \tilde{\action}} \frac{\rho(\obs'_1)}{\rho(\obs'_2)}T(\obs'_2 \mid \tilde{\obs}, \tilde{\action}) u(\tilde{\obs}, \tilde{\action})}\\[.1cm]
 &= \frac{T(\obs'_2 \mid \obs, \action)u(\obs, \action)}{\sum_{\tilde{\obs}, \tilde{\action}} T(\obs'_2 \mid \tilde{\obs}, \tilde{\action}) u(\tilde{\obs}, \tilde{\action})} =\PP_u(\obs, \action \mid \obs'_2).
\end{align*}
This implies that $\backabs$ is a backward KI abstraction over
$\observations_h$.

\paragraph{Efficient Implementation of $\homingalg$.} 
As stated, the most computationally expensive component of
$\homingalg$ is the $\Ocal(N\horizon)$ calls to $\psdpalg$ for
learning the policy covers. This has a total computational cost of
$\Ocal(N\horizon^3\npol + \cpol(\npol) N\horizon^2)$, but in practice,
it can be significantly reduced. Empirically, we use two important
optimizations: First, we parallelize the $N$ calls to $\psdpalg$ for
optimizing the internal reward functions in each iteration of the
algorithm
(\pref{line:homer-learn-policy-start}--\pref{line:homer-learn-policy-end}). Second
and perhaps more significantly, we can attempt to find compositional
policies using a greedy search procedure ($\gpsalg$). Here, rather
than perform full dynamic programming to optimize reward $R_{i,h}$, we
find the policy $\hat{\policy}_h$ for the last time step, and then we
append this policy to the best one from our cover
$\Psi_{h-1}$. Formally, we compute $\hat{\pi}_{1:h-1} = \argmax_{\pi
  \in \Psi_{h-1}} V(\pi\circ\hat{\pi}_h; R_{i,h})$, where $V(\cdot;R)$
is the value function with respect to reward function $R$ and $\circ$
denotes policy composition. Then, since the optimal value with respect
to $R_{i,h}$ is at most $1$, we check if $V(\hat{\pi}_{1:h-1}\circ
\hat{\pi}_h; R_{i,h}) \geq 1-\epsilon$. If it is we need not perform a
full dynamic programming backup, otherwise we revert to
$\psdpalg$. $\gpsalg$ may succeed even though the policies we are
trying to find are non-compositional in general. In the favorable
situation where $\gpsalg$ succeeds, actually no further samples are
required, since we can re-use the real transitions from the regression
step, and we need only solve one contextual bandit problem, instead of
$H$. Empirically, both of these optimizations may yield significant
statistical and computational savings.

\section{Can We Use Existing State Abstraction Oracles?}
\label{sec:state-abstraction-failures}

Our analysis so far verifies the utility of the backward KI state
abstraction: it enables efficient exploration via $\oraclealg$, and it
can be learned using contrastive estimation procedure with a
bottleneck classifier as in $\homingalg$. Do other, previously
studied, state abstractions admit similar properties?

In this section, we discuss prior approaches for learning state
abstractions. In Block-MDPs, we show that these approaches fail to
find a policy cover when interleaved with a $\psdpalg$-style routine
used to find policies that visit the abstract states, following the
structure of $\homingalg$. Note that it may be possible to embed these
approaches in other algorithmic frameworks and successfully explore.

\begin{figure}[t]
\centering
\begin{minipage}{\textwidth}
\centering
\subfloat[]{
     \includegraphics[scale=0.27]{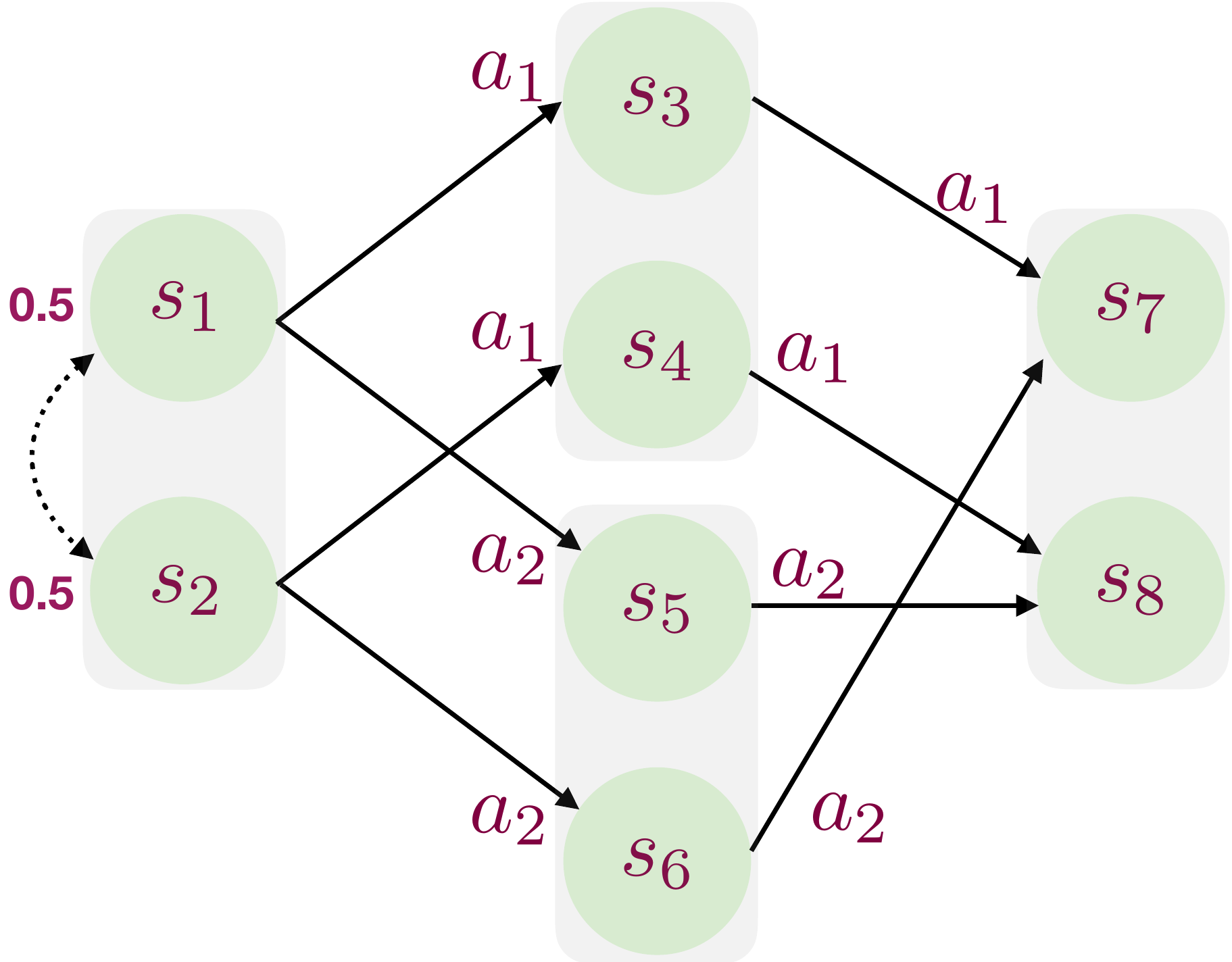}\label{fig:simon-deepak-counterexample}
}
\hspace{\stretch{1}}%
\subfloat[]{
     \includegraphics[scale=0.3]{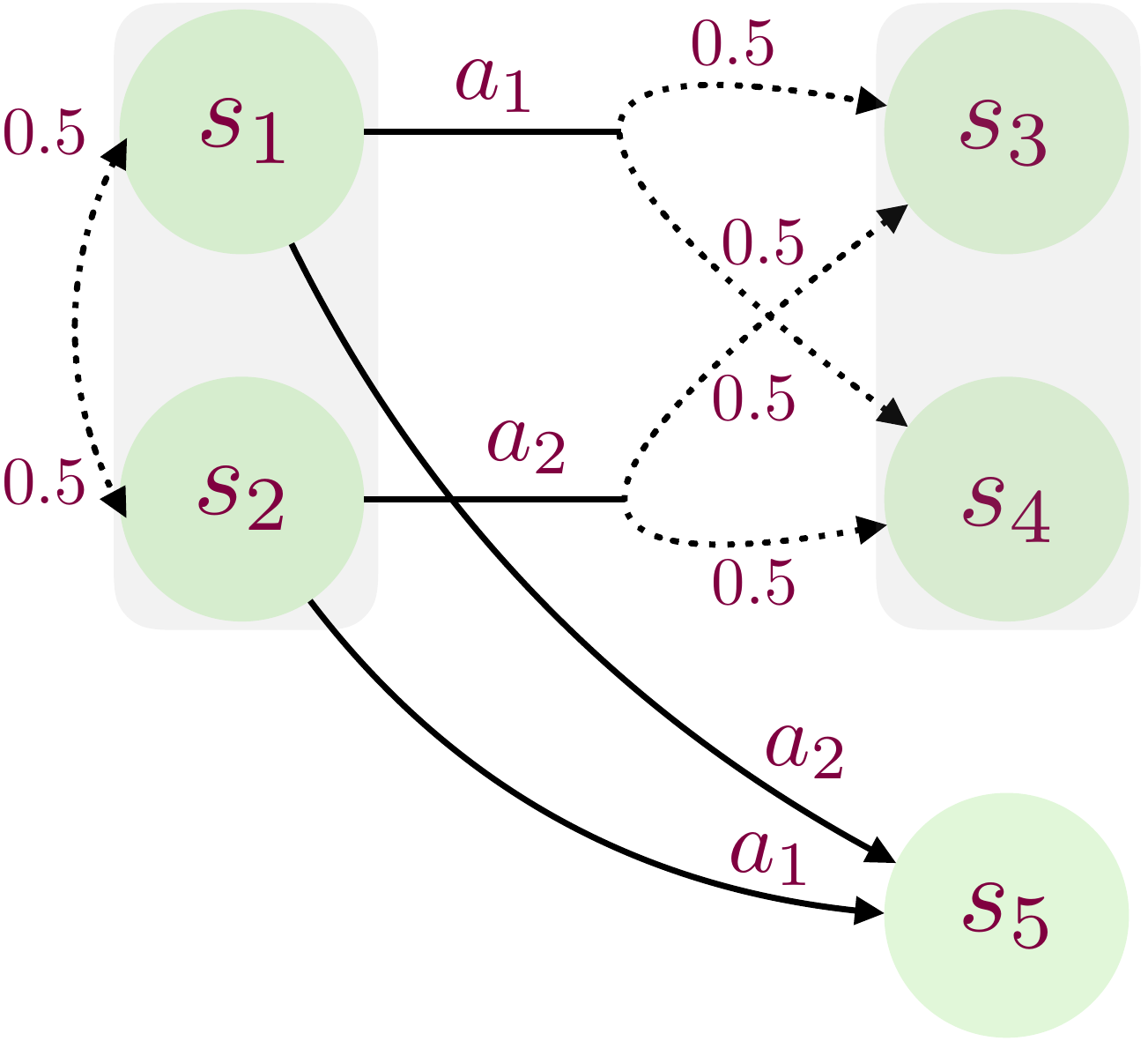}\label{fig:simon-alg-counterexample}
}
\hspace{\stretch{1}}%
\subfloat[]{
  \includegraphics[scale=0.2]{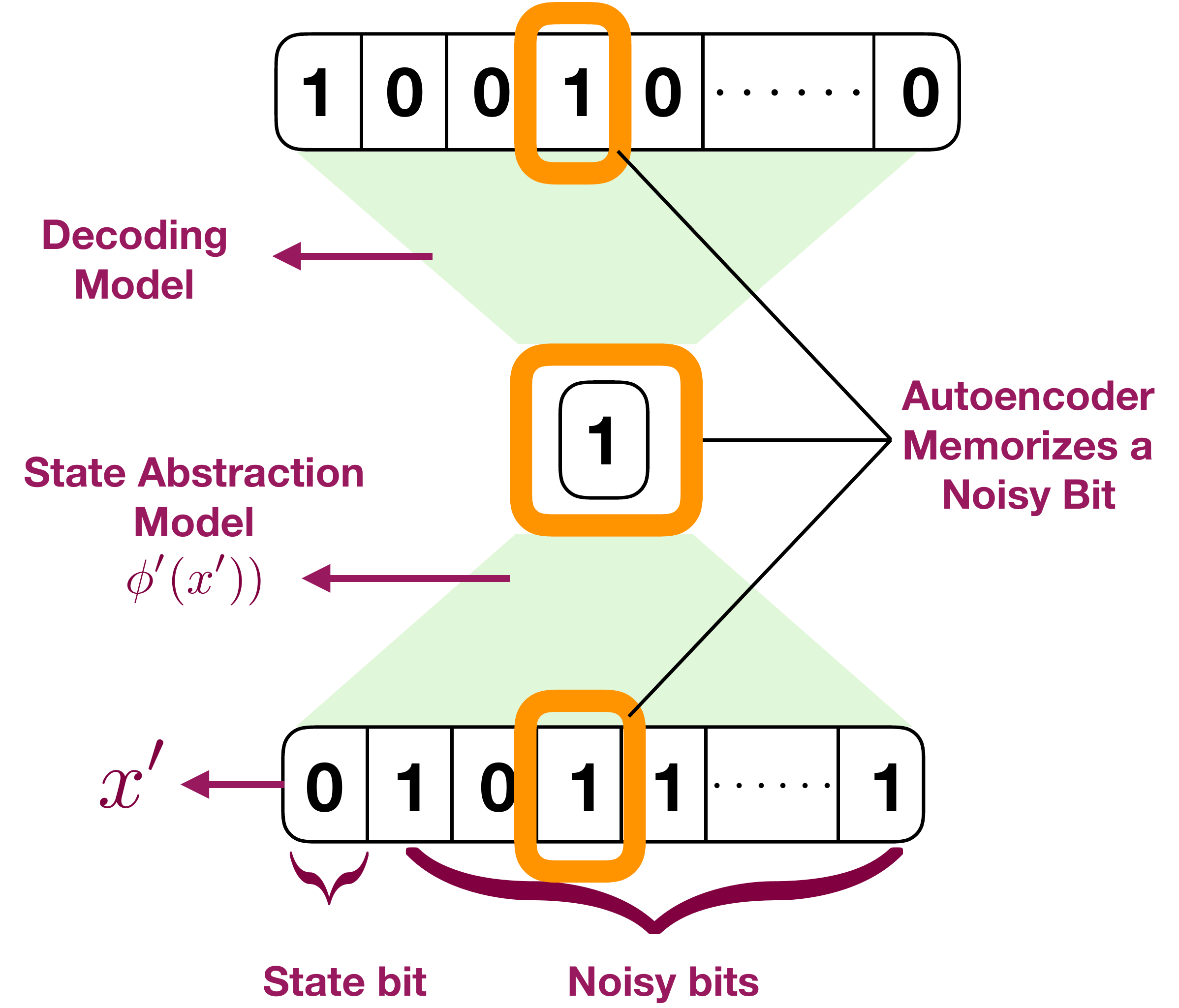}\label{fig:autoencoder-counterexample}
  }
\hspace{\stretch{1}}
\vspace{-0.3cm}
\caption{Counterexamples for prior work on abstraction/representation
  learning. We do not show observations for brevity.  \textbf{Left:} A
  Block MDP where predicting the previous action from
  observations~\cite{pathak2017curiosity} or predicting the previous
  abstract state and action fails~\cite{du2019provably}. 
\textbf{Middle:} A Block-MDP where the model-based algorithm
of~\citet{du2019provably} fails.
\textbf{Right:} Illustration of a failure mode for the autoencoding approach of~\citet{tang2017exploration}, where optimal reconstruction loss is attained by memorizing noise.
See text for more details.}
\label{fig:counterexamples}
\end{minipage}%
\end{figure}

\paragraph{Predicting Previous Action from Observations.}
Curiosity-based approaches learn a representation by predicting the
previous action from the previous and current
observation~\cite{pathak2017curiosity}. When embedded in a
$\psdpalg$-style routine, this approach fails to guarantee coverage of
the state space, as can be seen
in~\pref{fig:simon-deepak-counterexample}. A Bayes optimal predictor
of previous action $a$ given previous and current observations $x,x'$
collapses the observations generated from $\{s_3,s_4\}$,
$\{s_5,s_6\}$, and $\{s_7,s_8\}$ together. To see why, the agent can
only transition to $\{s_3,s_4\}$ by taking action $a_1$, so we can
perfectly predict the previous action even if all of the observations
from these states have the same representation. This also happens with
$\{s_5,s_6\}$ and $\{s_7,s_8\}$. Unfortunately, collapsing
observations from $\{s_7,s_8\}$ together creates an unfavorable
tie-breaking scenario when trying to find a policy to visit this
representation. For example, the policy that takes action $a_1$ in
$s_1$ and $s_3$ and $a_2$ in $s_2$ and $s_6$ deterministically reaches
$s_7$, so it visits this representation maximally, but it never visits
$s_8$. So this approach for representation learning, interleaved with
policy optimization, does not yield a policy cover.

\paragraph{Predicting Previous Action and Abstract State.} 
Instead of predicting the previous action,~\citet{du2019provably}
learn a representation by predicting the previous action \emph{and}
previous abstract state. As they show, this approach provably explores
a restricted class of Block-MDPs, but unfortunately it fails in the
general case. For example in~\pref{fig:simon-deepak-counterexample}, a
Bayes optimal predictor collapses observations from $\{s_1,s_2\},
\{s_3,s_4\}, \{s_5,s_6\}$, and $\{s_7,s_8\}$, leading to the same
failure for policy optimization as the curiosity-based approach. This
state collapse is caused by a stochastic start state; $\{s_1,s_2\}$
cannot be separated by this approach and using the joint
representation for $\{s_1,s_2\}$ as a prediction target causes a
cascading failure. Note that~\citet{du2019provably} assume a
deterministic start state in their analysis.

Instead of a $\psdpalg$-style routine, \citet{du2019provably} use a
model-based approach for building a policy cover, where the learned
policies operate directly on the abstract states. Actually this
approach avoids the tie-breaking issue in policy optimization and does
succeed for the example in~\pref{fig:simon-deepak-counterexample}, but
it fails in~\pref{fig:simon-alg-counterexample}. If policies are
defined over abstract states, we must take the same action in
$s_1$ and $s_2$ (as this approach can never separate a stochastic
start state), so we can reach $\{s_3,s_4\}$ with probability at most
$\nicefrac{1}{2}$, while a policy operating directly on observations
could reach these states with probability $1$. Chaining this
construction together shows that this approach can at best find an $\alpha$-policy cover where $\alpha$ is exponentially small
in the horizon.

\paragraph{Training Autoencoders.} 
The final approach uses an autoencoder to learn a representation,
similar to~\citet{tang2017exploration}. Here we representation $\phi$
and decoder $U$ by minimizing reconstruction loss $\texttt{dist}(x,
U(\phi(x)))$ over a training set of raw observations, where
$\texttt{dist}$ is a domain-specific distance
function.~\pref{fig:autoencoder-counterexample} shows that this
approach may fail to learn a meaningful representation altogether. The
problem contains just two states and the observations are
$d$-dimensional binary vectors, where the first bit encodes the state
and the remaining bits are sampled from
$\textrm{Ber}(\nicefrac{1}{2})$ (it is easy to see that this is a
Block-MDP). For this problem, optimal reconstruction loss is achieved
by a representation that ignores the state bit and memorizes the
noise. For example, if $\phi$ has a single output bit (which suffices
as there are only two states), it is never worse to preserve a noise
bit than the state bit. In fact, if one state is more probable than
the other, then predicting a noise bit along with the most likely
state results in \emph{strictly} better performance than predicting
the state bit.
So a representation using this approach can ignore state
information and is not useful for exploration.

\paragraph{Bisimulation.} 
A number of other abstraction definitions have been proposed and
studied in the state abstraction literature
(c.f.,~\cite{givan2003equivalence,Li06towardsa}). The finest
definition typically considered is \emph{bisimulation} or
\emph{model-irrelevance abstraction}, which aggregates two
observations $x_1,x_2$ if they share the same reward function and the
same transition dynamics over the abstract states, e.g., for each
abstract state $s'$, $T(\phi(x') = s' \mid x_1,a) = T(\phi(x') = s'
\mid x_2,a)$, where $\phi$ is the abstraction. A natural reward-free
notion simply aggregates states if they share the same dynamics over
abstract states, ignoring the reward condition. There are two issues
with using bisimulations and, as a consequence, coarser abstraction
notions. First, the trivial abstraction that aggregates all
observations together is a reward-free bisimulation, which is clearly
unhelpful for exploration. More substantively, learning a
reward-sensitive bisimulation is statistically intractable, requiring
a number of samples that is exponential in horizon~\cite[][Proposition
  B.1]{modi2019sample}.

An even finer definition than bisimulation, which has appeared
informally in the literature, aggregates two observations if they
share the same reward function and the same transition dynamics over
the observations~\cite[][Equation 2]{jiang2018notes}. The reward-free
version is equivalent to forward kinematic inseparability. However, we
are not aware of any prior work that attempts to learn such an
abstraction, as we do here.

\paragraph{Summary.}
These arguments show that previously studied state-abstraction or
representation learning approaches cannot be used for provably
efficient exploration in general Block-MDPs, at least when used with a
$\homingalg$-like algorithm. We emphasize that our analysis does not
preclude the value of these approaches in other settings (e.g.,
outside of Block-MDPs) or when used in other algorithms. Moreover, the
remarks here are of a worst case nature and do not necessarily imply
that the approaches are empirically ineffective.

\section{Related Work}
\label{sec:related}

Sample efficient exploration of Markov Decision Processes with a small
number of observed states has been studied in a number of papers~\citep{Rmax,DelayedQ,AuerRL},
initiated by the breakthrough result
of~\citet{kearns2002near}. While
state-of-the-art theoretical results provide near-optimal sample
complexity bounds for these small-state MDPs, the algorithms do not
exploit latent structure in the environment, and therefore, cannot
scale to the rich-observation environments that are popular in modern
empirical RL.

A recent line of theoretical
work~\cite{krishnamurthy2016richobs,jiang2017contextual} focusing on
rich observation reinforcement learning has shown that it is
information-theoretically possible to explore these environments and
has provided computationally efficient algorithms for some special
settings. In particular,~\citet{dann2018oracle} considers
deterministic latent-state dynamics while~\citet{du2019provably}
allows for limited stochasticity. As we have mentioned, the present
work continues in this line by eliminating assumptions required by
these results, further expanding the scope for tractable rich
observation reinforcement learning. 
In addition, our algorithm does not rely on abstract states for
defining policies or future prediction problems, therefore, avoiding
cascading errors due to inaccurate predictions (See discussion
in~\pref{sec:state-abstraction-failures}). On the other hand, we rely
on different computational assumptions, which we show to be
empirically reasonable.

On the empirical side, a number of approaches have been proposed for
exploration with large observation spaces using
pseudo-counts~\cite{tang2017exploration}, optimism-driven
exploration~\cite{chen2017ucb}, intrinsic
motivation~\cite{bellemare2016unifying}, and prediction
errors~\cite{pathak2017curiosity}. While these algorithms can perform
well on certain RL benchmarks, we lack a deep understanding of their
behavior and failure modes. As the examples
in~\pref{sec:state-abstraction-failures} show, using the
representations learned by these methods for provably efficient
exploration is challenging, and may not be possible in some cases.

Most closely related to our work,~\citet{nachum2018near} use a
supervised learning objective similar to ours for learning state
abstractions. However, they do not address the problem of exploration
and do not provide any sample complexity guarantees. Importantly, we
arrive at our supervised learning objective from the desire to learn
kinematic inseparability which itself is motivated by the $\oraclealg$
algorithm.

\section{Proof of Concept Experiments}
\label{sec:exp}
We evaluate $\homingalg$ on a challenging problem called a
\emph{diabolical combination lock} that contains high-dimensional
observations, precarious dynamics, and anti-shaped, sparse rewards.
The problem is googal-sparse, in that the probability of finding the
optimal reward through random search is $10^{100}$.

\begin{figure*}
\includegraphics[scale=0.25]{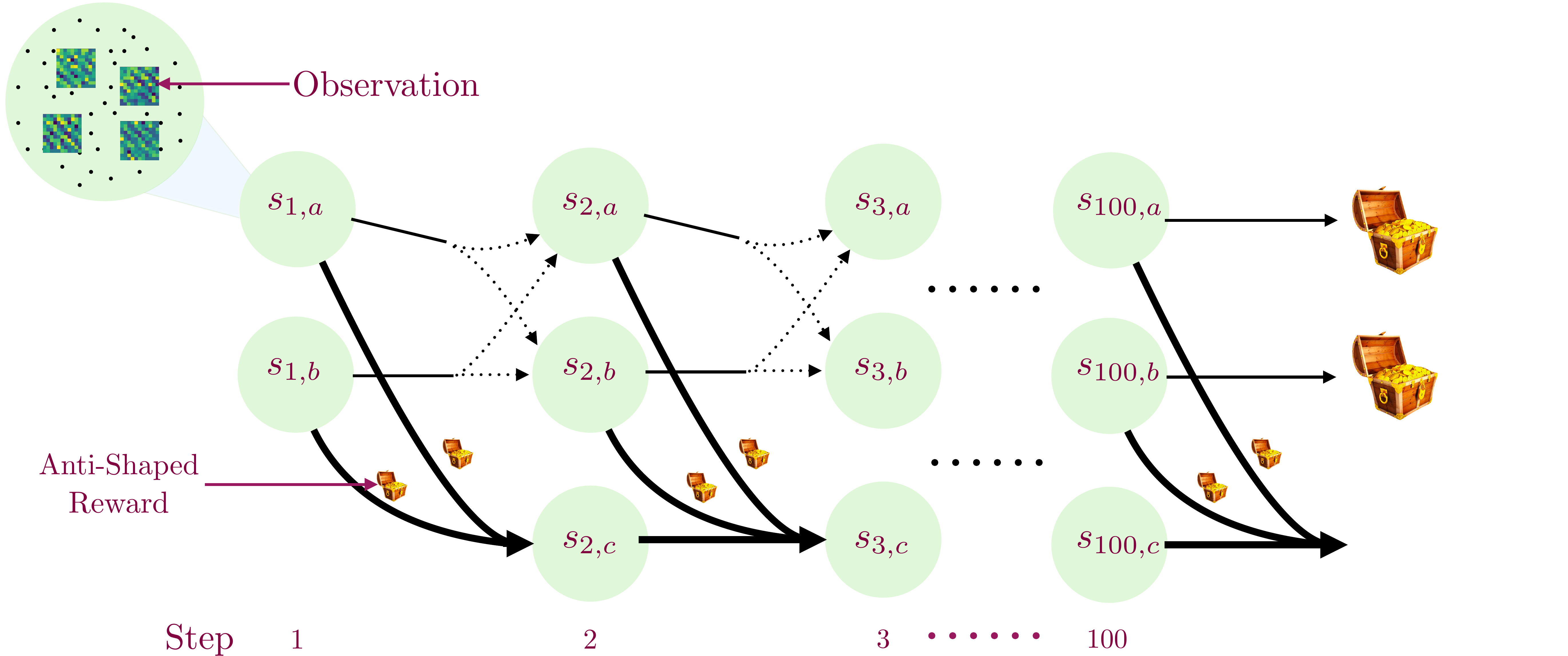}
\caption{Illustrates the diabolical combination lock problem which contains multiple challenges including sparse anti-shaped rewards, rich-observations, long horizons and extremely sparse good rewards.}
\label{fig:diabolical-combolock}
\end{figure*}

\paragraph{The environment.} 
\pref{fig:diabolical-combolock} depicts the structure of the
diabolical combination lock, which formally is a class of RL
environments.
For a fixed horizon $\horizon$ and action space size $K$, the state
space is given by $\states \defeq \{\state_{1, a}, \state_{1, b}\}
\cup \{\state_{h, a}, \state_{h, b}, \state_{h, c}\}_{h=2}^H$ and the
action space by $\actions \defeq \{a_1, ..., a_K\}$.  The agent starts
in either $\state_{1, a}$ or $\state_{1, b}$ with equal
probability. After taking $h$ actions the agent is in $\state_{h+1,
  a}, \state_{h+1, b}$ or $\state_{h+1, c}$.

Informally, the states $\{\state_{h, a}\}_{h=1}^H$ and $\{\state_{h,
  b}\}_{h=1}^H$ are ``good'' states from which optimal return are
still achievable, while the states $\{\state_{h, c}\}_{h=2}^H$ are
``bad'' states from which it it impossible to achieve the optimal
return.  Each good state is associated with a single good action,
denoted $u_h$ for $s_{h,a}$ and $v_h$ for $s_{h,b}$, which transitions
the agent uniformly to one of the two good states at the next time
step. All other actions, as well as all actions from the bad state,
lead to the bad state at the next time.
Formally the dynamics are given by:
\begin{align*}
& T(\cdot \mid \state_{h,a},u_h) = T(\cdot \mid \state_{h,b}, v_h) = \unf\rbr{\cbr{s_{h+1,a},s_{h+1,b}}};\\
& T(s_{h+1,c} \mid \state_{h,a},a') = \one\cbr{a' \ne u_h}, \quad T(s_{h+1,c} \mid \state_{h,b},a') = \one\cbr{a' \ne v_h}, \quad T(s_{h+1,c} \mid \state_{h,c},a') = 1.
\end{align*}
We fix the vectors $u_{1:H},v_{1:H}$ before the learning process by
choosing each coordinate uniformly from $\actions$.

The agent receives a reward of $1$ on taking action $u_\horizon$ in
$\state_{\horizon, a}$ or action $v_\horizon$ in $\state_{\horizon,
  b}$. Upon transitioning from $\state_{h, a}$, $\state_{h, b}$ to
$\state_{h+1, c}$ the agent receives an anti-shaped reward of
$0.1\times\textrm{Ber}(\nicefrac{1}{2})$.
This anti-shaped reward is easy for the agent to collect, but it
prevents the agent from achieving the optimal return in any given
episode.
The agent receives a reward of $0$ for all other transitions. The
reachability parameter is $\nicefrac{1}{2}$ and $V(\policy^\star) = 1$
where for all $h \in [\horizon]$, $\policy^\star$ takes action $u_h$
in $\state_{h, a}$ and action $v_h$ in $\state_{h, b}$.

The agent never directly observes the state and instead receives an
observation $\obs \in \RR^d$ where $d = 2^{\lceil \log_2
  (\horizon+4)\rceil}$, generated stochastically as follows.  First,
the current state information and time step are encoded into one-hot
vectors which are concatenated together and added with an isotropic
Gaussian vector with mean $0$ and variance $0.1$.  This vector is then
padded with an all-zero vector to lift into $d$ dimension and finally
multiplied by a fixed rotation matrix.\footnote{We use a Hadamard
  matrix for the rotation, which is a square matrix with entries in
  $\{-1,+1\}$ and mutually orthogonal rows. For sizes that are powers
  of two, these matrices can be easily constructed using an inductive
  approach known as Sylvester's method.}

Our main experiments consider $\horizon=100$ and $|\actions|=K=10$. In
this case, the probability of receiving the optimal reward when taking
actions uniformly at random is $10^{-100}$.\footnote{For comparison,
  $10^{100}$ is more than the current estimate of the total number of
  elementary particles in the universe.} Moreover, for any fixed
sequence of actions the probability of getting a reward of $1$ is at
most $2^{-\tau}$ where $\tau\defeq\sum_{h=1}^{100} \one\cbr{u_h \ne
  v_h}$. As $u_{1:H}$ and $v_{1:H}$ are chosen randomly, we have
$\EE[\tau] = 90$ in these instances.

\begin{figure}[t]
\centering
\begin{minipage}{\textwidth}
\centering
\hspace*{\stretch{1}}%
\subfloat[]{
     \includegraphics[scale=0.45]{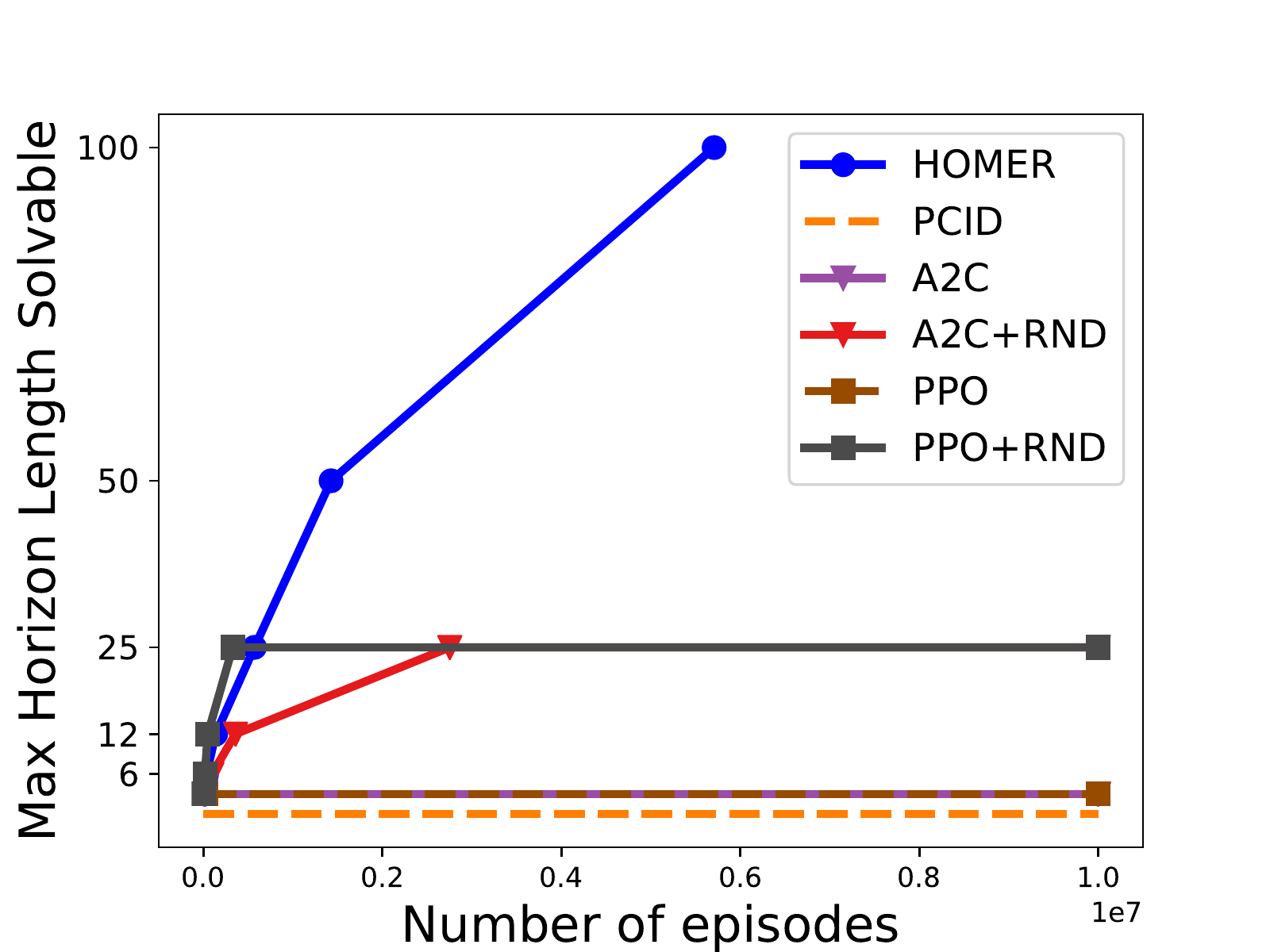}\label{fig:horizon-plot}
}
\hspace{\stretch{2}}%
\subfloat[]{
  \includegraphics[scale=0.45]{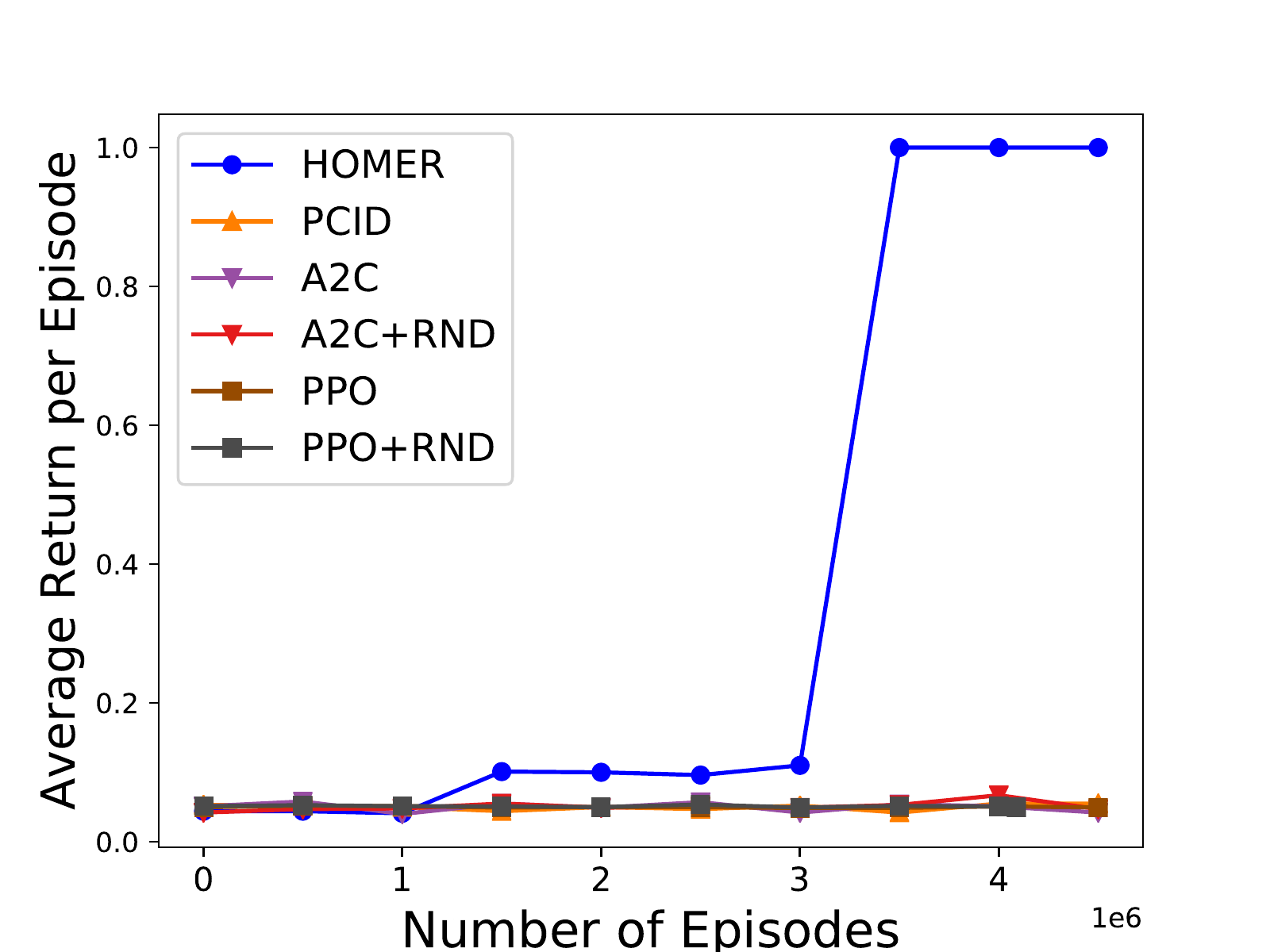}\label{fig:value-function-plot}
  }
\hspace{\stretch{1}}
\vspace{-0.3cm}
\caption{Results on the diabolical combination lock problem with
  action space $(K)$ of size $10$. \textbf{Left:} Horizon of the
  problem against episodes needed to learn a policy with value at
  least half of optimal. \textbf{Right:} Empirical return of policy
  against number of episodes for horizon of 100. The value of the
  optimal policy is $1$.}
\label{fig:result-combo-locks}
\end{minipage}%
\end{figure}

\begin{figure*}
\parbox{0.98\textwidth}{ $\mathbf{\ppo}$ \hfill (fails to explore from time step  5)}\\[.2cm]
\includegraphics[scale=.21]{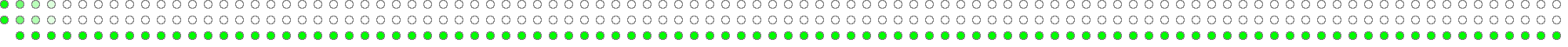}\\[.1cm]
\parbox{0.98\textwidth}{ $\mathbf{{\tt BEST\,\,\,\,RND}}$   \hfill (fails to explore from time step 28)}\\[.2cm]
\includegraphics[scale=.21]{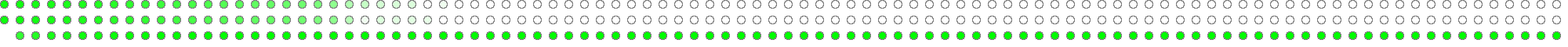}\\[.1cm]
\parbox{0.98\textwidth}{ $\mathbf{\homingalg}$  \hfill (successfully explores at all time steps)}\\[.2cm]
\includegraphics[scale=.21]{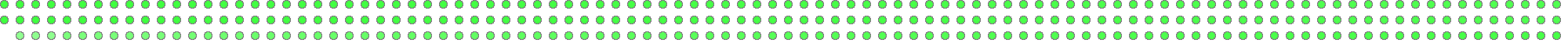}
\caption{Visualization of the visitation probabilities for algorithms on the diabolical combination lock problem. The top row, middle row and the bottom row represent states in $\{\state_{h, a}\}_{h=1}^{100}$, $\{\state_{h, b}\}_{h=1}^{100}$ and $\{\state_{h, c}\}_{h=1}^{100}$ respectively. The $h^{th}$ column represents states reachable at time step $h$. We do not show observations or transition edges for brevity. We sample $100,000$ episodes uniformly through the execution of the algorithm and compute the number of time ${\tt count}[\state]$ the agent visits a state $\state$. The count statistics is shown using the opacity of the fill of each state. Formally, we set opacity of $\state$ as $\propto \ln({\tt count}[\state]+1)$. The more opaque the circles are the more frequently the agent visits them. $\homingalg$  is able to explore well for all time steps unlike baselines. The $\mathbf{{\tt BEST\,\,RND}}$ baseline is the best algorithm (using $\ppo$) with random distillation bonus.}
\label{fig:traces}
\end{figure*}

\paragraph{$\homingalg$ implementation.} 
We use non-stationary deterministic policies, where each policy is
represented as a tuple of $H$ linear models $\policy = (W_1, W_2,
\cdots, W_\horizon)$. Here $W_h \in \mathbb{R}^{|\actions| \times d}$
for each $h \in [H]$. Given an observation $\obs \in \mathbb{R}^d$ at
time step $h$, the policy takes the action $\policy(\obs) \defeq
\argmax_{\action \in \actions} (W_h\obs)_a$.

We represent the state abstraction function $\phi: \observations
\rightarrow [N]$ using a linear model $A \in \mathbb{R}^{N \times d}$,
but rather than a single regression function with bottlenecks on both
observations, we train two separate models. For the backward
abstraction, given a tuple $\obs,\action,\obs'$ we form a vector by
concatenating $B\obs,\one_\action$, and $z$ together, where $B \in
\RR^{M\times d}$, $\one_{\action}$ is the one-hot encoding of the
action, and $z_i \propto \exp((A\obs')_i + g_i)$ applies the Gumbel
softmax trick~\cite{jang2016categorical} to convert $A\obs'$ into a
probability distribution ($g_i$ is an independent Gumbel random
variable). Then we pass the concatenated vector into a two layer
feed-forward neural network with leaky rectified linear
units~\cite{Maas13rectifiernonlinearities} and a softmax output
node,\footnote{We use a softmax output node since we use cross-entropy
  loss rather than square loss in our
  implementation. See~\pref{app:appendix-experiment}.}  to obtain the
prediction. For the forward abstraction, we define the predictor
similarly, but apply the Gumbel softmax trick to $B\obs$.  We learn
the weight matrices $A,B$ that form the abstraction, as well as the
parameters of the feed-forward neural network. We also allow the
capacity of the forward and backward abstractions to be different,
i.e., $M \ne N$.

We also make three other empirically motivated modifications to the algorithm. 
First, we optimize cross-entropy loss instead of squared loss
in~\pref{alg:learn_homing_policy},~\pref{line:homer-learn-encoding-function}. Second,
we implement the contextual bandit oracle
(\pref{alg:rl_via_homing_policies},~\pref{line:psdp-csc-call}) by
training the policy model to predict the reward via square loss
minimization.
Lastly, we use a more sample-efficient data collection procedure:
We collect a set of observed transitions $\{(\obs_i, \action_i,
\obs'_i)\}_{i=1}^n$ using our sampling procedure
(\pref{alg:rl_via_homing_policies},~\pref{line:homer-sampling-procedure}),
and we create imposter transitions by resampling within this set. For
example, when training the model with bottleneck on $\obs'$, we create
imposter transitions $\{(\obs_i, \action_i, \tilde{\obs}_i)\}_{i=1}^n$
where $\tilde{\obs}_i$ are chosen uniformly from
$\{\obs_1',\ldots,\obs_n'\}$.
We also found that the optimization is more stable if we initialize
the state abstraction model by first training without the
discretization step.  This is achieved by removing the Gumbel-Softmax
step.
For full details on optimization and hyperparameter tuning
see~\pref{app:appendix-experiment}.

\paragraph{Baselines.} 
We compare our method against Advantage Actor Critic ($\aac$)
~\cite{mnih16} and Proximal Policy Optimization ($\ppo$)
~\citep{schulman2017proximal}. By default, these algorithms use naive
exploration strategies based on entropy bonuses which are often
insufficient for challenging exploration problems.  Therefore, we also
augment these methods with an exploration bonus based on Random
Network Distillation (RND)~\cite{burda2018exploration}, denoted
$\aacrnd$ and $\ppornd$.  The RND exploration bonus consists of
randomly initializing and fixing a neural network mapping observations
to an embedding space and training a second network to predict the
outputs of the first network using observations collected by the
agent. The prediction error of this network on the current observation
is then used as an auxiliary reward signal. The intuition is that the
prediction error will be high on unseen observations, thus rewarding
the agent for reaching novel states.

We also consider the model-based algorithm ($\du$)
of~\citet{du2019provably} which we have discussed
in~\pref{sec:state-abstraction-failures}. This algorithm enjoys a
sample complexity guarantee for Block MDPs with certain margin-type
assumptions, but the diabolical combination lock does not satisfy
these assumptions, so the guarantee does not apply. We run this
algorithm primarily to understand if these assumptions are also
empirically limiting.

We use publicly available code for running these baselines. For
details see~\pref{app:appendix-experiment}.

\paragraph{Results.} 
\pref{fig:horizon-plot} reports the number of episodes needed to
achieve $\nicefrac{1}{2}$ of the optimal value, for different values
of the horizon $H$.
We run each algorithm for a maximum of $10$ million episodes.  $\aac$
and $\ppo$ fail to solve the problem for $\horizon>3$, and do not
escape the local maximum of collecting the anti-shaped reward, which
is unsurprising as these algorithms do not engage in strategic
exploration. $\du$ fails to solve the problem for any horizon length
as it cannot learn to separate observations from $\state_{h, a}$ and
$\state_{h, b}$ for any $h$. As $\du$ defines the policy class over
abstract states, the agent is forced to take a single action in both
$\state_{h, a}$ and $\state_{h, b}$ which leads to a failure to
explore. This confirms that the margin condition for $\du$ is both an
empirical and theoretical limitation. $\aacrnd$ and $\ppornd$ both
perform reasonably well and are able to solve problems with horizon up
to $25$. However, they fail for the longer horizons of $50$ and
$100$. $\homingalg$ is the only algorithm able to solve the problem
for all horizon lengths tested.

\pref{fig:value-function-plot} focuses on $H=100$ and shows the
average return per episode throughout the training process. None of
the baselines learn the optimal policy. The plot for $\homingalg$
shows three plateaus corresponding to three different phases of the
algorithm: (i) learning a policy cover, (ii) learning the
reward-sensitive policy $\hat{\policy}$ and (iii) sampling actions
from $\hat{\policy}$.

We visualize the visitation probabilities for the best performing
algorithms in~\pref{fig:traces}.
Observe that the visitation probabilities for $\homingalg$ do not
decay significantly with time step, as they do for the baselines.
 
\begin{figure*}
\includegraphics[scale=.22]{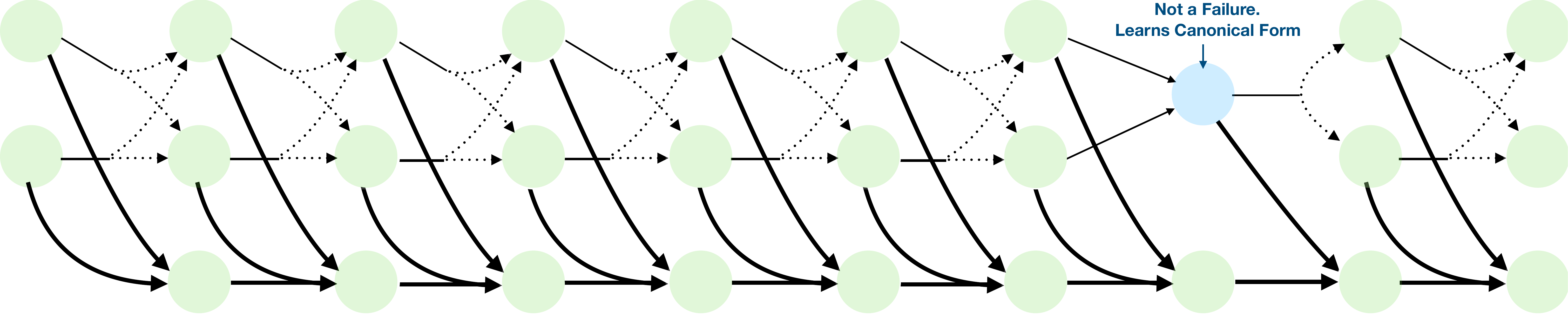}
\vspace{-0.2cm}
\caption{Recovered dynamics for the first $10$ time steps of the
  diabolical combination lock with $H=100$. The abstract states in the
  first two rows correspond to $\{\state_{h, a} \mid h \in [10]\}$ and
  $\{\state_{h, b} \mid h \in [10]\}$ while the last row correspondsto
  $\{\state_{h, c} \mid h \in [10]\}$. At time step $8$, the model
  maps observations from $\{\state_{8, a}, \state_{8, b}\}$ to the
  same abstract state since actions $u_8$ and $v_8$ are the same,
  which makes these states kinematically inseparable.}
\label{fig:learned-dynamics}
\end{figure*}

We visualize the recovered state abstraction and dynamics
in~\pref{fig:learned-dynamics} for the first 10 time steps when
$H=100$. We use the learned forward and backward abstractions at each
time step to recover the state abstraction, and we estimate the
transition dynamics by collecting trajectories from the policy cover
and recording the empirical frequencies of transitioning between abstract
states.\footnote{Formally, let $\{\Psi_h\}_{h=1}^\horizon$ be the
  learned policy cover and $\{\overline{\phi}_h\}_{h=1}^\horizon$ be
  the learned state abstraction which is a composition of forward and
  backward abstractions. For time step $h$, we draw $n$ transitions
  $\obs_i, \action_i, \obs' \sim \unf(\Psi_{h-1}) \circ
  \unf(\actions)$. The model $\hat{T}_h(j \mid i, \action)$ for time
  step $h$ is given by: $\hat{T}(j \mid i, \action) =
  \frac{\sum_{k=1}^n \one\{\overline{\phi}_h(\obs'_k) = j, \action_k =
    \action, \overline{\phi}_{h-1}(\obs_k) = i \}}{\sum_{k=1}^n
    \one\{\action_k = \action, \overline{\phi}_{h-1}(\obs_k) = i \}}$
  for any learned abstract states $j, i$ and action $\action$.}
As~\pref{fig:learned-dynamics} shows, $\homingalg$ is able to recover
the canonical latent transition dynamics.\footnote{Recall that
  canonical form of a Block MDP is created by merging kinematically
  inseparable states into one state.}  Over all 100 time steps, we
found that $\homingalg$ only made a single state abstraction error,
merging observations from two kinematically separable states together
into a single abstract state.
Note that, $\homingalg$ does not use the learned model for planning,
which helps avoids cascading errors that arise from such
state-abstraction errors.  However, utilizing the learned dynamics for
model-based planning may improve sample efficiency and is a natural
direction for future work.

\section{Conclusion}
We present $\homingalg$, a model-free reinforcement learning algorithm
for rich observation environments. Theoretically, we prove that
$\homingalg$ learns both a policy cover and a near optimal policy for
any Block MDP using only polynomially many trajectories and in
polynomial time, under certain computational and expressivity
assumptions. Empirically, we show that $\homingalg$ can solve a
difficult exploration problem in a tractable manner, and even recover
the latent transition dynamics.  Experimenting with $\homingalg$ on
real-world reinforcement learning scenarios is a natural direction for
future work.

\vspace{-0.2cm}
\section*{Acknowledgements}
We would like to thank Alekh Agarwal and Miroslav Dudik for formative
discussion, Nan Jiang for pointers to previous results on state
abstraction, and Wen Sun for proof reading an early version of this
paper.  We also thank Vanessa Milan for guidance on the design of the
figures in the paper.

\vfill
\pagebreak

\bibliographystyle{plainnat}
\bibliography{main}

\appendix
\pagebreak
\onecolumn
\section*{Appendices}

\begin{table*}[!h]
\centering
\begin{tabular}{c|l}
\hline
\textbf{Symbol} & \textbf{Definition} \\
\hline
$[N]$ & Defines the set $\{1, 2, \cdots, N\}$ for any $N \in \mathbb{N}$.\\
$\Delta(\Ucal)$ & The space of probability distribution over the set $\mathcal{U}$.\\
$\unf(\Ucal)$ & A uniform distribution over the set $\Ucal$.\\
$\textrm{supp}(p)$ & Support of a distribution. For any $p \in \Delta(\Ucal)$, we have $\textrm{supp}(p) \defeq \{u \in \Ucal \mid p(u) > 0\}$.\\
\hline
$\mdp$ & Denotes a Block Markov Decision Process (Block MDP). \\ 
$\horizon$ & Number of actions the agent takes for any episode.\\
$\states$ & A finite state space of $\mdp$. The process is layered, so states also encode the time step. \\
$\states_h$ & The set of states reachable at time step $h$. $\states \defeq \cup_{h\in [\horizon]} \states_h$.\\
$\observations$ & The observation space. May be infinitely large, but is assumed to be countable. \\
$\observations_h$ & Set of observations reachable at time step $h$. $\observations \defeq \cup_{h\in [\horizon]} \observations_h$.\\
$\actions$ & The finite discrete action space of $\mdp$\\
$\emf$ & An emission function $\emf: \state \rightarrow \Delta(\observations)$. $\emf(\obs \mid \state)$ denotes the probability of \\ 
& observing $\obs$ in state $\state$. Note that $\textrm{supp}(\emf(\cdot | \state)) \cap \textrm{supp}(\emf(\cdot | \state')) = \emptyset$ when $\state \ne \state'$.\\
$\decoder$ & A decoder function $\decoder: \observations \rightarrow \states$. $\decoder(\obs) = \state$ iff $\emf(\obs|\state) > 0$. \\ %
$\startdist(\state)$ & Probability of starting in state $\state$ at the beginning of any episode.\\
$\transition(\state' \mid \state, \action)$ & The probability of transitioning to $\state' \in \states$ when taking action $\action \in \actions$ in $\state \in \states$.\\
$\policy$ & A policy $\policy: \observations \rightarrow \Delta(\actions)$, which may or may not be stationary. \\
$(\policy_1, \ldots, \policy_h)$ & A $h$-step policy where the $t^{\textrm{th}}$ action ($1\le t \le h$) is taken according to $\policy_t$.\\
$\policies$ & The policy class, $\policies: \observations \to \Delta(\actions)$. \\ %
$\nspolicies$ & The set of non-stationary policies: $\nspolicies \defeq \{(\policy_1,\ldots,\policy_\horizon): \policy_t \in \policies\}$. \\ %
$\PP_\policy(\state)$ & Probability of visiting $\state$ when following $\policy$, from the starting distribution $\startdist$.\\
$\homingp_\state$ & Homing policy of the state $\state$, $\homingp_\state \defeq \arg\max_{\policy \in \policies_{\textrm{NS}}}\PP_{\policy}(\state)$.\\
& Due to~\pref{lem:homing_deterministic}, we take $\homingp_\state$ to be deterministic and non-stationary.\\
$\homingp_\obs$ & Analogous homing policy of the observation $\obs$, $\homingp_\obs \defeq \arg\max_{\policy \in \policies_{\textrm{NS}}} \PP_\policy(\obs)$\\
& It is easy to see from the Block MDP assumption that $\homingp_\obs = \homingp_{\state}$ where $\state = \decoder(\obs)$.\\
$V(\policy; R)$ & Value for (non-stationary) policy $\policy$ on reward function $R$ from starting distribution $\startdist$. \\
& $R$ may have type $R:\Xcal\to \RR$ or $R: (\Xcal \times \Acal) \to \RR$ and may also be stochastic. \\
$V(\state; \policy,  R)$ and $V(\obs; \policy, R)$ & Value functions for $\policy$ on $R$, defined over states or observations.\\
$\eta(\state)$ & Maximum visitation probability for state $\state$, $\eta(\state) \defeq \sup_{\policy \in \policies} \PP_\policy(\state)$. \\
$\eta_{min}$ & Reachability parameter, $\eta_{min} \defeq \min_{\state \in \states}\eta(\state)$.\\
$\Fcal$ & Regressor class containing functions $f$ from $\observations \times \actions \times \observations \rightarrow [0, 1]$. \\
$\Phi_N$ & Decoder function class $\Phi_N: \observations \to [N]$, for a fixed $N$.\\
$\timestep$ & Time function that maps states/observations to the time step where they are reachable, \\
& which is well-defined due to the layered assumption. \\ %
\hline
\end{tabular}
\caption{List of notations and brief definitions. Operators defined on
  states are lifted to observations in the natural way, e.g.,
  $\startdist(\obs) \defeq \sum_{\state}\startdist(\state)\emf(\obs \mid
  \state)$.}
\label{tab:symbol-notation}
\end{table*}

\paragraph{Notation and Overview.}
See~\pref{tab:symbol-notation} for an overview of the notation and definitions used in this appendix. 
The appendix is structured as follows:
\begin{packed_item}
\item \pref{app:appendix-homing}: Properties of homing policies;
\item \pref{app:appendix-kinematic-inseparability}: Properties of kinematic inseparability;
\item \pref{app:psdp-appendix}: Basic results for Policy Search from Dynamic Programming (PSDP);
\item \pref{app:appendix-oracle-alg}: Analysis of the $\oraclealg$ algorithm;
\item \pref{app:appendix-sample-complexity}: Analysis of the $\homingalg$ algorithm;
\item \pref{app:supporting}: Supporting results;
\item \pref{app:appendix-experiment}: Experimental details.
\end{packed_item}

\section{Properties of Homing Policies}
\label{app:appendix-homing}

In this section we prove some basic properties of homing policies. For
this section only, we consider a fully expressive policy set $\Upsilon
\defeq (\observations \to \Delta(\actions))$. We further define the
set of \emph{all} deterministic policies $\Upsilon_{\textrm{det}}
\defeq (\observations \to \actions)$. Clearly we have
$\Upsilon_{\textrm{det}} \subset \Upsilon$. Note that both of these
classes contain non-stationary policies due to the layered structure
of the environments we consider. In particular, we have
$(\pi_1,\ldots,\pi_H) \in \Upsilon$ whenever $\pi_h \in \Upsilon$ for
all $h$, with a similar statement for $\Upsilon_{\textrm{det}}$.

The first result is that for every state, there exists a
\emph{deterministic non-stationary} policy $\policy \in
\Upsilon_{\textrm{det}}$ that is a homing policy for that state. This
motivates our decision to restrict our search to only these policies
in experiments. The result also appears in~\citep{bagnell2004policy},
but we provide a proof for completeness.

\begin{lemma}
\label{lem:homing_deterministic}
For any, possibly stochastic, reward function $R$, we have
\begin{align*}
\max_{\policy \in \Upsilon_{\textrm{det}}} V(\policy; R) = \max_{\policy \in \Upsilon} V(\policy; R),
\end{align*}
where $V(\policy;R)$ is the value for policy $\policy$ under reward
function $R$.  In particular, the result holds for $R(\obs) =
\one\{\decoder(\obs) = \state\}$, for any $\state \in
\states$, which yields:
\begin{align*}
\max_{\policy \in \Upsilon_{\textrm{det}}} \PP_{\policy}(\state) = \max_{\policy \in \Upsilon} \PP_{\policy}(\state).
\end{align*}
\end{lemma}
\begin{proof}
  As $\Upsilon_{\textrm{det}} \subset \Upsilon$, that
  the LHS is at most the RHS is obvious. We are left to establish the
  other direction.

  The proof is a simple application of dynamic programming. Assume
  inductively that there exists a policy $\tilde{\pi}_{h:H} \in
  \Upsilon_{\textrm{det}}$ such that, for any distribution $Q \in
  \Delta(\observations_{h})$, we have
  \begin{align*}
    \EE_{\obs_h \sim Q}\sbr{ V(\tilde{\pi}_{h:H}; R, \obs_h)} \geq \max_{\policy_{h:H} \in \Upsilon} \EE_{\obs_h \sim Q}\sbr{ V(\pi_{h:H}; R, \obs_h)},
  \end{align*}
  where the value function here denotes the future reward, according
  to $R$, when executing the policy from the starting observation
  $\obs_h$. The base case is when $h=H$, in which case, it is easy to
  verify that the claim holds for the policy $\tilde{\pi}_H(\obs_H)
  \defeq \argmax_{\action \in \actions} \EE[ R \mid \obs_H]$.

  Define the policy component for time step $h-1$ as
\begin{align*}
  \forall \obs_{h-1} \in \observations_{h-1}: \tilde{\pi}_{h-1}(\obs_{h-1}) \defeq \argmax_{\action \in \actions} \EE[ R + V(\tilde{\pi}_{h:H}; R, \obs_{h}) \mid \obs_{h-1}].
\end{align*}
Now for any potentially stochastic policy $\pi_{h-1:H}$, and any
distribution $Q \in \Delta(\observations_{h-1})$ we have
\begin{align*}
\EE_{\obs_{h-1} \sim Q}\sbr{ V(\tilde{\pi}_{h-1:H}; R, \obs_{h-1})} & \geq \EE_{\obs_{h-1} \sim Q}\sbr{ R + V(\tilde{\pi}_{h:H}; R, \obs_{h}) \mid a \sim \pi_{h-1}}\\
& \geq \EE_{\obs_{h-1} \sim Q}\sbr{ R + V(\pi_{h:H}; R, \obs_{h}) \mid a \sim \pi_{h-1}}\\
& = \EE_{\obs_{h-1} \sim Q}\sbr{ V(\pi_{h:H}; R, \obs_{h-1})}.
\end{align*}
This proves the inductive step. We conclude the proof by noting that
$V(\policy;R) = \EE_{\obs_1 \sim \startdist} V(\policy;R,\obs_1)$, for
which we have established the optimality guarantee for
$\tilde{\pi}_{1:H}$.
\end{proof}

We also observe that homing policies do not grow compositionally. In
other words, we may not be able to construct homing policies for
states $\states_h$, by appending a one-step policy to the homing
policies for $\states_{h-1}$. Note that this holds even when working
with the unrestricted policy class $\Upsilon$. This observation
justifies the global policy search procedure $\psdpalg$ for finding
the homing policies.

For the statement, for a policy subset $\Pi'$, we use the notation
$\Delta(\Pi')$ to denote the set of \emph{mixture policies} that, on
each episode samples a policy $\pi \in \Pi'$ from a distribution and
executes that policy. Note that this is not the same as choosing a new
policy from the distribution on a per time-step basis.

\begin{lemma}\label{prop:comp-homing} 
There exists a Block MDP $\mdp$ a time step $h \in [H]$ and a state
$\state \in \states_h$ such that
\begin{align*}
\eta(\state) > \sup_{\pi_{\textrm{mix}} \in \Delta(\{\homingp_s\}_{s \in \states_{h-1}}), \pi_h \in \Upsilon} \PP_{\pi_{\textrm{min}} \circ_h \pi_h}(s).
\end{align*}
Here $\pi_{\textrm{mix}}$ is a \emph{mixture policy} over the homing
policies $\{\homingp_s\}_{s \in \states_{h-1}}$ for the states at time
step $h-1$, and $\circ_h$ denotes policy composition at time $h$.
\end{lemma}
\begin{proof} 
See~\pref{fig:mdp-example-left}. The homing policy for $\state_5$
takes action $\action_1$ in $\state_1$, which yields a visitation
probability for $\state_5$ of $1$. However, the homing policies for
states $\state_2, \state_3$, and $\state_4$ do not take action
$\action_1$ in $\state_1$.
\end{proof}

\section{Properties of Kinematic Inseparability}
\label{app:appendix-kinematic-inseparability}

In this section, we establish several useful properties for kinematically inseparable (KI) state abstractions. Recall~\pref{def:bki},~\pref{def:fki}, and~\pref{def:ki}.
\begin{fact} Forward kinematic inseparability (KI), backward KI and KI defines an equivalence relation on $\observations$.
\end{fact}
\begin{proof}
That these relations are reflexive, symmetric, and transitive, all follow trivially from the definitions, in particular using the fact that equality itself is symmetric and transitive.
\end{proof}

\begin{lemma}
\label{lem:state-inseparable} 
Let $\obs_1, \obs_2 \in \observations$. If $\decoder(\obs_1) =
\decoder(\obs_2)$ then $\obs_1$ and $\obs_2$ are KI. This implies that
they are also forward KI and backward KI.
\end{lemma}
\begin{proof}
Fix any $\obs \in \observations$, $\action \in \actions$ and $u \in
\Delta(\observations \times \actions)$ with $\textrm{supp}(u) =
\observations \times \actions$. We show below that $\obs_1$ and
$\obs_2$ are forward KI and backward KI which together establishes
that desired claim.

\paragraph{Forward KI:} By the Block MDP structure, we have
\begin{align*}
T(\obs \mid \obs_1, \action) = T(\obs \mid \decoder(\obs_1), \action) = T(\obs \mid \decoder(\obs_2), \action) = T(\obs \mid \obs_2, \action)
\end{align*}

\paragraph{Backward KI:} 
Again, using the Block MDP structure:
\begin{align*}
\PP_u(\obs, \action \mid \obs_1) 
& = \frac{T(\obs_1 \mid \obs, \action) u(\obs, \action)}{\EE_{(\tilde{\obs}, \tilde{\action}) \sim u }\sbr{ T(\obs_1 \mid \tilde{\obs}, \tilde{\action}) }} 
= \frac{\emf(\obs_1 \mid \decoder(\obs_1))T(\decoder(\obs_1) \mid \obs, \action) u(\obs, \action)}{\EE_{(\tilde{\obs}, \tilde{\action})\sim u }\sbr{ \emf(\obs_1 \mid \decoder(\obs_1))T(\decoder(\obs_1) \mid \tilde{\obs}, \tilde{\action}) }} \\ 
&= \frac{T(\decoder(\obs_1) \mid \obs, \action) u(\obs, \action)}{\EE_{(\tilde{\obs}, \tilde{\action})\sim u}\sbr{ T(\decoder(\obs_1) \mid \tilde{\obs}, \tilde{\action}) }} 
=\frac{T(\decoder(\obs_2) \mid \obs, \action) u(\obs, \action)}{\EE_{(\tilde{\obs}, \tilde{\action})\sim u}\sbr{ T(\decoder(\obs_2) \mid \tilde{\obs}, \tilde{\action}) }} \\
&= \frac{\emf(\obs_2 \mid \decoder(\obs_2))T(\decoder(\obs_1) \mid \obs, \action) u(\obs, \action)}{\EE_{(\tilde{\obs}, \tilde{\action})\sim u}\sbr{ \emf(\obs_2 \mid \decoder(\obs_2))T(\decoder(\obs_1) \mid \tilde{\obs}, \tilde{\action}) }} 
=  \frac{T(\obs_2 \mid \obs, \action) u(\obs, \action)}{\EE_{ (\tilde{\obs}, \tilde{\action}) \sim u}\sbr{ T(\obs_2 \mid \tilde{\obs}, \tilde{\action}) }}  = \PP_u(\obs, \action \mid \obs_2).\tag*\qedhere
\end{align*}
\end{proof}

The next simple fact shows that observations that appear at different
time points are always separable.
\begin{fact}
\label{fact:time-separable}
If $\obs,\obs'$ are forward or backward KI, then $\timestep(\obs) =
\timestep(\obs')$, where recall that $\tau(\obs)$ denotes the time
step where $\obs$ is reachable.
\end{fact}
\begin{proof}
If $h \defeq \tau(\obs) \ne \tau(\obs') := h'$ then $T(\cdot \mid
\obs,a) \in \Delta(\observations_{h+1})$ while $T(\cdot \mid \obs',a)
\in \Delta(\observations_{h'+1})$, so these distributions cannot be
equal. A similar argument holds for Backward KI.
\end{proof}

Using the transitivity property for backward KI, we can consider sets
of observations that are all pairwise backward KI.  The next lemma
provides a convenient characterization for backward KI sets.
\begin{lemma}
\label{lem:bki-ratio} Let $\observations'
 \subseteq \observations$ be a set of backward KI observations. Then
 $\exists u \in \Delta(\observations)$ with $\textrm{supp}(u) =
 \observations$ such that for all $\obs', \obs'' \in \observations'$
 we have:
\begin{align}
\forall \obs \in \observations, \action \in \actions, \qquad \frac{T(\obs' \mid \obs, \action)}{u(\obs')} = \frac{T(\obs'' \mid \obs, \action)}{u(\obs'')}.\label{eq:bki-ratio}
\end{align}
The converse is also true: if~\pref{eq:bki-ratio} holds for some $u
\in \Delta(\observations)$ with full support and all $x',x'' \in
\observations' \subset \observations$ then $\observations'$ are is a
backward KI set.
\end{lemma}
\begin{proof}
Fix $\tilde{u} \in \Delta(\observations \times \actions)$ with
$\textrm{supp}(\tilde{u}) = \observations \times \actions$. Define
$u(\obs) \defeq \EE_{(\tilde{\obs}, \tilde{\action})\sim \tilde{u} }\sbr{
  T(\obs \mid \tilde{\obs}, \tilde{\action}) }$. Observe that by
construction $u(\obs) > 0$ for all $\obs \in \observations$. Let
$\obs', \obs'' \in \observations'$ then as $\obs'$ and $\obs''$ are
backward KI, we have that for all $\obs \in \observations, \action\in
\actions$:
\begin{align*}
\PP_{\tilde{u}}(\obs, \action \mid \obs'') =\PP_{\tilde{u}}(\obs, \action \mid \obs') &\Rightarrow
\frac{T(\obs'' \mid \obs, \action) \tilde{u}(\obs, \action)}{\EE_{(\tilde{\obs}, \tilde{\action})\sim \tilde{u}}\sbr{T(\obs'' \mid \tilde{\obs}, \tilde{\action}) }}
 = \frac{T(\obs' \mid \obs, \action) \tilde{u}(\obs, \action)}{\EE_{(\tilde{\obs}, \tilde{\action})\sim \tilde{u} }\sbr{ T(\obs' \mid \tilde{\obs}, \tilde{\action}) }} \\
&\Rightarrow \frac{T(\obs'' \mid \obs, \action)}{u(\obs'')} = \frac{T(\obs' \mid \obs, \action)}{u(\obs')}. 
\end{align*}
For the converse, let $\tilde{u} \in
\Delta(\observations\times\actions)$ have full support. Then we have
\begin{align*}
\PP_{\tilde{u}}(\obs, \action \mid \obs'_1) = 
\frac{T(\obs'_1 \mid \obs, \action)\tilde{u}(\obs, \action)}{\sum_{\tilde{\obs}, \tilde{\action}} T(\obs'_1 \mid \tilde{\obs}, \tilde{\action}) \tilde{u}(\tilde{\obs}, \tilde{\action})}
&= \frac{\frac{u(\obs'_1)}{u(\obs'_2)}T(\obs'_2 \mid \obs, \action)\tilde{u}(\obs, \action)}{\sum_{\tilde{\obs}, \tilde{\action}} \frac{u(\obs'_1)}{u(\obs'_2)}T(\obs'_2 \mid \tilde{\obs}, \tilde{\action}) \tilde{u}(\tilde{\obs}, \tilde{\action})}=\PP_{\tilde{u}}(\obs, \action \mid \obs'_2), %
\end{align*}
and so $\obs_1,\obs_2$ are backward KI.
\end{proof}

We next show that an ordering relation between policy visitation
probabilities is preserved through backward KI. This key structural
property allows us to use the backward KI relationship to find a
policy cover.

\begin{lemma}\label{lem:bki-policy-ratio}
Let $\observations' \subseteq \observations$ be a set of backward KI
observations and let $\observations'_1, \observations'_2 \subset
\observations'$. For any $\policy_1, \policy_2 \in \Upsilon$, we have
$\frac{\PP_{\policy_1}(\observations'_1)}{\PP_{\policy_2}(\observations'_1)}
=
\frac{\PP_{\policy_1}(\observations'_2)}{\PP_{\policy_2}(\observations'_2)}$.
\end{lemma}
\begin{proof}
Assume that $\observations' \subseteq \observations_h$ for some $h$,
which is without loss of generality, since if they are observable at
different time steps, then they are trivially
separable. From~\pref{lem:bki-ratio}, there exists a $u \in
\Delta(\observations)$ supported everywhere such that for any
$\obs'_1, \obs'_2 \in \observations'$ we have: $\frac{T(\obs'_1 \mid
  \obs, \action)}{u(\obs'_1)} = \frac{T(\obs'_2 \mid \obs,
  \action)}{u(\obs'_2)}$ for any $\obs \in \observations$ and $\action
\in \actions$.  Let $\policy$ be any policy and define its occupancy
measure at time $h-1$, $\xi_{h-1} \in \Delta(\observations_{h-1}
\times \actions)$, as $\xi_{h-1}(\obs,\action) \defeq
\EE_\policy\sbr{\one\{\obs_{h-1} = \obs, \action_{h-1} =
  \action\}}$. Then for any fixed $\tilde{\obs} \in \observations'$
and $j \in \{1,2\}$ we have
\begin{align*}
\PP_\policy(\observations'_j) = \EE_{(\obs, \action) \sim \xi}\sbr{\sum_{\obs' \in \observations'_j}\transition(\obs' \mid \obs, \action)} = \EE_{\obs, \action \sim \xi}\sbr{ \sum_{\obs' \in \observations'_j} \frac{u(\obs')}{u(\tilde{\obs})}\transition(\tilde{\obs} \mid \obs, \action)} = \frac{u(\observations'_j)}{u(\tilde{\obs})} \PP_{\policy}(\tilde{\obs}),
\end{align*}
where the second inequality follows
from~\pref{lem:bki-ratio}. Re-arranging, we have that
$\frac{\PP_\policy(\observations'_1)}{\PP_\policy(\observations'_2)} =
\frac{u(\observations'_1)}{u(\observations'_2)}$, and as the right
hand side does not depend on $\policy$, the result follows.
\end{proof}

Lastly, we show that a set of observations are backward KI, then a
single policy simultaneously maximizes the visitation probability to
these observations. Moreover, we can construct a reward function for
which this common policy is the reward maximizer. Recall that
$\Upsilon$ is the (unrestricted) set of \emph{all} policies.

\begin{lemma}
\label{lem:abstraction-policy}
Let $\observations' \subseteq \observations$ be a set of backward
kinematically inseparable observations. Then there exists a policy
$\policy \in \Upsilon$ that maximizes $\PP_\pi(\obs')$ simultaneously for all $\obs' \in \observations'$. 
Further, this policy is the optimal
policy for the internal reward function $\rewardf'(\obs,
\action) \defeq 1\{\obs \in \observations'\}$.
\end{lemma}
\begin{proof} 
Let $\obs_1, \obs_2 \in \observations'$ and define $\policy \defeq \argmax_{\policy
  \in \Upsilon} \PP_{\policy}(\obs_1)$. Let $\policy_2 \in \Upsilon$ be any other
policy. Then, by~\pref{lem:bki-policy-ratio}, we have
\begin{align*}
\PP_{\policy}(\obs_1) \geq \PP_{\policy_2}(\obs_1) \Leftrightarrow \PP_{\policy}(\obs_2) \geq \PP_{\policy_2}(\obs_2). 
\end{align*}
As the left hand side is true by definition of $\policy$, we see that
$\policy$ also maximizes the visitation probability for $\obs_2$. As
this is true for any $\obs_2$, we have that $\pi$ simultaneously
maximizes the visitation probability for all $\obs \in
\observations'$.

Clearly, for this policy and the specified reward function $\rewardf'$, we have
\begin{align*}
\EE_{\policy}\sbr{ \sum_h \rewardf'(\obs_h,\action_h)}  =
\PP_{\policy}(\observations') = \sum_{\obs' \in \observations'}
\PP_{\policy}(\obs') \geq \max_{\policy' \in \Upsilon}\sum_{\obs' \in
  \observations'} \PP_{\policy'}(\obs') = \max_{\policy' \in \Upsilon}
\EE_{\policy'}\sbr{\sum_h \rewardf'(\obs_h,\action_h)}.
\end{align*}
Here we are assuming $\observations$ is countable, as we have mentioned.
\end{proof}

\section{Analysis of Policy Search by Dynamic Programming}
\label{app:psdp-appendix}

This section provides a detailed statistical and computation analysis
of the Policy Search by Dynamic Programming (PSDP) algorithm, with
pseudocode in~\pref{alg:rl_via_homing_policies}. The main guarantee is
as follows:

\begin{theorem}
\label{thm:psdp-spread}
Let $(\hat{\policy}_1, \hat{\policy}_2, \cdots, \hat{\policy}_h) =
\psdpalg(\Psi, \rewardf, h, \policies, n)$ be the policy returned
by~\pref{alg:rl_via_homing_policies} using policy covers
$\Psi=\{\Psi_t\}_{t=1}^h$ where $\Psi_t$ is an $\alpha$-policy cover
for $\states_t$ and $|\Psi_t| \le N$ for all $t \in [h]$. Assume that
either $R$ is an internal reward function corresponding to time $h+1$,
or that $R$ is the external reward function and $h = H$, and
that~\pref{assum:realizability} holds. Then for any $\delta \in
(0,1)$, with probability at least $1- h\delta$ we have:
\begin{align*}
V(\hat{\policy}_{1:h}; R) \geq \max_{\policy_1,\ldots,\policy_h \in \policies} V(\policy_{1:h}; R) - \frac{Nh\errcsc}{\alpha} \qquad \mbox{where} \qquad \errcsc \defeq 4 \sqrt{\frac{|\actions|}{n} \ln\left(\frac{2|\policies|}{\delta}\right)}.
 \end{align*}
The algorithm runs in polynomial time with $h$ calls to the contextual
bandit oracle. 
\end{theorem}

Before turning to the proof, we state a standard generalization bound
for the contextual bandit problems created by the algorithm.  These
problems are induced by an underlying distribution $Q$ over tuples
$(\obs,\vec{r})$ where $\obs \in \observations$ and $\vec{r} \in
[0,1]^{|\actions|}$, and a \emph{logging policy}
$\policy_{\textrm{log}}$. Formally, we obtain tuples $(\obs, \action,
p, r) \sim Q_{\textrm{log}}$ where $(\obs,\vec{r})\sim Q$, $\action
\sim \policy_{\textrm{log}}(\obs)$, $p \defeq
\policy_{\textrm{log}}(\action \mid \obs)$ is the probability of
choosing the action for the current observation, and $r \defeq
\vec{r}(\action)$. In our application, we always have
$\policy_{\textrm{log}} \defeq \unf(\actions)$ so that $p = 1/|\actions|$.
Given a dataset of $n$ tuples $D := \{(\obs_i, \action_i,p_i,
r_i)\}_{i=1}^n \iidsim Q_{\textrm{log}}$, we invoke the contextual
bandit oracle, $\cboracle(D,\Pi)$, to find a policy $\hat{\pi}$. The
following proposition provides a performance guarantee for $\hat{\pi}$. 

\begin{proposition}
\label{prop:csc-call} 
Let $D \defeq \{(\obs_i, \action_i,p_i,r_i)\}_{i=1}^n \iidsim
Q_{\textrm{log}}$ be a dataset of $n$ samples from a contextual
bandit distribution $Q_{\textrm{log}}$ induced by the uniform
logging policy interacting with an underlying distribution
$Q$. Let $\hat{\pi} = \cboracle(D, \policies)$. Then for any
$\delta \in (0,1)$, with probability at least $1-\delta$, we have
\begin{equation*}
\EE_{(x,\vec{r})\sim Q}[\vec{r}(\hat{\policy}(\obs)] \ge \max_{\policy \in \policies} \EE_{(x,\vec{r})\sim Q}[\vec{r}(\policy(\obs)]  - \errcsc.
\end{equation*}
\end{proposition}
\begin{proof}
The proof is a standard generalization bound for contextual
bandits~\citep[c.f.,][]{langford2008epoch}. We provide a short proof
for completeness. 

For policy $\policy$, define $R(\policy) \defeq
\EE_{Q}\sbr{\vec{r}(\policy(\obs))}$, $\hat{r}_i(\policy) =
\frac{\one\{\action_i = \policy(\obs_i)\} r_i}{p_i} =
|\Acal|\one\{\action_i = \policy(\obs_i)\} r_i$, and observe that the
contextual bandit oracle finds the policy that maximizes
$\hat{R}(\policy) \defeq
\frac{1}{n}\sum_{i=1}^n\hat{r}_i(\policy)$. The random variables
$\hat{r}_i(\pi)$ satisfy the following useful properties:
\begin{align*}
\textrm{Unbiased:} ~~ & \EE_{Q_{\textrm{log}}}\sbr{\hat{r}(\pi)} = \EE_{Q} \sbr{\sum_{\action} \policy_{\textrm{log}}(\action \mid \obs) \frac{ \one\{\action = \policy(\obs)\} \vec{r}(\action)}{\policy_{\textrm{log}}(\action \mid \obs)}} = \EE_{Q}\sbr{\vec{r}(\policy(\obs))}.\\
\textrm{Low variance:} ~~ & \Var[\hat{r}(\policy)] \leq \EE_{Q_{\textrm{log}}}\sbr{\hat{r}^2(\policy)} \leq \EE_{Q_{\textrm{log}}}\sbr{\frac{\one\{\action = \policy(\obs)\}}{p^2}} = |\actions| \cdot \PP_{Q_{\textrm{log}}}[\action = \policy(\obs)] = |\actions|\\
\textrm{Range:} ~~ & \abr{\hat{r}(\policy)} \leq |\actions|.
\end{align*}

Therefore, using Bernstein's inequality (\pref{prop:bernstein}) and union bound we have that
for any $\delta \in (0,1)$, with probability at least $1-\delta$:
\begin{equation*}
\forall \policy \in \policies, \qquad \left| \hat{R}(\policy) - R(\policy) \right| \le \frac{2|\actions|}{3n} \ln\left(\frac{2|\policies|}{\delta} \right) + \sqrt{\frac{2|\actions|}{n} \ln\left(\frac{2|\policies|}{\delta}\right)} =: \Delta.
\end{equation*}
The contextual bandit oracle finds $\hat{\policy}$ that maximizes the
empirical quantity $\hat{R}(\policy)$, so, by a standard
generalization argument, we have
\begin{align*}
R(\hat{\policy}) \geq \hat{R}(\policy) - \Delta \geq \max_{\policy \in \policies}\hat{R}(\policy)- \Delta \geq \max_{\policy \in \policies} R(\policy) - 2\Delta.
\end{align*}
Of course as the reward vector is bounded in $[0,1]$ we always have
$R(\hat{\policy}) \geq \max_{\policy \in \policies} R(\policy) - 1$,
which means that with probability at least $1-\delta$, we have
\begin{align*}
R(\hat{\policy}) \geq \max_{\policy \in \policies} R(\policy) - \min\cbr{1, 2\Delta}.
\end{align*}
Finally, if $2\Delta \leq 1$ then
$\frac{4|\actions|\ln(2|\Pi|/\delta)}{3n} \leq
\sqrt{\frac{4|\actions|\ln(2|\Pi|/\delta)}{3n}}$. This observation
leads to the definition $\errcsc$.
\end{proof}

\begin{proof}[Proof of~\pref{thm:psdp-spread}]
Let $\policy^\star_1, \policy^\star_2, \cdots \policy^\star_h =
\argmax_{\policy_1, \policy_2, \cdots \policy_h \in
  \policies}V(\policy_{1:h}; R)$ be the optimal non-stationary policy,
for reward function $R$ with time horizon $h$.  $\psdpalg$ solves a
sequence of $h$ contextual bandit problems to learn policies
$\hat{\policy}_t$ for $t=h,\ldots,1$. The $t^{\textrm{th}}$ problem is
induced by a distribution $Q_t$ supported over $\observations_t
\times [0,1]^{|\actions|}$, which is defined inductively as follows:
The observations are induced by choosing a $\pi_t \sim \unf(\Psi_t)$
and executing $\pi_t$ for $t-1$ steps to visit $x_t$. The reward given
$\obs_t$ and an action $\action \in \actions$ is $R(x_{h+1})$ where
the trajectory is completed by first executing $\action$ from $\obs_t$
and then following $\hat{\pi}_{t+1:h}$. As in~\pref{prop:csc-call},
the contextual bandit dataset is induced by this distribution
$Q_t$ the uniform logging policy.

By~\pref{prop:csc-call} with probability at least $1-h\delta$, we have
that for all $t$, $\hat{\pi}_t$ satisfies
\begin{align*}
\EE_{(\obs,\vec{r})\sim Q_t} \sbr{\vec{r}(\hat{\pi}_t(\obs))} \geq \max_{\policy \in \policies}\EE_{(\obs,\vec{r})\sim Q_t} \sbr{\vec{r}(\policy_t(\obs))} - \errcsc,
\end{align*}
where $Q_t$ is as defined above. Using the definition of
$Q_t$, this guarantee may be written as
\begin{align*}
\forall t \in [h]: \EE_{\obs \sim Q_t} \sbr{V(x; \hat{\pi}_{t:h}, R)} \geq \max_{\pi \in \policies} \EE_{\obs \sim Q_t} \sbr{V(x; \pi\circ \hat{\pi}_{t+1:h}, R)} - \errcsc.
\end{align*}
Define $Q^\star_t \in \Delta(\observations_t)$ to be the
distribution of observations visited by executing
$\policy^\star_{1:t-1}$. By the performance difference lemma (\pref{lem:perf-diff})~\cite[c.f.,]{bagnell2004policy,kakade2003sample,ross2014reinforcement}, we have
\begin{align*}
V(\hat{\pi}_{1:h}; R) - V(\pi^\star_{1:h}; R) &= \sum_{t=1}^h \EE_{x_t \sim Q_t^\star}\sbr{ V(x_t; \pi^\star_t\circ\hat{\pi}_{t+1:h}, R) - V(x_t; \hat{\pi}_{t:h}, R)}\\
& \leq \sum_{t=1}^h \EE_{x_t \sim Q_t^\star}\sbr{ V(x_t; \tilde{\pi}^\star_t \circ \hat{\pi}_{t+1:h}, R) - V(x_t; \hat{\pi}_{t:h}, R)},
\end{align*}
where $\tilde{\pi}^\star_t(\obs) \defeq \argmax_{\action \in \actions}
\EE\sbr{ R(\obs,\action) + V(x_{t+1}; \hat{\pi}_{t+1:h}, R) \mid x_t
  =x, \action_t = \action}$, for all $\obs \in \observations_t$. With
this definition, the inequality here is immediate, by definition of
the value function.

\pref{assum:realizability} implies that $\tilde{\pi}_t^\star \in
\policies$ for each $t$, which is immediate for the external reward
function. If we are using the internal reward function with some $h <
H$, then by construction the internal reward function is defined only
at time $h+1$, so we may simply append arbitrary policies
$\hat{\pi}_{h+1:H}$ without affecting the reward or the value
function. Formally, we have
\begin{align*}
V(\hat{\pi}_{1:h}; R) - V(\pi^\star_{1:h}; R) &\leq \sum_{t=1}^h \EE_{x_t \sim Q_t^\star}\sbr{ V(x_t; \tilde{\pi}^\star_t \circ \hat{\pi}_{t+1:H}, R) - V(x_t; \hat{\pi}_{t:H}, R)}\\
& = \sum_{t=1}^h \EE_{x_t \sim Q_t}\sbr{ \frac{Q_t^\star(\obs_t)}{Q_t(\obs_t)}\rbr{V(x_t; \tilde{\pi}^\star_t \circ \hat{\pi}_{t+1:H}, R) - V(x_t; \hat{\pi}_{t:H}, R)}}\\
& \leq \sum_{t=1}^h \sup_{\obs_t} \abr{\frac{Q_t^\star(\obs_t)}{Q_t(\obs_t)}} \cdot \EE_{\obs_t \sim Q_t} \sbr{\abr{V(x_t; \tilde{\pi}^\star_t \circ \hat{\pi}_{t+1:H}, R) - V(x_t; \hat{\pi}_{t:H}, R)}}\\
& \leq \sum_{t=1}^h \sup_{\obs_t} \abr{\frac{Q_t^\star(\obs_t)}{Q_t(\obs_t)}} \cdot \errcsc.
\end{align*}
The first line appends $\hat{\pi}_{h+1:H}$ to the roll-out policy,
which as we argued does not affect the value function for any
policy. The second line simply introduces the distribution $Q_t$
that we used for learning $\hat{\pi}_t$. The third line is Holder's
inequality, and in the fourth line, we use the fact that
$\tilde{\pi}_t^\star$ is \emph{pointwise} better than $\hat{\pi}_t$,
so we can remove the absolute values. Then we simply use our guarantee
from~\pref{prop:csc-call}. 

We finish the proof by using the policy cover property (\pref{def:policy-cover}), namely that
\begin{align*}
\sup_{\obs_t} \abr{\frac{Q_t^\star(\obs_t)}{Q_t(\obs_t)}} = \sup_{\state_t} \abr{\frac{\PP_{\pi^\star_{1:t-1}}[\state_t]}{\frac{1}{|\Psi_{t}|}\sum_{\pi \in \Psi_{t}}\PP_\pi[\state_t]}} \leq \sup_{\state_t} \frac{\eta(\state_t)}{\frac{1}{N}\alpha\eta(\state_t)} = \frac{N}{\alpha}.
\end{align*}
Combining terms proves the theorem. 
\end{proof}

\section{Analysis of the $\oraclealg$ Algorithm}
\label{app:appendix-oracle-alg}

In this section we study the $\oraclealg$ algorithm, which we assume
is given a backward KI abstraction $\backwardabs: \observations
\rightarrow [N]$. We establish that with this abstraction, both
exploration and policy optimization are relatively
straightforward. Formally, we prove the following theorem:
\begin{theorem}[Main theorem for $\oraclealg$]
\label{thm:exp-oracle}
For any Block MDP, given a backward kinematic inseparability
abstraction $\backwardabs: \observations \rightarrow [N]$ such that
$\backwardabs \in \Phi_N$, and parameters $(\epsilon,\eta,\delta) \in
(0,1)^3$ with $\eta \leq \etamin$, $\oraclealg$ outputs policy covers
$\{\cover_h\}_{h=1}^H$ and a reward sensitive policy $\hat{\policy}$
such that the following guarantees hold, with probability at least
$1-\delta$:
\begin{enumerate}
\item For each $h \in [H]$, $\cover_h$ is a $\nicefrac{1}{2}$-policy cover for $\states_h$;
\item $V(\hat{\policy}) \geq \max_{\policy \in \nspolicies} V(\policy) - \epsilon$.
\end{enumerate}
The sample complexity is $\otil\rbr{ NH^2\npol + H\neval}
= \otil\rbr{N^2 H^3|\actions|\log(|\policies|/\delta)
  \rbr{NH/\etamin^2 + 1/\epsilon^2}}$, and the running time is
$\otil\rbr{NH^3\npol + H^2 \neval + NH^2\cdot \cpol(\npol)
  + H\cdot \cpol(\neval)}$, where $\npol$ and $\neval$ are defined
in~\pref{alg:oracle-exp}.
\end{theorem}

The intuition for the proof is as follows: Consider the time step
$h$. We know from~\pref{lem:abstraction-policy} that for any $i \in
[N]$ the policy $\policy \in \Upsilon$ (that is, an unrestricted
policy) maximizing the internal reward $R_{i}(\obs,\action) \defeq
\one\{\backwardabs(\obs) = i \land \step(\obs) = h\}$ simultaneously
maximizes $\PP_\policy(\obs)$ for each $\obs \in \observations_h$ such
that $\backwardabs(\obs) = i$.  In other words, if index $i$
corresponds to a set of observations $\observations^{(i)}_h \subseteq
\observations_h$, then this policy is a homing policy for every
observation $\obs \in \observations^{(i)}_h$,
\emph{simultaneously}. As this reward function is in our class of
internal rewards (which we will verify), we have realizability and
completeness by~\pref{assum:realizability}, so~\pref{thm:psdp-spread}
reveals that we learn a policy that approximately maximizes this
reward function. Naturally, this policy is an approximate homing
policy for all observations $\obs \in \observations_h^{(i)}$.  By
repeating this argument for all indices $i \in [N]$ and working
inductively from $h=1, \ldots, H$, we build a policy cover for all
$\observations_h$. This also gives a policy cover for
$\states_h$. Once we have this cover, learning the reward sensitive
policy is a straightforward application of~\pref{thm:psdp-spread}.

\begin{proof}
We first establish the policy cover guarantee, for which we
inductively build up $\cover_1,\ldots,\cover_H$. As the starting
distribution is fixed, we may simply take $\cover_1 \defeq \emptyset$,
which establishes the base case.

Now fix $h$ and let us condition on the event that
$\cover_1,\ldots,\cover_{h-1}$ are $\nicefrac{1}{2}$-policy covers for
$\states_1,\ldots,\states_{h-1}$, respectively. Fix $i \in
[N]$. Define the reward function $R_{i,h}(\obs,\action) \defeq
\one\{\timestep(\obs) = h \wedge \backwardabs(\obs) = i\}$ and
$\observations^{(i)}_h \defeq \{ \obs \in \observations_h \mid
\backwardabs(\obs) = i\}$. We will assume $\observations^{(i)}_h \neq
\emptyset$ otherwise $R_{i, h}$ is 0 everywhere so we can ignore this
case. As $\backwardabs \in \Phi_N$, it follows that $R_{i,h}$ is in
our class of internal reward functions, so~\pref{assum:realizability}
holds, and we may apply~\pref{thm:psdp-spread}. As such, the call to
$\psdpalg$ in this loop of $\oraclealg$ finds a policy $\policy_{i,h}$
with the following guarantee: with probability $1-h\delta$
\begin{align*}
V(\policy_{i,h}; R_{i,h}) \geq \max_{\policy' \in \policies^h} V(\policy'; R_{i,h}) - 2Nh\errcsc,
\end{align*}
where at the end of the proof we will instantiate $\errcsc$ with our
choice of $n$ and our failure probability. In particular it should be
understood that $\errcsc$ is defined in terms of $\npolo$. It is
straightforward to see from the definition of $R_{i,h}$ that
$V(\policy'; R_{i,h}) = \PP_{\policy'}(\observations^{(i)}_h)$.

Let $\state \in \states_h$ be such that $\observations_\state \cap
\observations^{(i)}_h \neq \emptyset$ where recall
$\observations_\state$ is the set of observations emitted from
$\state$. (Such a state must exist as $\observations^{(i)}_h$ is
non-empty.) By transitivity of backward KI, we have that
$\observations_\state \cup \observations^{(i)}_h$ is a backward KI
set, and so~\pref{lem:bki-policy-ratio} yields that for any policy
$\policy$, we have
$\PP_{\policy}(\observations^{(i)}_h)/\PP_{\policy_\state}(\observations^{(i)}_h)
= \PP_{\policy}(\state)/\PP_{\policy_\state}(\state)$. Recall
$\policy_\state \defeq \argmax_{\policy} \PP_{\policy} (\state)$ is
the homing policy of $\state$ and $\eta(\state) =
\max_{\policy}\PP_{\policy}(\state)$. Using this with our guarantee
for $\psdpalg$ we have:
\begin{equation*}
\PP_{\policy_{i, h}}(\state) \ge \eta(\state) - 2Nh \errcsc
\frac{\PP_{\policy_\state}(\state)}{\PP_{\policy_\state}(\observations^{(i)}_h)}.
\end{equation*}
This holds for any $i \in [N]$ and $\state \in \states$ satisfying
$\observations_\state \cap \observations_h^{(i)} \neq \emptyset$. The
failure probability for any single call to $\psdpalg$ is $h \delta \le
\horizon \delta$ and so the by union bound the probability of failure
for $N$ calls is $N \horizon \delta$.

For any $\state \in \states_h$ we define $j(\state) \defeq \arg\max_{i
  \in [N]} \PP_{\policy_\state}(\observations^{(i)}_h \cap
\observations_\state)$. We can then bound
$\PP_{\policy_\state}(\observations^{(j(\state))}_h)$ as:
\begin{equation*}
\PP_{\policy_\state}(\observations^{(j(\state))}_h) \ge \PP_{\policy_\state}(\observations^{(j(\state))}_h \cap \observations_\state) \ge \frac{1}{N} \sum_{i=1}^N \PP_{\policy_\state}(\observations^{(i)}_h \cap \observations_\state) = \frac{1}{N} \PP_{\policy_\state}(\cup_{i \in [N]}\observations^{(i)}_h \cap \observations_\state) = \frac{1}{N} \PP_{\policy_\state}(\state),
\end{equation*}
where we use the definition of $j(s)$, the fact that the maximum of a
set of values is greater than the mean, and the fact that
$\observations_h^{(i)}$ form a partition of
$\observations_h$. Applying this bound, we obtain
\begin{equation*}
\forall \state \in \states_h, \,\, \exists j(\state) \in [N], \qquad \PP_{\policy_{j(\state), h}}(\state) \ge \eta(\state) - 2N^2h \errcsc,
\end{equation*}
with probability of at least $1 - N\horizon \delta$. It is easy to see
that if $2N^2h \errcsc$ is at most $\frac{\etamin}{2}$ then $\cover_h
\defeq \{\policy_{i,h}\}$ is a $\frac{1}{2}$-policy cover of
$\states_h$. Thus we require 
\begin{equation*}
\npolo(\delta) \geq \frac{256 N^4H^2 |\actions|\log(2 |\policies|/\delta)}{\etamin^2}.
\end{equation*}
We apply this argument inductively for every $h\in[\horizon]$. Each
step in this induction fails with probability of at most $N \horizon
\delta$ so by union bound the failure probability is at most $N
\horizon^2 \delta$.

The reward sensitive learning is similar, we simply
apply~\pref{thm:psdp-spread} and observe that we have
$\nicefrac{1}{2}$-policy covers $\Psi_1,\ldots,\Psi_H$, so the error
guarantee is $2NH\errcsc$ with failure probability at most $\horizon
\delta$.  To bound this error by $\epsilon$ we want to set
$\neval(\delta)$ as:
\begin{equation*}
\neval(\delta) \geq \frac{64 N^2H^2 |\actions| \log(2|\policies|/\delta)}{\epsilon^2}.
\end{equation*}

The total failure probability is $N \horizon^2 \delta + \horizon
\delta \le 2 N \horizon^2 \delta$. Therefore, we rebind $\delta
\mapsto \frac{\delta}{2N\horizon^2}$ and set our hyperparameters to
$\npol\left( \frac{\delta}{2N\horizon^2}\right)$ and $\neval\left(
\frac{\delta}{2N\horizon^2}\right)$.

We conclude the proof by calculating the overall sample and
computational complexity. The sample compleixty is simply
$\order\rbr{\npolo H^2 N + \neval H}$ which evaluates to
\begin{align*}
\otil\rbr{ N^2 H^3|\actions| \log(|\policies|/\delta)\rbr{\frac{N^3H}{\etamin^2} + \frac{1}{\epsilon^2}}}.
\end{align*}
The running time is dominated by collecting the trajectories and
invoking the contextual bandit oracle. For the former, we collect
$\order\rbr{NH^2\npol + H\neval}$ trajectories in total, but each
trajectory takes $O(H)$ time. For the latter, we make $NH^2$ calls to
the oracle with $\npol$ samples and $H$ calls to the oracle with
$\neval$ samples.
\end{proof}

\pref{thm:exp-oracle} assumes access to a backward KI abstraction in
our decoding class $\Phi_N$. But does $\Phi_N$ contain a backward KI
abstraction?  The next proposition shows that our realizability
assumption guarantees that we can construct a backward KI abstraction
function using members of $\Phi_N$ for $N \geq \nkid$.

\begin{proposition}
\label{prop:bki-in-class}
For any $h \in [\horizon]$ and $N \ge \nkid$, there exists a
$\phi^{\textup{(B)}}_h\in\Phi_N$ such that for any $\obs'_1, \obs'_2 \in
\observations_h$ if $\phi^{\textup{(B)}}_h(\obs'_1) =
\phi^{\textup{(B)}}_h(\obs'_2)$ then $\obs'_1$ and $\obs'_2$ are
backward kinematically inseparable.
\end{proposition}
\begin{proof}
Fix $N \ge \nkid$ and $h \in [\horizon]$. Let $\rho \in
\Delta(\states_h)$ with $\supp(\rho) =
\states_h$. From~\pref{assum:realizability}, there exists $f_\rho \in
\Fcal_N$ such that for any $\obs \in \observations_{h-1}, \action \in
\actions, \obs' \in \observations_h$:
\begin{align*}
f_\rho(\obs, \action, \obs') = \frac{\transition(\decoder(\obs') \mid
  \decoder(\obs), \action)}{\transition(\decoder(\obs') \mid \decoder(\obs),
  \action) + \rho(\decoder(\obs')} = \frac{\transition(\obs' \mid \obs,
  \action)}{\transition(\obs' | \obs, \action) + \rho(\obs')}.
\end{align*}
From our construction of the function class $\Fcal_N$, we can
decompose $f_\rho$ into two state abstractions $\phi^{\textup{(F)}}_{h-1}, \phi^{\textup{(B)}}_h$, and a
lookup table $p$ such that $f_\rho(\obs, \action, \obs') =
p(\phi^{\textup{(F)}}_{h-1}(\obs), \action, \phi^{\textup{(B)}}_h(\obs'))$. Then, for any $\obs'_1, \obs'_2 \in
\observations_h$ with $\phi^{\textup{(B)}}_h(\obs'_1) = \phi^{\textup{(B)}}_h(\obs'_2)$, 
we have $f_\rho(\obs, \action, \obs'_1) = f_\rho(\obs, \action,
\obs'_2)$, which implies
\begin{equation*}
\frac{\transition(\obs'_1 \mid \obs, \action)}{\transition(\obs'_1 \mid \obs, \action) + \rho(\obs'_1)} = \frac{\transition(\obs'_2 \mid \obs, \action)}{\transition(\obs'_2 \mid \obs, \action) + \rho(\obs'_2)} \Rightarrow \frac{\transition(\obs'_1 \mid \obs, \action)}{\rho(\obs'_1)} = \frac{\transition(\obs'_2 \mid \obs, \action)}{\rho(\obs'_2)}.
\end{equation*}
If $\obs \not\in \observations_{h-1}$ then $ \frac{\transition(\obs'_1
  \mid \obs, \action)}{\rho(\obs'_1)} = \frac{\transition(\obs'_2 \mid
  \obs, \action)}{\rho(\obs'_2)} = 0$ trivially due to
layering. Therefore,~\pref{lem:bki-ratio} then reveals that $\obs'_1,
\obs'_2$ are backward KI, which proves the claim.
\end{proof}

\pref{prop:bki-in-class} enables us to define an abstraction
$\phi^{\textup{(B)}}: \observations \rightarrow [N]$ such that
$\phi^{\textup{(B)}}(\obs) = \phi^{\textup{(B)}}_{\step(\obs)}(\obs)$
for any $\obs \in \observations$ and where $\{\phi^{\textup{(B)}}_1,
\ldots, \phi^{\textup{(B)}}_\horizon\}$ are defined as
in~\pref{prop:bki-in-class} for a fixed $N \ge
\nkid$.~\pref{thm:exp-oracle} then holds using $\phi^{\textup{(B)}}$
as input, even though it may not be in $\Phi_N$, as long as each
element $\phi_h^{\textup{(B)}} \in \Phi_N$.
This can be verified by noting that the proof of~\pref{thm:exp-oracle}
uses only two properties of the input state abstraction: (i) it maps
backward maps backward KI observations from \emph{the same time step}
together and (ii) we have policy completeness with respect to the
internal reward functions defined by the abstraction functions. The
former is already accomplished by our construction of
$\phi^{\textup{(B)}}$. The latter can be verified by observing that
any internal reward function $R_{ih}$ defined using
$\phi^{\textup{(B)}}$ can be equivalently expressed as $R_{ih}(\obs,
\action) = \one\{\step(\obs) = h \land \phi^{\textup{(B)}}_h(\obs) =
i\}$ for all $\obs \in \observations, \action \in \actions$. Since
$\phi^{\textup{(B)}}_h \in \Phi_N$, we have policy completeness for
this reward function.

\section{Analysis for the $\homingalg$ algorithm}
\label{app:appendix-sample-complexity}

In this section we present the analysis for $\homingalg$. The proof is
inductive in nature, proceeding from time point $h=1$ to $h=\horizon$,
where for the $h^{\textrm{th}}$ step, we use the policy covers from
time points $h'=1,\ldots,h-1$ to construct the policy cover at time
$h$. Formally, the inductive claim is that for each $h$, given
$\alpha$-policy covers $\cover_1,\ldots,\cover_{h-1}$ over
$\states_1,\ldots,\states_{h-1}$, the $h^{\textrm{th}}$ iteration of
$\homingalg$ finds an $\alpha$-policy cover $\cover_h$ for
$\states_h$. We will verify the base case shortly, and we break down
the inductive argument into two main components: In the first part, we
analyze the contrastive estimation problem and introduce a coupling to
show that the supervised learning problem relates to backward
kinematic inseparability. In this second part, we use this coupling to
show that invoking $\psdpalg$ with the learned representation yields a
policy cover for time point $h$.

\paragraph{The base case.} 
The base case is that $\cover_1$ found by the algorithm is a policy
cover over $\states_1$. This is easy to see, since for any states
$\state \in \states_1$, we have $ \eta(\state)= \startdist(\state)$,
where recall that $\startdist$ is the starting stat distribution. We
can define $\cover_1$ to be any finite set of policies, which
immediately is a $1$-policy cover, but since we never actually execute
these policies, we simply set $\cover_1 = \emptyset$
in~\pref{alg:learn_homing_policy} (\pref{line:homer-initialization}).

\subsection{The supervised learning problem and a coupling}
In this subsection we analyze the supervised learning problem induced
by $\homingalg$, which is a form of contrastive estimation
(\pref{line:homer-learn-encoding-function}
in~\pref{alg:learn_homing_policy}). We reason about the Bayes optimal
predictor for this problem, obtain a finite-sample generalization
bound, and introduce a coupling to elucidate the connection to
backward kinematic inseparability. Fix a time point $h \in
\{2,\ldots,\horizon\}$ and inductively assume that we have $\alpha$
policy covers $\cover_{1},\ldots,\cover_{h-1}$ over
$\states_1,\ldots,\states_{h-1}$ respectively. For the rest of this
subsection, we suppress dependence on $h$ in observations, that is we
will always take $\obs \in \observations_{h-1}$, and $\obs' \in
\observations_h$.

The supervised learning problem at time $h$ is induced by a
distribution $D \in
\Delta(\observations_{h-1},\actions,\observations_h,\{0,1\})$, which
is defined as follows: Two tuples are obtained
$(\obs_1,\action_1,\obs_1'), (\obs_2,\action_2,\obs_2')$ are obtained
by sampling $\policy_1,\policy_2 \sim \unf(\cover_{h-1})$, and
executing the policies $\policy_1 \circ \unf(\actions)$ and $\policy_2
\circ \unf(\actions)$, respectively. Then with probability
$\nicefrac{1}{2}$ the sample from $D$ is $(\obs_1,\action_1,\obs_1,1)$
and with the remaining probability, the sample is
$(\obs_1,\action_1,\obs_2',0)$. Let $D(\obs,\action,\obs' \mid y)$ be
the conditional probability of the triple, conditional on the label
$y$. Let $\prior \in \Delta(\observations_h)$ denote the marginal
probability distribution over $\obs'$, that is $\prior(\obs') \defeq
\PP_{\policy \circ \unf: \policy \sim\unf(\cover_{h-1})}[\obs']$, and let
$\marginal \in \Delta(\observations_{h-1})$ be the marginal
distribution over $\obs$, that is $\marginal(\obs) \defeq \PP_{\policy
  \sim \unf(\cover_{h-1})}[\obs]$. These definitions are lifted to the
state spaces $\states_{h-1}$ and $\states_h$ in the natural way.

With these definition, we have
\begin{align*}
D(\obs,\action,\obs' \mid y=1) = \frac{\marginal(\obs)}{|\actions|}\cdot \transition(\obs' \mid \obs,\action),\qquad
D(\obs,\action,\obs' \mid y=0) = \frac{\marginal(\obs)}{|\actions|}\cdot\prior(\obs').
\end{align*}

The first lemma uses the fact that $\cover_{h-1}$ is an $\alpha$-policy
cover to lower bound the marginal probability $\prior(\state_h)$,
which ensure we have adequate coverage in our supervised learning problem.
\begin{lemma}
\label{lem:bound-rho-prior}
If $\cover_{h-1}$ is an $\alpha$-policy cover over $\states_{h-1}$, then for any $\state \in \states_h$, we have $\prior(\state)\geq \frac{\alpha\eta(\state)}{N|\actions|}$.
\end{lemma}
\begin{proof}
For any $\state \in \states_h$, we first upper bound $\eta(\state)$ by
\begin{align*}
\eta(\state) &= \sup_{\policy \in \nspolicies} \PP_{\policy}(\state) =  \sup_{\policy \in \nspolicies} \sum_{\state_{h-1} \in \states_{h-1}} \PP_{\policy}[\state_{h-1}] \EE_{\obs \sim \emf(. \mid \state_{h-1})}\left[\sum_{\action \in \actions} \policy(\action \mid \obs) \transition(\state \mid \state_{h-1}, \action)\right] \\
&
\le \sum_{\state_{h-1} \in \states_{h-1}}  \sup_{\policy \in \nspolicies} \PP_{\policy}[\state_{h-1}] \sum_{\action \in \actions} \transition(\state \mid \state_{h-1}, \action) =  \sum_{\state_{h-1} \in \states_{h-1}}  \eta(\state_{h-1}) \sum_{\action \in \actions} \transition(\state \mid \state_{h-1}, \action).
\end{align*} 
We can also lower bound $\prior$ as
\begin{align*}
\prior(\state) &= \sum_{\substack{\state_{h-1}\in\states_{h-1} \\ \action \in \actions}} \frac{\PP_{\policy \sim \unf(\cover_{h-1})}[\state_{h-1}]}{|\actions|} \transition(\state \mid \state_{h-1},\action)
\ge  \frac{\alpha}{N |\actions|}\sum_{\substack{\state_{h-1}\in\states_{h-1} \\ \action \in \actions}} \eta(\state_{h-1}) \transition(\state \mid \state_{h-1}, \action) 
\ge \frac{\alpha \eta(\state)}{N|\actions|}.
\end{align*}
Here the first identity expands the definition, and in the first
inequality we use the fact that $\cover_{h-1}$ is an $\alpha$-policy
cover. The last inequality uses our upper bound on $\eta(\state)$.
\end{proof}

The next lemma characterizes the Bayes optimal predictor for square
loss minimization with respect to $D$. Recall that the Bayes optimal
classifier is defined as
\begin{align*}
f^\star \defeq \argmin_{f} \EE_{(\obs,\action,\obs',y) \sim D}\sbr{\rbr{f(\obs,\action,\obs') - y}^2}
\end{align*}
where the minimization is over \emph{all} measurable functions.
\begin{lemma}
\label{lem:p-form}
The Bayes optimal predictor for square loss minimization over $D$ is
\begin{align*}
f^\star(\obs,\action,\obs') \defeq \frac{\transition(\decoder(\obs') \mid \decoder(\obs),a)}{\transition(\decoder(\obs') \mid \decoder(\obs),a) + \prior(\decoder(\obs'))}
\end{align*}
Under~\pref{assum:realizability}, we have that $f^\star \in \Fcal_N$
for any $N \geq \nkid$.
\end{lemma}
\begin{proof}
As we are using the square loss, the Bayes optimal predictor is the
conditional mean, so $f^\star(\obs,\action,\obs') = \EE_D[y \mid
  (\obs,\action,\obs')] = D(y=1 \mid \obs, \action,\obs')$. By Bayes
rule and the fact that the marginal probability for both labels is
$\nicefrac{1}{2}$, we have
\begin{align*}
D(y=1 \mid \obs,\action,\obs') &= \frac{D(\obs,\action,\obs' \mid y=1)}{D(\obs,\action,\obs' \mid y=1) + D(\obs,\action,\obs' \mid y=0)} = \frac{\transition(\obs' \mid \obs,\action)}{\transition(\obs'\mid \obs,\action) + \prior(\obs')}\\
& = \frac{\transition(\decoder(\obs') \mid \decoder(\obs),a)}{\transition(\decoder(\obs') \mid \decoder(\obs),a) + \prior(\decoder(\obs'))}. \tag*\qedhere
\end{align*}
\end{proof}

Now that we have characterized the Bayes optimal predictor, we turn to
the learning rule. We perform empirical risk minimization over $n$ iid
samples from $D$ to learn a predictor $\hat{f}\in\Fcal_N$ (We will
bind $n = \nreg$ toward the end of the proof). As $\Fcal_N$ has
pointwise metric entropy growth rate $\ln \Ncal(\Fcal_N,\varepsilon)
\leq c_0d_N\ln(1/\varepsilon)$, a standard square loss generalization
analysis (see~\pref{prop:sq-loss}) yields the following corollary,
which follows easily from~\pref{prop:sq-loss}.

\begin{corollary}
\label{corr:sq-loss}
For any $\delta \in (0,1)$ with probability at least $1-\delta$, the
empirical risk minimizer, $\hat{f}$ based $n$ iid samples from $D$
satisfies\footnote{As we remark in~\pref{app:supporting}, sharper
  generalization analyses are possible, with more refined notions of
  statistical complexity. Such results are entirely composable with
  our analyses.}
\begin{align*}
\EE_{D}\sbr{\rbr{\hat{f}(\obs,\action,\obs') - f^\star(\obs,\action,\obs')}^2} \leq \errreg \mbox{ with } \errreg \defeq \frac{16\rbr{\ln |\Phi_N| + N^2|\actions|\ln(n)+
  \ln(\nicefrac{2}{\delta})}}{n}.
\end{align*}
\end{corollary}
\begin{proof}
The proof follows from a bound on the pointwise covering number of the
class $\Fcal_N$. For any $\varepsilon > 0$ we first form a cover of
the class $\Wcal_N$ by discretizing the output space to $Z \defeq
\{\epsilon,\ldots,\lfloor\nicefrac{1}{\epsilon}\rfloor \epsilon\}$,
and letting $W_N$ be all functions from $[N]\times \actions \times [N]
\to Z$. Clearly we have $|W_N| \leq
(\nicefrac{1}{\epsilon})^{N^2|\actions|}$, and it is easy to see that
$W_N$ is a pointwise cover for $\Wcal_N$. Then we form $F_N = \{
(x,a,x') \mapsto w(\phi(x),a,\phi'(x')) : w \in W_N, \phi,\phi' \in
\Phi_N\}$, which is clearly a pointwise cover for $\Fcal_N$ and has
size $|\Phi_N|^2 |W_N|$. In other words, the pointwise log-covering
number is $N^2|\actions| \ln(1/\varepsilon) + 2\ln |\Phi_N|$, which we
plug into the bound in~\pref{prop:sq-loss}. Taking $\varepsilon = 1/n$
there the bound from~\pref{prop:sq-loss} is at most
\begin{align*}
\frac{6}{n} + \frac{16\ln |\Phi_N| + 8N^2|\actions|\ln(n)+
  8\ln(2/\delta)}{n} \leq \errreg. \tag*\qedhere
\end{align*}
\end{proof}

\paragraph{The coupling.}
\newcommand{\dcoup}{D_{\textup{coup}}}

For the next step of the proof, we introduce a useful coupling
distribution based on $D$. Let $\dcoup \in
\Delta(\observations_{h-1}\times\actions\times\observations_h\times\observations_h)$
have density $\dcoup(\obs,\action,\obs_1',\obs_2') =
D(\obs,\action)\prior(\obs_1')\prior(\obs_2')$. That is, we sample
$\obs,\action$ by choosing $\policy \sim\unf(\cover_{h-1})$, rolling in, and
then taking a uniform action $\action_{h-1} \sim
\unf(\actions)$. Then, we obtain $\obs_1',\obs_2'$ independently by
sampling from the marginal distribution $\prior$ induced by $D$.
 
It is also helpful to define the shorthand notation $V:
\observations_{h}\times\observations_h\times\observations_{h-1}\times\actions
\to \RR$ by
\begin{align*}
V(\obs_1',\obs_2',\obs,\action) \defeq \frac{\transition(\decoder(\obs_1') \mid \decoder(\obs),\action)}{\prior(\decoder(\obs_1'))} - \frac{\transition(\decoder(\obs_2') \mid \decoder(\obs),\action)}{\prior(\decoder(\obs_2'))}.
\end{align*}
This function is lifted to operate on states $\states_h$ in the
natural way. Note also that, as $\prior(\cdot) > 0$ everywhere, $V$ is
well-defined. Observe that $V$ is related to the notion of backward
kinematic inseparability. Finally, define
\begin{align*}
b_i \defeq \EE_{\obs' \sim \prior}\sbr{\one\{\backabs(\obs') = i\}},
\end{align*}
which is the prior probability over the learned abstract state $i$,
where $\backabs$ is the learned abstraction function for time $h$
implicit in the predictor $\hat{f}$. In the next lemma, we show that
$\backabs$ approximately learns a backward KI abstraction by relating
the excess risk of $\hat{f}$ to the performance of the decoder via the
$V$ function.

\begin{lemma}
\label{lem:coupling}
Let $\hat{f} =: (\hat{w},\forabs,\backabs)$ be the empirical risk
minimizer on $n$ iid samples from $D$, that is the output of
$\cpeoracle(\Fcal_N,\Dcal)$. Under the $1-\delta$ event
of~\pref{corr:sq-loss}, for each $i \in [N]$ we have
\begin{align*}
\EE_{\dcoup}\sbr{\one\{\backabs(\obs_1') = i = \backabs(\obs_2')\} \abr{V(\obs_1',\obs_2',\obs,\action)}} \leq 8 \sqrt{b_i \errreg}.
\end{align*}
\end{lemma}
\begin{proof}
For the proof, it is helpful to introduce a second coupled
distribution $\dcoup'$ in which $\obs,\action$ are sampled as before,
but now $\obs_1',\obs_2' \iidsim D(\cdot \mid \obs,\action)$, instead
of from the prior. Note that this condition probability is $D(\obs'
\mid \obs,\action) =\nicefrac{1}{2} \emf(\obs' \mid \decoder(\obs))\rbr{
  \transition(\decoder(\obs') \mid \decoder(\obs),\action)+
  \prior(\decoder(\obs'))}$. To translate from $\dcoup$ to $\dcoup'$
we expand the definition of $V$ and introduce $f^\star$. The main
observation here is that $V$ is normalized by $\prior(\cdot)$ but
$f^\star$ is normalized, essentially by $D(\cdot \mid \obs,\action)$.
\begin{align*}
V(\obs_1',\obs_2',\obs,\action) & = \frac{\prior(\obs_2')\transition(\obs_1' \mid \obs,\action) - \prior(\obs_1')\transition(\obs_2' \mid \obs,\action)}{\prior(\obs_1')\rho_h(\obs_2')}\\
& = \frac{4D(\obs_1' \mid \obs,\action) D(\obs_2' \mid \obs,\action)}{\prior(\obs_1')\prior(\obs_2')} \cdot \rbr{f^\star(\obs,\action,\obs_1') - f^\star(\obs,\action,\obs_2')}.
\end{align*}
The last step follows since, in the first term the emission
distributions cancel, while the cross terms cancel when we introduce
the least common multiple in the term $f^\star(\obs,\action,\obs_1') -
f^\star(\obs,\action,\obs_2')$. Specifically
\begin{align*}
f^\star(\obs,\action,\obs_1') - f^\star(\obs,\action,\obs_2') &= \frac{\transition(\decoder(\obs_1') \mid \decoder(\obs),\action)}{\transition(\decoder(\obs_1') \mid \decoder(\obs),\action) + \prior(\decoder(\obs_1'))} - \frac{\transition(\decoder(\obs_2') \mid \decoder(\obs),\action)}{\transition(\decoder(\obs_2') \mid \decoder(\obs),\action) + \prior(\decoder(\obs_2')}\\
& = \frac{\prior(\decoder(\obs_2'))\transition(\decoder(\obs_1') \mid \decoder(\obs),\action) - \prior(\decoder(\obs_1'))\transition(\decoder(\obs_2') \mid \decoder(\obs),\action)}{\rbr{\transition(\decoder(\obs_1') \mid \decoder(\obs),\action) + \prior(\decoder(\obs_1'))}{\transition(\decoder(\obs_2') \mid \decoder(\obs),\action) + \prior(\decoder(\obs_2'))}}.
\end{align*}
Nevertheless, this calculation lets us translate from $\dcoup$ to
$\dcoup'$ while moving from $V$ to $f^\star$. For shorthand, let
$\Ecal_i \defeq \one\{\backabs(\obs_1') = i = \backabs(\obs_2')\}$, so
the above derivation yields
\begin{align}
\EE_{\dcoup}\sbr{\Ecal_i \cdot \abr{V(\obs_1',\obs_2',\obs,\action)}} = 4 \EE_{\dcoup'}\sbr{\Ecal_i \cdot \abr{f^\star(\obs,\action,\obs_1') - f^\star(\obs,\action,\obs_2')}}.\label{eq:w-translation}
\end{align}
Now that we have introduce $f^\star$, we can introduce $\hat{f}$ and
relate to the excess risk
\begin{align*}
& \EE_{\dcoup'}\sbr{\Ecal_i \cdot \abr{f^\star(\obs,\action,\obs_1') - f^\star(\obs,\action,\obs_2')}} \\
& ~~~~ \leq \EE_{\dcoup'} \sbr{\Ecal_i \cdot \rbr{\abr{f^\star(\obs,\action,\obs_1') - \hat{f}(\obs,\action,\obs_1')} + \abr{f^\star(\obs,\action,\obs_2') - \hat{f}(\obs,\action,\obs_1')}}}\\
& ~~~~ = \EE_{\dcoup'} \sbr{\Ecal_i \cdot \rbr{\abr{f^\star(\obs,\action,\obs_1') - \hat{f}(\obs,\action,\obs_1')} + \abr{f^\star(\obs,\action,\obs_2') - \hat{f}(\obs,\action,\obs_2')}}}\\
& ~~~~ \leq 2 \EE_{D} \sbr{\one\{\backabs(\obs') = i\} \abr{f^\star(\obs,\action,\obs') - \hat{f}(\obs,\action,\obs')}} \leq 2 \sqrt{b_i \EE_D\sbr{\rbr{f^\star(\obs,\action,\obs') - \hat{f}(\obs,\action,\obs')}^2}} \\
& ~~~~ \leq 2\sqrt{b_i \errreg}.
\end{align*}
The first step is the triangle inequality. The key step is the
second one: the identity uses the fact that under event $\Ecal_i$, we
know that $\backabs(\obs_1') = \backabs(\obs_2')$, which, by
the bottleneck structure of $\hat{f}$, yields that
$\hat{f}(\obs,\action,\obs_1') = \hat{f}(\obs,\action,\obs_2')$. In
the third step we combine terms, drop the dependence on the other
observaiton, and use the fact that $\dcoup'$ shares marginal
distributions with $D$. Finally, we use the Cauchy-Schwarz inequality,
the fact that $\EE_{D} \one\{\backabs(\obs') = i\} = b_i$,
and~\pref{corr:sq-loss}. Combining with~\pref{eq:w-translation} proves
the lemma.
\end{proof}

\paragraph{An aside on realizability.}
Given the bottleneck structure of our function class $\Fcal_N$, it is
important to ask whether realizability is even feasible. In particular
for a bottleneck capacity of $N$, each $f \in \Fcal_N$ has a range of
at most $N^2|\actions|$ discrete values. If we choose $N$ to be too
small, we may not have enough degrees of freedom to express
$f^\star$. By inspection $f^\star$ has a range of at most
$|\states|^2|\actions|$, so certainly a bottleneck capacity of $N \geq
|\states|$ suffices. In the next proposition, we show that in fact $N
\geq \nkid$ suffices, which motivates the condition
in~\pref{assum:realizability}.

\begin{proposition}
Fix $h \in \{2,\ldots,\horizon\}$. If $\obs_1,\obs_1 \in \observations_{h-1}$
are kinematically inseparable observations, then for all $\obs' \in
\observations_h$ and $\action \in \actions$, we have
$f^\star(\obs_1,\action,\obs') =
f^\star(\obs_2,\action,\obs')$. Analogously, if $\obs_1',\obs_2' \in
\observations_h$ are kinematically inseparable, then for all $\obs \in
\observations_{h-1}$ and $\action \in \actions$, we have
$f^\star(\obs,\action,\obs_1') = f^\star(\obs,\action,\obs_2')$.
\end{proposition}
\begin{proof}
We prove the forward direction first.  As $\obs_1, \obs_2$ are forward
KI, we have
\begin{align*}
f^\star(\obs_1, \action, \tilde{\obs})  = \frac{\transition(\tilde{\obs} \mid \obs_1, \action) }{\transition(\tilde{\obs} \mid \obs_1, \action)  + \prior(\tilde{\obs})} = \frac{\transition(\tilde{\obs} \mid \obs_2, \action) }{\transition(\tilde{\obs} \mid \obs_2, \action)  + \prior(\tilde{\obs})} = f^\star(\obs_2, \action, \tilde{\obs}).  
\end{align*}
For the backward direction, as $\obs_1', \obs_2'$ are backward KI,
from~\pref{lem:bki-ratio} there exists $u \in \Delta(\observations)$
with $\supp(u) = \observations$ such that $\frac{\transition(\obs_1'
  \mid \obs, \action)}{u(\obs_1')} = \frac{\transition(\obs_2' \mid
  \obs, \action)}{u(\obs_2')}$. Further, $\prior$ satisfies
$\prior(\obs_1') =
\frac{u(\obs_1')}{u(\obs_2')}\prior(\obs_2')$. Thus, we obtain
\begin{align*}
f^\star(\obs,  \action, \obs_1')  = \frac{\transition(\obs_1' \mid \obs, \action) }{\transition(\obs_1' \mid \obs, \action)  + \prior(\obs_1')} =  \frac{ \frac{u(\obs_1')}{u(\obs_2')}\transition(\obs_2' \mid \obs, \action) }{ \frac{u(\obs_1')}{u(\obs_2')}\transition(\obs_2' \mid \obs, \action)  +  \frac{u(\obs_1')}{u(\obs_2')}\prior(\obs_2')} 
=f^\star(\obs,  \action, \obs_2').\tag*\qedhere
\end{align*}
\end{proof}

\subsection{Building the policy cover}
\pref{lem:coupling} relates our learned decoder function $\backabs$ to
backward KI. For some intuition as to why, it is helpful to consider
the asymptotic regime, where $n \to\infty$, so that $\errreg \to
0$. In this regime,~\pref{lem:coupling} shows that whenever $\backabs$
maps two observations to the same abstract state, these observations
must have $V = 0$ for all $\obs,\action$. As our distribution has full
support, by~\pref{lem:bki-ratio}, these observations must be backward
KI.  Of course, this argument only applies to the asymptotic
regime. In this section, we establish a finite-sample analog, and we
show how using the internal reward functions induced by
$\backabs$ in $\psdpalg$ yields a policy cover for $\states_h$.

The first lemma is a comparison lemma, which lets us compare
visitation probabilities for two policies.
\begin{lemma}
\label{lem:policy-comparison}
Assume $\cover_{h-1}$ is an $\alpha$ policy cover for
$\states_{h-1}$. Then for any two policies $\policy_1,\policy_2$ and any
state $\state \in \states_h$, we have
\begin{align*}
\PP[\state \mid \policy_1] - \PP[\state \mid \policy_2] \leq \min_{i \in [N]}\frac{\prior(\state)}{b_i} \rbr{\PP[\backabs(\obs') = i \mid \policy_1] - \PP[\backabs(\obs') = i \mid \policy_2]} + \frac{16 N |\actions| \prior(\state)}{\alpha b_i^{3/2}} \sqrt{\errreg}.
\end{align*}
\end{lemma}
\begin{proof}
The key step is to observe that by the definition of $V$
\begin{align*}
\forall \obs_2',\obs,\action: \EE_{\obs_1' \sim \prior}\sbr{ \one\{\backabs(\obs_1') = i\} V(\obs_1',\obs_2',\obs,\action)}  = \sum_{\obs_1'} \one\{\backabs(\obs_1') = i\} \transition(\obs_1' \mid \obs,\action) - \frac{b_i \transition(\obs_2' \mid \obs,\action)}{\prior(\obs_2')}.
\end{align*}
Using this identity, we may express the visitation probability for a policy as
\begin{align*}
& \PP[\state \mid \policy_1] = \EE_{(\obs,\action) \sim \policy_1} \sum_{\obs_2'} \one\{\decoder(\obs_2') = \state\} \transition(\obs_2' \mid \obs,\action)\\
& = \frac{1}{b_i} \EE_{(\obs,\action) \sim \policy_1, \obs_2'\sim \prior}\sum_{\obs_1'}\one\{\decoder(\obs_2') = \state \wedge \backabs(\obs_1') = i\}\rbr{ \transition(\obs_1' \mid \obs,\action) - \prior(\obs_1')V(\obs_1',\obs_2',\obs,\action)}\\
& = \frac{\prior(s)}{b_i} \PP[\backabs(\obs') = i \mid \policy] - \frac{1}{b_i}\EE_{\dcoup}\sbr{\frac{\policy(\obs,\action)}{D(\obs,\action)} \one\{\decoder(\obs_2') = \state \wedge \backabs(\obs_1') = i\}V(\obs_1',\obs_2',\obs,\action)}.
\end{align*}
Here we are using the shorthand $\policy(\obs,\action) = \PP[\obs \mid
  \policy] \policy(\action \mid \obs)$ for the policy occupancy
measure, with a similar notation the distribution $D$ induced by our
policy cover $\cover_{h-1}$.  Using the inductive hypothesis that
$\cover_{h-1}$ is a $\alpha$-policy cover
(essentially~\pref{lem:bound-rho-prior}), we have
\begin{align*}
\abr{\frac{\policy(\obs,\action)}{D(\obs,\action)}} = \abr{\frac{\PP[\obs \mid \policy] \cdot \policy(\action \mid \obs)}{\EE_{\policy' \sim \unf(\cover_{h-1})} \PP[\obs \mid \policy'] \cdot \nicefrac{1}{|\actions|}}} \leq |\actions| \abr{\frac{\PP[\decoder(\obs) \mid \policy]}{\EE_{\policy' \sim \unf(\cover_{h-1})}\PP[\decoder(\obs)\mid \policy']}} \leq \frac{|\actions| N}{\alpha}.
\end{align*}
Combining, the second term above is at most
\begin{align*}
\frac{N |\actions|}{\alpha b_i} \EE_{\dcoup} \sbr{ \one\{\decoder(\obs_2') = \state \wedge \backabs(\obs_1') = i\}\abr{V(\obs_1',\obs_2',\obs,\action)}}.
\end{align*}
Let us now work with just the expectation. Recall that we can lift the
definition of $V$ to operate on states $\state \in \states_h$ in lieu
of observations. Using this fact, we have that under the probability
$1-\delta$ event of~\pref{lem:coupling}
\begin{align*}
& \EE_{\dcoup} \sbr{ \one\{\decoder(\obs_2') = s \wedge \backabs(\obs_1') = i\}\abr{V(\obs_1',\obs_2',\obs,\action)}} = \prior(\state) \EE_{D} \sbr{\one\{\backabs(\obs_1') = i\} \abr{V(\obs_1',\state,\obs,\action)}}\\
& = \prior(\state) \EE_D \sbr{\one\{\backabs(\obs_1') = i\} \abr{V(\obs_1',\state,\obs,\action)} \frac{\EE_{\obs_2' \sim \prior}\one\{\backabs(\obs_2') = i\}}{b_i}}\\
& = \frac{\prior(\state)}{b_i} \EE_{\dcoup}\sbr{\one\{ \backabs(\obs_1') = i = \backabs(\obs_2')\} \abr{V(\obs_1',\obs_2',\obs,\action)}}\\
& \leq \frac{\prior(\state)}{b_i} \EE_{\dcoup}\sbr{\one\{ \backabs(\obs_1') = i = \backabs(\obs_2')\} \abr{V(\obs_1',\obs_2',\obs,\action)}}
 \leq \prior(\state) \cdot 8\sqrt{\errreg/b_i},
\end{align*}
Putting things together, and using the same bound for the second
policy, we have the following comparison inequality, which holds in
the $1-\delta$ event of~\pref{corr:sq-loss}
\begin{align*}
\PP[\state \mid \policy_1] - \PP[\state \mid \policy_2] &\leq \frac{\prior(\state)}{b_i}\rbr{\PP[\backabs(\obs') = i \mid \policy_1] - \PP[\backabs(\obs') = i \mid \policy_2]} + \frac{16 N |\actions| \prior(\state)}{\alpha b_i^{3/2}} \sqrt{\errreg}.
\end{align*}
As this calculation applies for each $i$, we obtain the result. 
\end{proof}

In the next lemma, we introduce our policy cover.
\begin{lemma}
\label{lem:policy-cover-guarantee}
Assume that $\cover_1,\ldots,\cover_{h-1}$ are $\alpha$ policy covers
for $\states_1,\ldots,\states_{h-1}$, each of size at most $N$. Let
$\cover_h \defeq \cbr{\hat{\policy}_{i,h} : i \in [N]}$ be the policy
cover learned at step $h$ of $\homingalg$. Then in the $\geq
1-(1+NH)\delta$ probability event that the $N$ calls to $\psdpalg$
succeed and~\pref{corr:sq-loss} holds, we have that for any state
$\state \in\states_h$, there exists an index $i \in [N]$ such that
\begin{align*}
\PP[\state \mid \hat{\policy}_{i,h}] \geq \eta(\state) - \frac{N^2 h \errcsc}{\alpha} - \frac{16N^3|\actions|^{3/2}}{\alpha^{3/2}}\sqrt{\errreg/\eta(\state)}
\end{align*}
\end{lemma}
\begin{proof}
Let us condition on the success of~\pref{corr:sq-loss}, as well as the
success of the $N$ calls to $\psdpalg$. As the former fails with
probability at most $\delta$, and each call to $\psdpalg$ fails with
probability at most $H\delta$, the total failure probability here is
$(1+NH)\delta$.

In this event, by~\pref{thm:psdp-spread}, and the definition of the
internal reward function $R_i$, we know that
\begin{align*}
\PP[\backabs(\obs') = i | \hat{\policy}_{i,h}] \geq \max_{\policy \in \nspolicies} \PP[\backabs(\obs') = i | \policy] - \frac{Nh\errcsc}{\alpha}.
\end{align*}
Plugging this bound into~\pref{lem:policy-comparison}, we get
\begin{align*}
\PP[\state \mid \policy] \leq \PP[\state \mid \hat{\policy}_{i,h}] + \frac{Nh\prior(\state)\errcsc}{\alpha b_i} + \frac{16N|\actions|\prior(\state)}{\alpha b_i^{3/2}}\sqrt{\errreg}.
\end{align*}
This bound also holds for all $i \in [N]$. To optimize the bound, we
should choose the index $i$ that is maximally correlated with the
state $\state$. To do so, define $P_{\state,i} \defeq \EE_{\obs' \sim \prior}
\one\{\decoder(\obs') = \state, \backabs(\obs') = i\}$, and let us
choose $i(\state) = \max_i P_{\state,i}$. This index satisfies
\begin{align*}
b_{i(\state)} = \sum_{\state'} P_{\state',i(\state)} \geq P_{\state,i(\state)} = \max_{i} P_{\state,i} \geq \frac{1}{N}\sum_{i=1}^N P_{\state,i} = \frac{\prior(s)}{N}
\end{align*}
Plugging in this bound, we see that for every $\state$, there exists $i \in [N]$ such that
\begin{align*}
\eta(\state) = \max_{\policy \in \nspolicies} \PP[\state \mid\policy] \leq \PP[\state \mid \hat{\policy}_{i,h}] + \frac{N^2h\errcsc}{\alpha} + \frac{16N^{5/2}|\actions|}{\alpha}\sqrt{\errreg/\prior(\state)}
\end{align*}
We conclude the proof by introducing the lower bound on
$\prior(\state) \geq \frac{\alpha\eta(\state)}{N|\actions|}$
from~\pref{lem:bound-rho-prior} and re-arranging.
\end{proof}

\subsection{Wrapping up the proof}
\pref{lem:policy-cover-guarantee} is the core technical result, which
certifies that our learned policy cover at time $h$ yields good
coverage. We are basically done with the proof; all that remains is to
complete the induction, set all of the parameters, and take a union
bound. 

\paragraph{Union bound.} 
For each $h \in [\horizon]$ we must invoke~\pref{corr:sq-loss} once,
and we invoke~\pref{thm:psdp-spread} $N$ times. We also
invoke~\pref{thm:psdp-spread} once more to learn the reward sensitive
policy. Thus the total failure probability is $\horizon(\delta_1 +
N\horizon\delta_2) + \horizon \delta_3$ where $\delta_1$ appears in
$\errreg$, $\delta_2$ appears in $\errcsc$ for the internal reward
functions, and $\delta_3$ appears in $\errcsc$ for the external reward
functions. We therefore take $\delta_1 = \delta/(3\horizon)$ and
$\delta_2 = \frac{\delta}{3N\horizon^2}$ and $\delta_3 =
\frac{\delta}{3\horizon}$, which gives us the settings
\begin{align*}
\errcsc = 4\sqrt{\frac{|\actions|}{\npolo} \ln \rbr{\frac{4N\horizon^2 |\policies|}{\delta}}}, \qquad \errreg = \frac{16 \rbr{\ln(|\Phi_N|) + N^2|\actions|\ln(\nreg) + \ln(6H/\delta)}}{\nreg}, %
\end{align*}
for the inductive steps. With these choices, the total failure
probability for the algorithm is $\delta$.

\paragraph{The policy covers.}
Fix $h \in [\horizon]$ and inductively assume that
$\cover_{1},\ldots,\cover_{h-1}$ are $\nicefrac{1}{2}$-policy covers
for $\states_1,\ldots,\states_{h-1}$. Then
by~\pref{lem:policy-cover-guarantee}, for each $\state \in \states_h$
there exists $i \in [N]$ such that
\begin{align*}
\PP[\state \mid \hat{\policy}_{i,h}] \geq \eta(\state) - 2N^2 H \errcsc - 32\sqrt{2} N^{3}|\actions|^{3/2}\sqrt{\errreg/\eta(\state)}.
\end{align*}
We simply must set $\npolo$ and $\nreg$ so that the right hand side
here is at least $\eta(s)/2$. By inspection, sufficient conditions for
both parameters are:
\begin{align*}
\npolo \geq \frac{32^2 N^4 H^2 |\actions|}{\etamin^2} \ln \rbr{\frac{4N\horizon^2 |\policies|}{\delta}}, 
\qquad 
2 \underbrace{\nreg}_{\defeq v} \geq \underbrace{\frac{512^2 N^6 |\actions|^3}{\etamin^3}}_{=: a} \rbr{ \underbrace{N^2 |\actions|}_{=: c} \ln(\nreg) + \underbrace{\ln |\Phi_N| + \ln(6H/\delta)}_{=: b}}.
\end{align*}
To simplify the condition for $\nreg$ we use the following
transcendental inequality: For any $a > e$ and any $b$ if $v \geq
a\max\{c\ln(ac)+b,0\}$ then $2v \geq ac\ln(v) + ab$. To see why, observe
that
\begin{align*}
ac\ln(v) + ab = ac\ln(v/(ac)) + ac\ln(ac) + ab \leq v - ac + ac\ln(ac) + ab \leq 2v,
\end{align*}
where the first inequality is simply that $\ln(x) \leq x - 1$ for $x
> 0$, and the second inequality uses the lower bound on $v$. Using the
highlighted definitions, a sufficient condition for $\nreg$ is
\begin{align*}
\nreg \geq \frac{512^2 N^6 |\actions|^3}{\etamin^3}\rbr{N^2|\actions|\ln\rbr{\frac{512^2N^8|\actions|^4}{\etamin^3}} + \ln|\Phi_N| + \ln(6H/\delta)}.
\end{align*}
Note that the algorithm sets these quantities in terms of a parameter
$\eta$ instead of $\etamin$, which may not be known. As long as $\eta
\leq \etamin$ our settings of $\npolo$ and $\nreg$ certify that
$\cover_h$ is a $\nicefrac{1}{2}$-policy cover for
$\states_h$. Appealing to the induction, this establishes the policy
cover guarantee.

\paragraph{The reward sensitive step.}
Equipped with the policy covers, a single call to $\psdpalg$ with the
external reward $R$ and an application of~\pref{thm:psdp-spread}
yield the PAC guarantee. We have already accounted for the failure
probability, so we must simply set
$\neval$. Applying~\pref{thm:psdp-spread} with the definition of
$\delta_3 = \delta/(3\horizon)$, we get
\begin{align*}
\neval \geq \frac{64 N^2H^2|\actions|}{\epsilon^2}\ln\rbr{\frac{3\horizon|\policies|}{\delta}}.
\end{align*}

\paragraph{Sample complexity.}
As we solve $H$ supervised learning problem, make $NH$ calls to
$\psdpalg$ with parameter $\npolo$, and make 1 call to $\psdpalg$ with
parameter $\neval$, the sample complexity, measured in trajectories,
is
\begin{align*}
& H\cdot \nreg + N H^2 \npolo + H\neval = \\
& \otil\rbr{ \frac{N^8|\actions|^4 H}{\etamin^3}  + \frac{N^6|\actions|H}{\etamin^3}\ln(|\Phi_N|/\delta) %
 + \rbr{\frac{N^5H^4|\actions|}{\etamin^2} + \frac{N^2H^3|\actions|}{\epsilon^2}}\ln(|\policies|/\delta)}.
\end{align*}

\paragraph{Computational complexity.}
The running time is simply the time required to collect this many
trajectories, plus the time required for all of the calls to the
oracle. If $T$ is the number of trajectories, the running time is
\begin{align*}
\order\rbr{HT + H \creg(\nreg) + NH^2 \cpol(\npol) + H \cpol(\neval)}.
\end{align*}

\section{Supporting results}
\label{app:supporting}

The next lemma is the well-known performance difference lemma, which
has appeared in much prior
work~\cite{bagnell2004policy,kakade2003sample,ross2014reinforcement,dann2017unifying}. Our
version, which is adapted to episodic problems, is most closely
related to Lemma 4.3 of~\citep{ross2014reinforcement}. We provide a
short proof for completeness.
\begin{lemma}[Performance difference lemma]
\label{lem:perf-diff}
For any episodic decision process with any reward function $R$, and
any two non-stationary policies $\pi^{(1)}_{1:H}$ and
$\pi^{(2)}_{1:H}$, let $Q_h^{(1)} \in \Delta(\observations_h)$ be the
distribution at time $h$ induced by policy $\pi_{1:H}^{(1)}$.  Then we
have
\begin{align*}
V(\pi_{1:H}^{(1)}; R) - V(\pi_{1:H}^{(2)}) = \sum_{h=1}^H \EE_{x_h \sim Q_h^{(1)}}\sbr{V(x_h; \pi_h^{(1)}\circ \pi^{(2)}_{h+1:H}) - V(x_h; \pi^{(2)}_{h:H}) }.
\end{align*}
\end{lemma}
\begin{proof}
The proof is a standard telescoping argument. 
\begin{align*}
V(\pi_{1:H}^{(1)}; R) - V(\pi_{1:H}^{(2)}) & = V(\pi_{1:H}^{(1)}; R) - V(\pi_{1}^{(1)}\circ\pi_{2:H}^{(2)};R) + V(\pi_{1}^{(1)}\circ\pi_{2:H}^{(2)};R)- V(\pi_{1:H}^{(2)})\\
& = V(\pi_{1}^{(1)}\circ\pi_{2:H}^{(2)};R)- V(\pi_{1:H}^{(2)}) + \EE_{x_2\sim Q_2^{(1)}} \sbr{ V(\pi_{2:H}^{(1)};R) -V(\pi_{2:H}^{(2)}; R)}.
\end{align*}
The result follows by repeating this argument on the second term. 
\end{proof}

The next result is Bernstein's inequality. The proof can be found in a
number of textbooks~\cite[c.f.,][]{boucheron2013concentration}.
\begin{proposition}[Bernstein's inequality]
\label{prop:bernstein}
  If $U_1,\ldots,U_n$ are independent zero-mean random variables with
  $\abr{U_t} \leq R$ a.s., and $\frac{1}{n}\sum_{t=1}^n
  \Var(U_t) \leq \sigma^2$, then, for any $\delta \in (0,1)$, with
  probability at least $1-\delta$ we have
\begin{align*}
\frac{1}{n}\sum_{t=1}^n U_t \leq \sqrt{\frac{2\sigma^2\ln(1/\delta)}{n}} + \frac{2R\ln(1/\delta)}{3n}.
\end{align*}
\end{proposition}

The next proposition concerns learning with square loss, using a
function class $\Gcal$ with parametric metric entropy growth rate. Let
$D$ be a distribution over $(x,y)$ pairs, where $x \in \Xcal$ is an
example space and $y \in [0,1]$. With a sample $\{(x_i,y_i)\}_{i=1}^n
\iidsim D$ and a function class $\Gcal: \Xcal\to [0,1]$, we may
perform empirical risk minimization to find
\begin{align}
\hat{g} \defeq \argmin_{g \in \Gcal} \hat{R}_n(g) \defeq \argmin_{g \in \Gcal} \frac{1}{n}\sum_{i=1}^n (g(x_i) - y_i)^2.\label{eq:erm}
\end{align}
The population risk and minimizer are defined as
\begin{align*}
g^\star \defeq \argmin_{g \in \Gcal} R(g) \defeq \argmin_{g \in \Gcal} \EE_{(x,y) \sim D} (g(x) - y)^2. 
\end{align*}
We assume \emph{realizability}, which means that the Bayes optimal
classifier $x \mapsto \EE_D[y \mid x]$ is in our class, and as this
minimizes the risk over all functions we know that $g^\star(x)$ is
precisely this classifier.

We assume that $\Gcal$ has ``parametric" pointwise metric entropy
growth rate, which means that the pointwise covering number at scale
$\varepsilon$, which we denote $\Ncal(\Gcal,\varepsilon)$ scales as
$\Ncal(\Gcal,\varepsilon) \leq c_0 d \ln(1/\varepsilon)$, for a
universal constant $c_0 > 0$. Recall that for a function class $\Gcal:
\Xcal \to \RR$ the pointwise covering number at scale $\varepsilon$ is
the size of the smallest set $V: \Xcal \to \RR$ such that
\begin{align*}
\forall g \in \Gcal, \exists v \in V: \sup_{x} \abr{g(x) - v(x)} \leq \varepsilon.
\end{align*}
With the above definitions, we can state the main guarantee for the
empirical risk minimizer. 

\begin{proposition}\label{prop:sq-loss}
Fix $\delta \in (0,1)$. Let $\hat{g}$ be the empirical risk minimizer
in~\pref{eq:erm} based on $n$ samples from a distribution $D$. If
$\Gcal$ is realizable for $D$ and has parametric entropy growth rate,
then with probability at least $1-\delta$ we have
\begin{align*}
\EE_{(x,y)\sim D}\sbr{\rbr{\hat{g}(x) - g^\star(x)}^2} \leq \errreg, \mbox{ with } \errreg \defeq \inf_{\varepsilon > 0} \cbr{6\varepsilon + \frac{8\ln(2\Ncal(\Gcal,\varepsilon)/\delta)}{n}}, %
\end{align*}
where $C>0$ is a universal constant.
\end{proposition}
The result here is a standard square loss excess risk bound, and it is
perhaps the simplest such bound for well-specified infinite function
classes. Sharper guarantees based on using empirical covering numbers,
combinatorial parameters~\cite{alon1997scale}, or
localization~\citep{liang2015learning}, are possible, and completely
composable with the rest of our arguments. In other words, if a bound
similar to the one in~\pref{prop:sq-loss} is achievable under different
assumptions (e.g., different measure of statistical complexity), we
can incorporate it into the proof of~\pref{thm:main-theorem}. We
state, prove, and use this simple bound to keep the arguments self
contained.

\begin{proof}
Define
\begin{align*}
Z_i(g) = (g(x_i) - y_i)^2 - (g^\star(x_i)-y_i)^2.
\end{align*}
Using the realizability assumption that $\EE [y \mid x ] =
g^\star(x)$, it is easy to verify that
\begin{align*}
\EE[Z_i(g)] = \EE[g(x)^2 - g^\star(x)^2 - 2y(g(x)- g^\star(x))] = \EE[(g(x) - g^\star(x))^2].
\end{align*}
The variance is similarly controlled:
\begin{align*}
\Var[Z_i(g)] &\leq \EE[Z_i(g)^2] = \EE[(g(x) + g^\star(x)- 2y)^2 (g(x) - g^\star(x)^2]\\
& \leq 4 \EE[(g(x) - g^\star(x)^2] = 4 \EE [Z_i(g)],
\end{align*}
where we use that $y, g(\obs), g^\star(\obs)$ are in $[0,1]$.
Therefore, via Bernstein's inequality (\pref{prop:bernstein}), with
probability at least $1-\delta$ we have
\begin{align}
\abr{\frac{1}{n}\sum_i Z_i(g) - \EE[Z(g)]} \leq 2 \sqrt{ \frac{\EE[Z(g)]\ln(2/\delta)}{n}} + \frac{2\ln(2/\delta)}{n}.\label{eq:bernstein}
\end{align}
For the uniform convergence step, we show that $Z_i(g)$ is a Lipschitz
function in $g$:
\begin{align*}
\abr{Z_i(g)- Z_i(g')} &= \abr{(g(x_i) - g'(x_i))^2(g(x_i) + g'(x_i) - 2y_i)^2}  \leq 4\abr{g(x_i) - g'(x_i)},
\end{align*}
where we again use that $y_i, g(\obs_i)$ and $g'(\obs_i)$ are in $[0, 1]$.

Now let $V$ be a pointwise cover of $\Gcal$ at scale $\varepsilon$, so
that for any $g \in \Gcal$ there exits $v_g \in V$ such that:
\begin{align*}
\sup_{x} \abr{v_g(x) - g(x)} \leq \varepsilon.
\end{align*}
By our metric entropy assumptions, we have that $|V| \leq
\Ncal(\Gcal,\varepsilon) \leq \varepsilon^{-c_0 d}$. For any $g \in \Gcal$
we have
\begin{align*}
\frac{1}{n}\sum_{i=1}^n Z_i(g) &\leq \varepsilon + \frac{1}{n}\sum_{i=1}^{n} Z_i(v_g) \leq \varepsilon + \EE[Z(v_g)] + 2\sqrt{\frac{\EE[Z(v_g)]\ln(2\Ncal(\Gcal,\varepsilon)/\delta)}{n}} + \frac{2\ln(2\Ncal(\Gcal,\varepsilon)/\delta)}{n}\\
& \leq \varepsilon + 2\EE[Z(v_g)] + \frac{3\ln(2\Ncal(\Gcal,\varepsilon)/\delta)}{n}
 \leq 3\varepsilon + 2\EE[Z(v_g)] + \frac{3\ln(2\Ncal(\Gcal,\varepsilon)/\delta)}{n}.
\end{align*}
Here we are applying~\pref{eq:bernstein} uniformly over the pointwise
cover,  using the fact that $2\sqrt{ab} \leq a + b$, and using the
pointwise covering property. Similarly we can control the other tail
\begin{align*}
\EE[Z(g)] & \leq \varepsilon + \EE[Z(v_g)] \leq \varepsilon + \frac{1}{n}\sum_{i=1}^n Z_i(v_g) + 2\sqrt{\frac{\EE[Z(v_g)]\ln(2\Ncal(\Gcal,\varepsilon)/\delta)}{n}} + \frac{2\ln(2\Ncal(\Gcal,\varepsilon)/\delta)}{n}\\
& \leq \frac{5}{2}\varepsilon + \frac{1}{n}\sum_{i=1}^nZ_i(g) + \frac{1}{2}\EE[Z(g)] + \frac{4\ln(2\Ncal(\Gcal,\varepsilon)/\delta)}{n}
\end{align*}
Re-arranging and putting the two bounds together, the following holds simultaneously for all $g \in \Gcal$, with probability at least $1-\delta$
\begin{align*}
\frac{1}{2}\EE[Z(g)] - 3\varepsilon - \frac{4\ln(2\Ncal(\Gcal,\varepsilon)/\delta)}{n} \leq \frac{1}{n}\sum_{i=1}^nZ_i(g) \leq 2\EE[Z(g)] + 3\varepsilon + \frac{4\ln(2\Ncal(\Gcal,\varepsilon)/\delta)}{n}.\tag*\qedhere
\end{align*}
\end{proof}

\section{Experimental Setup and Optimization Details}
\label{app:appendix-experiment}

\paragraph{Learning Details for $\homingalg$.} We describe the details of the oracle and hyperparameter below:

\noindent\textit{Oracle Implementation:} We implement the $\cpeoracle$ subroutine by performing supervised binary classification instead of regression. Formally, we train the model $p_w(. \mid \obs', \action, \obs)$ on a training data $\{(\obs_i, \action_i, \obs'_i, y_i)\}_{i=1}^n$ as shown below:
\begin{equation*}
\max_{w} \frac{1}{n}\sum_{i=1}^n \ln p_w\left(y_i \mid \obs_i, \action_i, \obs'_i\right)
\end{equation*}
We use Adam optimization with mini batches. We separate a validation set from the training data. We train the model for a maximum number of epochs and use the model with the best validation performance. We first train the models without bottleneck and use it for initialization. The two training procedures are identical barring the difference in the models.

We learn policies for the $\cboracle$ subroutine by training a model to predict the immediate reward using mean squared loss error. This is equivalent to one-step Q-learning. Formally, we train a model $Q_\theta: \observations \times \actions \rightarrow \mathbb{R}$ on training data $\{(\obs_i, \action_i, r_i)\}_{i=1}^m$ as shown below:
\begin{equation*}
\max_{\theta} \frac{1}{m} \sum_{i=1}^m \left( Q_\theta(\obs_i, \action_i) - r_i \right)^2.
\end{equation*}
The policy corresponding to $Q_\theta$ deterministically takes the action $\arg\max_{\action \in \actions} Q_\theta(\obs, \action)$. We use Adam optimization with mini batches and terminate after a fixed number of epochs. We do not use a validation dataset for $\cboracle$ and use the policy model learned after the final epoch.

We use the two empirical optimizations described in~\pref{sec:complexity-analysis} which provide significant computational and statistical advantage.  
Hyperparameter values for the diabolical combination lock problem can be found in~\pref{tab:homer-hyperparam}.  We use the PyTorch library (version 1.1.0) for implementing $\homingalg$.\footnote{https://pytorch.org/} We use the standard mechanism provided by PyTorch for initializing the parameters.

\begin{table}[h!]
  \caption{$\homingalg$ Hyperparameters}
  \label{tab:homer-hyperparam}
  \centering
  \begin{tabular}{ll}
    \toprule
    Hyperparameter & Value \\
    \hline
    Learning Rate & $0.001$ (for both $\cpeoracle$ and $\cboracle$) \\
    Batch size & $32$ (for both $\cpeoracle$ and $\cboracle$) \\
    $\nreg$     & We sample $10,000$ observed transitions from which \\
    & we generate additional $10,000$ imposter transitions. \\
    $\npol$     & $20,000$  \\
    $N$ (capacity of a backward state abstraction model)  & 2 \\
    $M$ (capacity of a forward state abstraction model)  & 3 \\
    Maximum epochs for $\cpeoracle$ & 200  \\
    Maximum epochs for $\cboracle$ & 50 \\
    Validation data size ($\cpeoracle$) & 20\% of the size of training data for $\cpeoracle$. \\
    Hidden layer size for $\Fcal_N$ & 56 \\
    Gumbel-Softmax temperature & 1 \\
    \bottomrule
  \end{tabular}
\end{table}

\paragraph{Learning Details for $\du$.} We use the code made publicly available by the authors.\footnote{https://github.com/Microsoft/StateDecoding}  $\du$ uses a model for predicting the previous state and action and performs $k$-means clustering on the predicted probabilities. We used a linear model which is representationally sufficient for the diabolical combination lock problem. We used the hyperparameter setting  recommended by the authors who evaluated on a combination lock problem similar to ours. One key exception was the the data collection hyperparameter ($n$) used by their state decoding algorithm. We tried a value of $n$ in $\{200, 1000, 2000\}$. We list the hyperparameter choice in~\pref{tab:du-hyperparam}.

\begin{table}[h!]
  \caption{$\du$ Hyperparameters}
  \label{tab:du-hyperparam}
  \centering
  \begin{tabular}{ll}
    \toprule
    Hyperparameter & Value \\
    \hline
    Learning Rate & $0.03$  \\
    $n$ & 200 \\
    Number of clusters for $k$-means &  3 \\
    \bottomrule
  \end{tabular}
\end{table}

\paragraph{Learning Details for $\aac, \ppo, \aacrnd$ and $\ppornd$.} We train each baseline for a maximum of $10$ million episodes. All baseline models use fully-connected, 2-layer MLPs with 64 hidden units and ReLU non-linearities. For each baseline, we used the RMSProp optimizer \cite{RMSProp} and tuned learning rates over $\{0.01, 0.001, 0.0001\}$. For the methods which use the RND bonus, the random networks were 2-layer MLPs with 100 hidden units and ReLU non-linearities. We found that tuning the intrinsic reward coefficient $\lambda_I$ was important to obtain good performance, and performed an initial search over $\lambda_I \in \{1, 10, 100, 500\}$ for $H \in \{6, 12, 25\}$ for $\aacrnd$ and found that $\lambda_I = 100$ worked best across different horizon lengths. We then used this value for all subsequent experiments. We also experimented with higher entropy bonuses for $\aac$ for $H=6$, but this did not yield any improvement so we kept the default value of $0.01$ for subsequent experiments. We used the $\ppo$ and $\aac$ implementations provided in~\cite{deeprl} and kept other hyperparameters fixed at their default values. We list the hyperparameter values for $\aac$ and $\aacrnd$ in~\pref{tab:aac-hyperparam} and  for $\ppo$ and $\ppornd$ in~\pref{tab:ppo-hyperparam}.

\begin{table}[h!]
  \caption{$\aac$ and $\aacrnd$ Hyperparameters}
  \label{tab:aac-hyperparam}
  \centering
  \begin{tabular}{ll}
    \toprule
    Hyperparameter & Value \\
    \hline
    Learning Rate & $0.001$ \\
    Rollout length     & $H$ \\
    $\gamma$     & $0.99$  \\
    $\tau_\mathrm{GAE}$ & 0.95 \\
    Gradient clipping & 0.5 \\
    Entropy Bonus & 0.01 \\
    Extrinsic Reward coefficient $\lambda_I$     & $1.0$ \\
    Intrinsic Reward coefficient $\lambda_E$ (for $\aacrnd$)     & $100$  \\
    \bottomrule
  \end{tabular}
\end{table}

\begin{table}[h!]
  \caption{$\ppo$ and $\ppornd$ Hyperparameters}
  \label{tab:ppo-hyperparam}
  \centering
  \begin{tabular}{ll}
    \toprule
    Hyperparameter & Value \\
    \hline
    Learning Rate & $0.001$ \\
    Rollout length     & $H$ \\
    $\gamma$     & $0.99$  \\
    $\tau_\mathrm{GAE}$ & 0.95 \\
    Gradient clipping & 5 \\
    Entropy Bonus & 0.01 \\
    Optimization Epochs & 10 \\
    Minibatch size & 160 \\
    Ratio clip & 0.2 \\
    Extrinsic Reward coefficient $\lambda_I$     & $1.0$ \\
    Intrinsic Reward coefficient $\lambda_E$ (for $\ppornd$)     & $100$  \\
    \bottomrule
  \end{tabular}
\end{table}

\end{document}